\pdfoutput=1
\documentclass[letterpaper]{article}
\usepackage[totalwidth=480pt, totalheight=680pt]{geometry}

\usepackage{algorithm}
\usepackage{algorithmicx}
\usepackage{algpseudocode}
\usepackage{subcaption}

\usepackage{enumerate}
\usepackage{bbm}
\usepackage{amsmath}
\usepackage{amssymb}
\usepackage{amsthm}
\usepackage{pifont}
\usepackage{booktabs}
\usepackage{xcolor}
\definecolor{darkgreen}{rgb}{0,0.5,0}
\usepackage{hyperref}
\hypersetup{
    unicode=false,          
    colorlinks=true,        
    linkcolor=red,          
    citecolor=darkgreen,    
    filecolor=magenta,      
    urlcolor=blue           
}
\usepackage[capitalize, nameinlink]{cleveref}

\usepackage{thmtools}
\usepackage{thm-restate}
\theoremstyle{plain}
\newtheorem{theorem}{Theorem}

\newtheorem{lemma}[theorem]{Lemma}

\theoremstyle{definition}
\newtheorem{definition}[theorem]{Definition}

\theoremstyle{remark}

\newtheorem*{remark*}{Remark}

\usepackage{tikz}
\usetikzlibrary{calc, graphs, graphs.standard, shapes, arrows, arrows.meta, positioning, decorations.pathreplacing, decorations.markings, decorations.pathmorphing, fit, matrix, patterns, shapes.misc, tikzmark}

\newcommand{\eps}{\varepsilon}

\newcommand{\N}{\mathbb{N}}

\newcommand{\cA}{\mathcal{A}}
\newcommand{\cB}{\mathcal{B}}
\newcommand{\cI}{\mathcal{I}}

\newcommand{\cC}{\mathcal{C}}
\newcommand{\cE}{\mathcal{E}}
\newcommand{\cO}{\mathcal{O}}
\newcommand{\cV}{\mathcal{V}}

\newcommand{\skel}{\mathrm{skel}}
\newcommand{\Pa}{\texttt{Pa}}
\newcommand{\Ch}{\texttt{Ch}}
\newcommand{\Anc}{\texttt{Anc}}
\newcommand{\Des}{\texttt{Des}}

\newcommand{\wt}{\widetilde}
\newcommand{\argmax}{\mathrm{argmax}}
\newcommand{\argmin}{\mathrm{argmin}}
\newcommand{\dist}{\texttt{dist}}

\title{Active causal structure learning with advice}
\author{
Davin Choo\\
National University of Singapore
\and
Themis Gouleakis\\
National University of Singapore\\
\and
Arnab Bhattacharyya\\
National University of Singapore
}
\date{}

\begin{document}

\maketitle

\begin{abstract}
We introduce the problem of active causal structure learning with advice.
In the typical well-studied setting, the learning algorithm is given the essential graph for the observational distribution and is asked to recover the underlying causal directed acyclic graph (DAG) $G^*$ while minimizing the number of interventions made.
In our setting, we are additionally given side information about $G^*$ as advice, e.g.\ a DAG $G$ purported to be $G^*$.
We ask whether the learning algorithm can benefit from the advice when it is close to being correct, while still having worst-case guarantees even when the advice is arbitrarily bad.
Our work is in the same space as the growing body of research on \emph{algorithms with predictions}.
When the advice is a DAG $G$, we design an adaptive search algorithm to recover $G^*$ whose intervention cost is at most $\cO(\max\{1, \log \psi\})$ times the cost for verifying $G^*$; here, $\psi$ is a distance measure between $G$ and $G^*$ that is upper bounded by the number of variables $n$, and is exactly 0 when $G=G^*$.
Our approximation factor matches the state-of-the-art for the advice-less setting.
\end{abstract}

\section{Introduction}
\label{sec:introduction}

A \emph{causal directed acyclic graph} on a set $V$ of $n$ variables is a Bayesian network in which the edges model direct causal effects. A causal DAG can be used to infer not only the observational distribution of $V$ but also the result of any intervention on any subset of variables $V' \subseteq V$. In this work, we restrict ourselves to the \emph{causally sufficient} setting where there are no latent confounders, no selection bias, and no missingness in data.

The goal of \emph{causal structure learning} is to recover the underlying DAG from data.
This is an important problem with applications in multiple fields including philosophy, medicine, biology, genetics, and econometrics \cite{reichenbach1956direction,hoover1990logic,king2004functional,woodward2005making,rubin2006estimating,eberhardt2007interventions,sverchkov2017review,rotmensch2017learning,pingault2018using}.
Unfortunately, in general, it is known that observational data can only recover the causal DAG up to an equivalence class \cite{pearl2009causality,spirtes2000causation}.
Hence, if one wants to avoid making parametric assumptions about the causal mechanisms, the only recourse is to obtain experimental data from interventions \cite{eberhardt2005number,eberhardt2006n,eberhardt2010causal}.

Such considerations motivate the problem of \emph{interventional design} where the task is to find a set of interventions of optimal cost
which is sufficient to recover the causal DAG. There has been a series of recent works studying this problem \cite{he2008active,hu2014randomized,shanmugam2015learning,kocaoglu2017cost, lindgren2018experimental, greenewald2019sample,squires2020active,choo2022verification, choo2023subset} under various assumptions.
In particular, assuming causal sufficiency, \cite{choo2022verification} gave an adaptive algorithm that actively generates a sequence of interventions of bounded size, so that the total number of interventions is at most $\cO(\log n)$ times the optimal.

Typically though, in most applications of causal structure learning, there are domain experts and practitioners who can provide additional ``advice''
about the causal relations.
Indeed, there has been a long line of work studying how to incorporate expert advice into the causal graph discovery process; e.g.\ see \cite{meek1995,scheines1998tetrad,de2011efficient,flores2011incorporating,li2018bayesian,andrews2020completeness,fang2020ida}.
In this work, we study in a principled way how using purported expert advice can lead to improved algorithms for interventional design.

Before discussing our specific contributions, let us ground the above discussion with a concrete problem of practical importance.
In modern virtualized infrastructure, it is increasingly common for applications to be modularized into a large number of interdependent microservices.
These microservices communicate with each other in ways that depend on the application code and on the triggering userflow.
Crucially, the communication graph between microservices is often unknown to the platform provider as the application code may be private and belong to different entities.
However, knowing the graph is useful for various critical platform-level tasks, such as fault localization \cite{zhou2019latent}, active probing \cite{tan2019netbouncer}, testing \cite{jha2019ml}, and taint analysis \cite{clause2007dytan}.
Recently, \cite{wang2023fault} and \cite{ikram2022root} suggested viewing the microservices communication graph as a sparse causal DAG.
In particular, \cite{wang2023fault} show that arbitrary interventions can be implemented as fault injections in a staging environment, so that a causal structure learning algorithm can be deployed to generate a sequence of interventions sufficient to learn the underlying communication graph.
In such a setting, it is natural to assume that the platform provider already has an approximate guess about the graph, e.g.\ the graph discovered in a previous run of the algorithm or the graph suggested by public metadata tagging microservice code.
The research program we put forth is to design causal structure learning algorithms that can take advantage of such potentially imperfect advice\footnote{Note however that the system in \cite{wang2023fault} is not causally sufficient due to confounding user behavior and \cite{ikram2022root} does not actively perform interventions. So, the algorithm proposed in this work cannot be used directly for the microservices graph learning problem.}.

\subsection{Our contributions}

In this work, we study \emph{adaptive intervention design} for recovering \emph{non-parametric} causal graphs \emph{with expert advice}.
Specifically, our contributions are as follows.

\begin{itemize}
    \item{\textbf{Problem Formulation}.}
    Our work connects the causal structure learning problem with the burgeoning research area of \emph{algorithms with predictions} or \emph{learning-augmented algorithms} \cite{mitzenmacher2022algorithms} where the goal is to design algorithms that bypass worst-case behavior by taking advantage of (possibly erroneous) advice or predictions about the problem instance.
    Most work in this area has been restricted to online algorithms, data structure design, or optimization, as described later in \cref{sec:apsurvey}.
    However, as we motivated above, expert advice is highly relevant for causal discovery, and to the best of our knowlege, ours is the first attempt to formally address the issue of \emph{imperfect} advice in this context.

    \item{\textbf{Adaptive Search Algorithm}.}
    We consider the setting where the advice is a DAG $G$ purported to be the orientations of all the edges in the graph.
    We define a distance measure which is always bounded by $n$, the number of variables, and equals 0 when $G=G^*$.
    For any integer $k \geq 1$, we propose an adaptive algorithm to generate a sequence of interventions of size at most $k$ that recovers the true DAG $G^*$, such that the total number of interventions is $\cO(\log \psi(G, G^*) \cdot \log k)$ times the optimal number of interventions of size $k$.
    Thus, our approximation factor is never worse than the factor for the advice-less setting in \cite{choo2022verification}.
    Our search algorithm also runs in polynomial time.

    \item{\textbf{Verification Cost Approximation}.}
    For a given upper bound $k \geq 1$, a verifying intervention set for a DAG $G^*$ is a set of interventions of size at most $k$ that, together with knowledge of the Markov equivalence class of $G^*$, determines the orientations of all edges in $G^*$.
    The minimum size of a verifying intervention set for $G^*$, denoted $\nu_k(G^*)$, is clearly a lower bound for the number of interventions required to learn $G^*$ (regardless of the advice graph $G$). 
    One of our key technical results is a structural result about $\nu_1$.
    We prove that for any two DAGs $G$ and $G'$ within the same Markov equivalence class, we always have $\nu_1(G)\leq 2 \cdot \nu_1(G')$ and that this is tight in the worst case.
    Beyond an improved structural understanding of minimum verifying intervention sets, which we believe is of independent interest, this enables us to ``blindly trust'' the information provided by imperfect advice to some extent.
\end{itemize}

Similar to prior works (e.g.\ \cite{squires2020active,choo2022verification,choo2023subset}), we assume causal sufficiency and faithfulness while using ideal interventions.
Under these assumptions, running standard causal discovery algorithms (e.g.\ PC \cite{spirtes2000causation}, GES \cite{chickering2002optimal}) will always successfully recover the correct essential graph from data.
We also assume that the given expert advice is consistent with observational essential graph.
See \cref{sec:appendix-assumptions} for a discussion about our assumptions.

\subsection{Paper organization}

In \cref{sec:preliminaries}, we intersperse preliminary notions with related work.
Our main results are presented in \cref{sec:results} with the high-level technical ideas and intuition given in \cref{sec:techniques}.
\cref{sec:experiments} provides some empirical validation.
See the appendices for full proofs, source code, and experimental details.

\section{Preliminaries and Related Work}
\label{sec:preliminaries}

Basic notions about graphs and causal models are defined in \cref{sec:appprelim}.
To be \emph{very} brief, if $G = (V,E)$ is a graph on $|V| = n$ nodes/vertices where $V(G)$, $E(G)$, and $A(G) \subseteq E(G)$ denote nodes, edges, and arcs of $G$ respectively, we write $u \sim v$ to denote that two nodes $u,v \in V$ are connected in $G$, and write $u \to v$ or $u \gets v$ when specifying a certain direction.
The \emph{skeleton} $\skel(G)$ refers to the underlying graph where all edges are made undirected.
A \emph{v-structure} in $G$ refers to a collection of three distinct vertices $u,v,w \in V$ such that $u \to v \gets w$ and $u \not\sim w$.
Let $G = (V,E)$ be fully unoriented.
For vertices $u,v \in V$, subset of vertices $V' \subseteq V$ and integer $r \geq 0$, we define $\dist_{G}(u,v)$ as the shortest path length between $u$ and $v$, and $N_{G}^r(V') = \{v \in V: \min_{u \in V'} \dist_{G}(u,v) \leq r \} \subseteq V$ as the set of vertices that are $r$-hops away from $V'$ in $G$.
A directed acyclic graph (DAG) is a fully oriented graph without directed cycles.
For any DAG $G$, we denote its Markov equivalence class (MEC) by $[G]$ and essential graph by $\cE(G)$.
DAGs in the same MEC have the same skeleton and the essential graph is a partially directed graph such that an arc $u \to v$ is directed if $u \to v$ in \emph{every} DAG in MEC $[G]$, and an edge $u \sim v$ is undirected if there exists two DAGs $G_1, G_2 \in [G]$ such that $u \to v$ in $G_1$ and $v \to u$ in $G_2$.
It is known that two graphs are Markov equivalent if and only if they have the same skeleton and v-structures \cite{verma1990,andersson1997characterization} and the essential graph $\cE(G)$ can be computed from $G$ by orienting v-structures in $\skel(G)$ and applying Meek rules (see \cref{sec:appendix-meek-rules}).
In a DAG $G$, an edge $u \to v$ is a \emph{covered edge} if $\Pa(u) = \Pa(v) \setminus \{u\}$.
We use $\cC(G) \subseteq E(G)$ to denote the set of covered edges of $G$.

\subsection{Ideal interventions}
\label{sec:ideal-interventions}

An \emph{intervention} $S \subseteq V$ is an experiment where all variables $s \in S$ is forcefully set to some value, independent of the underlying causal structure.
An intervention is \emph{atomic} if $|S| = 1$ and \emph{bounded size} if $|S| \leq k$ for some $k \geq 1$; observational data is a special case where $S = \emptyset$.
The effect of interventions is formally captured by Pearl's do-calculus \cite{pearl2009causality}.
We call any $\cI \subseteq 2^V$ a \emph{intervention set}: an intervention set is a set of interventions where each intervention corresponds to a subset of variables.
An \emph{ideal intervention} on $S \subseteq V$ in $G$ induces an interventional graph $G_S$ where all incoming arcs to vertices $v \in S$ are removed \cite{eberhardt2005number}.
It is known that intervening on $S$ allows us to infer the edge orientation of any edge cut by $S$ and $V \setminus S$ \cite{eberhardt2007causation,hyttinen2013experiment,hu2014randomized,shanmugam2015learning,kocaoglu2017cost}.

We now give a definition and result for graph separators.

\begin{definition}[$\alpha$-separator and $\alpha$-clique separator, Definition 19 from \cite{choo2022verification}]
Let $A,B,C$ be a partition of the vertices $V$ of a graph $G = (V,E)$.
We say that $C$ is an \emph{$\alpha$-separator} if no edge joins a vertex in $A$ with a vertex in $B$ and $|A|, |B| \leq \alpha \cdot |V|$. We call $C$ is an \emph{$\alpha$-clique separator} if it is an \emph{$\alpha$-separator} and a clique.
\end{definition}

\begin{theorem}[\cite{gilbert1984separatorchordal}, instantiated for unweighted graphs]
\label{thm:chordal-separator}
Let $G = (V,E)$ be a chordal graph with $|V| \geq 2$ and $p$ vertices in its largest clique.
There exists a $1/2$-clique-separator $C$ involving at most $p-1$ vertices.
The clique $C$ can be computed in $\cO(|E|)$ time.
\end{theorem}

For ideal interventions, an $\cI$-essential graph $\cE_{\cI}(G)$ of $G$ is the essential graph representing the Markov equivalence class of graphs whose interventional graphs for each intervention is Markov equivalent to $G_S$ for any intervention $S \in \cI$.
There are several known properties about $\cI$-essential graph properties \cite{hauser2012characterization,hauser2014two}:
Every $\cI$-essential graph is a chain graph\footnote{A partially directed graph is a \emph{chain graph} if it does \emph{not} contain any partially directed cycles where all directed arcs point in the same direction along the cycle.}
with chordal\footnote{A chordal graph is a graph where every cycle of length at least 4 has an edge that is not part of the cycle but connects two vertices of the cycle; see \cite{blair1993introduction} for an introduction.} chain components.
This includes the case of $\cI = \emptyset$.
Orientations in one chain component do not affect orientations in other components.
In other words, to fully orient any essential graph $\cE(G^*)$, it is necessary and sufficient to orient every chain component in $\cE(G^*)$.

For any intervention set $\cI \subseteq 2^V$, we write $R(G, \cI) = A(\cE_{\cI}(G)) \subseteq E$ to mean the set of oriented arcs in the $\cI$-essential graph of a DAG $G$.
For cleaner notation, we write $R(G,I)$ for single interventions $\cI = \{I\}$ for some $I \subseteq V$, and $R(G,v)$ for single atomic interventions $\cI = \{\{v\}\}$ for some $v \in V$.
For any interventional set $\cI \subseteq 2^V$, define $G^{\cI} = G[E \setminus R(G,\cI)]$ as the \emph{fully directed} subgraph DAG induced by the \emph{unoriented arcs} in $\cE_{\cI}(G)$, where $G^{\emptyset}$ is the graph obtained after removing all the oriented arcs in the observational essential graph due to v-structures.
See \cref{fig:prelim-example} for an example.
In the notation of $R(\cdot, \cdot)$, the following result justifies studying verification and adaptive search via ideal interventions only on DAGs without v-structures, i.e.\ moral DAGs (\cref{def:moraldag}): since $R(G,\cI) = R(G^{\emptyset},\cI) \;\dot\cup\; R(G, \emptyset)$, any oriented arcs in the observational graph can be removed \emph{before performing any interventions} as the optimality of the solution is unaffected.\footnote{The notation $A \;\dot\cup\; B$ denotes disjoint union of sets $A$ and $B$.}

\begin{theorem}[\cite{choo2023subset}]
\label{thm:moral-suffices}
For any DAG $G = (V,E)$ and intervention sets $\cA, \cB \subseteq 2^V$,
\[
R(G,\cA \cup \cB)
= R(G^{\cA},\cB) \;\dot\cup\ R(G^{\cB},\cA) \;\dot\cup\; (R(G,\cA) \cap R(G,\cB))
\]
\end{theorem}

\begin{definition}[Moral DAG]
\label{def:moraldag}
A DAG $G$ is called a \emph{moral DAG} if it has no v-structures.
So, $\cE(G) = \skel(G)$.
\end{definition}

\subsection{Verifying sets}

A \emph{verifying set} $\cI$ for a DAG $G \in [G^*]$ is an intervention set that fully orients $G$ from $\cE(G^*)$, possibly with repeated applications of Meek rules (see \cref{sec:appendix-meek-rules}), i.e.\ $\cE_{\cI}(G^*) = G^*$.
Furthermore, if $\cI$ is a verifying set for $G^*$, then so is $\cI \cup S$ for any additional intervention $S \subseteq V$.
While there may be multiple verifying sets in general, we are often interested in finding one with a minimum size.

\begin{definition}[Minimum size verifying set]
\label{def:min-verifying-set}
An intervention set $\cI \subseteq 2^V$ is called a verifying set for a DAG $G^*$ if $\cE_{\cI}(G^*) = G^*$.
$\cI$ is a \emph{minimum size verifying set} if $\cE_{\cI'}(G^*) \neq G^*$ for any $|\cI'| < |\cI|$.
\end{definition}

For bounded size interventions, the \emph{minimum verification number} $\nu_k(G)$ denotes the size of the minimum size verifying set for any DAG $G \in [G^*]$; we write $\nu_1(G)$ for atomic interventions.
That is, any revealed arc directions when performing interventions on $\cE(G^*)$ respects $G$.
\cite{choo2022verification} tells us that it is necessary and sufficient to intervene on a minimum vertex cover of the covered edges $\cC(G)$ in order to verify a DAG $G$, and that $\nu_1(G)$ is efficiently computable given $G$ since $\cC(G)$ induces a forest.

\begin{theorem}[\cite{choo2022verification}]
\label{thm:verification-characterization}
Fix an essential graph $\cE(G^*)$ and $G \in [G^*]$.
An atomic intervention set $\cI$ is a minimal sized verifying set for $G$ if and only if $\cI$ is a minimum vertex cover of covered edges $\cC(G)$ of $G$.
A minimal sized atomic verifying set can be computed in polynomial time since the edge-induced subgraph on $\cC(G)$ is a forest.
\end{theorem}

For any DAG $G$, we use $\cV(G) \subseteq 2^V$ to denote the set of all \emph{atomic} verifying sets for $G$.
That is, each \emph{atomic} intervention set in $\cV(G)$ is a minimum vertex cover of $\cC(G)$.

\subsection{Adaptive search using ideal interventions}

Adaptive search algorithms have been studied in earnest \cite{he2008active,hauser2014two,shanmugam2015learning,squires2020active,choo2022verification,choo2023subset} as they can use significantly less interventions than non-adaptive counterparts.\footnote{If the essential graph $\cE(G^*)$ is a path of $n$ nodes, then non-adaptive algorithms need $\Omega(n)$ atomic interventions to recover $G^*$ while $\cO(\log n)$ atomic interventions suffices for adaptive search.}

Most recently, \cite{choo2022verification} gave an efficient algorithm for computing adaptive interventions with provable approximation guarantees on general graphs.

\begin{theorem}[\cite{choo2022verification}]
\label{thm:fullsearch}
Fix an unknown underlying DAG $G^*$.
Given an essential graph $\cE(G^*)$ and intervention set bound $k \geq 1$, there is a deterministic polynomial time algorithm that computes an intervention set $\cI$ adaptively such that $\cE_{\cI}(G^*) = G^*$, and $|\cI|$ has size\\
1. $\cO(\log(n) \cdot \nu_1(G^*))$ when $k = 1$\\
2. $\cO(\log(n) \cdot \log (k) \cdot \nu_k(G^*))$ when $k > 1$.
\end{theorem}

Meanwhile, in the context of local causal graph discovery where one is interested in only learning a \emph{subset} of causal relationships, the \texttt{SubsetSearch} algorithm of \cite{choo2023subset} incurs a multiplicative overhead that scales logarithmically with the number of relevant nodes when orienting edges within a node-induced subgraph.

\begin{definition}[Relevant nodes]
Fix a DAG $G^* = (V,E)$ and arbitrary subset $V' \subseteq V$.
For any intervention set $\cI \subseteq 2^V$ and resulting interventional essential graph $\cE_{\cI}(G^*)$, we define the \emph{relevant nodes} $\rho(\cI, V') \subseteq V'$ as the set of nodes within $V'$ that is adjacent to some unoriented arc within the node-induced subgraph $\cE_{\cI}(G^*)[V']$.
\end{definition}

For an example of relevant nodes, see \cref{fig:prelim-example}: For the subset $V' = \{A,C,D,E,F\}$ in (II), only $\{A,C,D\}$ are relevant since incident edges to $E$ and $F$ are all oriented.

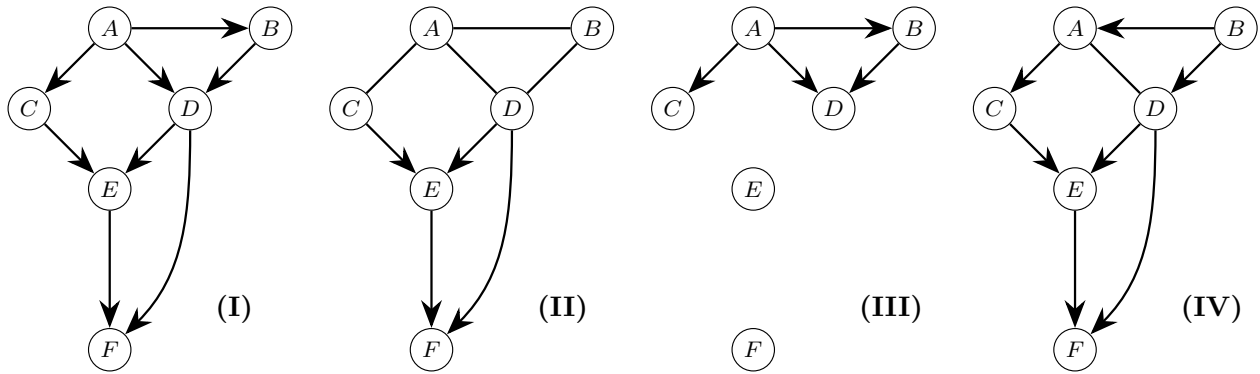
\begin{figure}[tb]
\centering
\resizebox{\linewidth}{!}{%
\begin{tikzpicture}
%
%
\node[draw, circle, minimum size=15pt, inner sep=0pt] at (0,0) (A-gstar) {\footnotesize $A$};
\node[draw, circle, minimum size=15pt, inner sep=0pt] at ($(A-gstar) + (2,0)$) (B-gstar) {\footnotesize $B$};
\node[draw, circle, minimum size=15pt, inner sep=0pt] at ($(A-gstar) + (-1,-1)$) (C-gstar) {\footnotesize $C$};
\node[draw, circle, minimum size=15pt, inner sep=0pt] at ($(A-gstar) + (1,-1)$) (D-gstar) {\footnotesize $D$};
\node[draw, circle, minimum size=15pt, inner sep=0pt] at ($(A-gstar) + (0,-2)$) (E-gstar) {\footnotesize $E$};
\node[draw, circle, minimum size=15pt, inner sep=0pt] at ($(E-gstar) + (0,-2)$) (F-gstar) {\footnotesize $F$};

\draw[thick, -{Stealth[scale=1.5]}] (A-gstar) -- (B-gstar);
\draw[thick, -{Stealth[scale=1.5]}] (A-gstar) -- (C-gstar);
\draw[thick, -{Stealth[scale=1.5]}] (A-gstar) -- (D-gstar);
\draw[thick, -{Stealth[scale=1.5]}] (B-gstar) -- (D-gstar);
\draw[thick, -{Stealth[scale=1.5]}] (C-gstar) -- (E-gstar);
\draw[thick, -{Stealth[scale=1.5]}] (D-gstar) -- (E-gstar);
\draw[thick, -{Stealth[scale=1.5]}] (D-gstar) to[in=45,out=270] (F-gstar);
\draw[thick, -{Stealth[scale=1.5]}] (E-gstar) -- (F-gstar);

\node[below right=of E-gstar] {\textbf{(I)}};

%
%
\node[draw, circle, minimum size=15pt, inner sep=0pt] at (4,0) (A-gessential) {\footnotesize $A$};
\node[draw, circle, minimum size=15pt, inner sep=0pt] at ($(A-gessential) + (2,0)$) (B-gessential) {\footnotesize $B$};
\node[draw, circle, minimum size=15pt, inner sep=0pt] at ($(A-gessential) + (-1,-1)$) (C-gessential) {\footnotesize $C$};
\node[draw, circle, minimum size=15pt, inner sep=0pt] at ($(A-gessential) + (1,-1)$) (D-gessential) {\footnotesize $D$};
\node[draw, circle, minimum size=15pt, inner sep=0pt] at ($(A-gessential) + (0,-2)$) (E-gessential) {\footnotesize $E$};
\node[draw, circle, minimum size=15pt, inner sep=0pt] at ($(E-gessential) + (0,-2)$) (F-gessential) {\footnotesize $F$};

\draw[thick] (A-gessential) -- (B-gessential);
\draw[thick] (A-gessential) -- (C-gessential);
\draw[thick] (A-gessential) -- (D-gessential);
\draw[thick] (B-gessential) -- (D-gessential);
\draw[thick, -{Stealth[scale=1.5]}] (C-gessential) -- (E-gessential);
\draw[thick, -{Stealth[scale=1.5]}] (D-gessential) -- (E-gessential);
\draw[thick, -{Stealth[scale=1.5]}] (D-gessential) to[in=45,out=270] (F-gessential);
\draw[thick, -{Stealth[scale=1.5]}] (E-gessential) -- (F-gessential);

\node[below right=of E-gessential] {\textbf{(II)}};

%
%
\node[draw, circle, minimum size=15pt, inner sep=0pt] at (8,0) (A-gnull) {\footnotesize $A$};
\node[draw, circle, minimum size=15pt, inner sep=0pt] at ($(A-gnull) + (2,0)$) (B-gnull) {\footnotesize $B$};
\node[draw, circle, minimum size=15pt, inner sep=0pt] at ($(A-gnull) + (-1,-1)$) (C-gnull) {\footnotesize $C$};
\node[draw, circle, minimum size=15pt, inner sep=0pt] at ($(A-gnull) + (1,-1)$) (D-gnull) {\footnotesize $D$};
\node[draw, circle, minimum size=15pt, inner sep=0pt] at ($(A-gnull) + (0,-2)$) (E-gnull) {\footnotesize $E$};
\node[draw, circle, minimum size=15pt, inner sep=0pt] at ($(E-gnull) + (0,-2)$) (F-gnull) {\footnotesize $F$};

\draw[thick, -{Stealth[scale=1.5]}] (A-gnull) -- (B-gnull);
\draw[thick, -{Stealth[scale=1.5]}] (A-gnull) -- (C-gnull);
\draw[thick, -{Stealth[scale=1.5]}] (A-gnull) -- (D-gnull);
\draw[thick, -{Stealth[scale=1.5]}] (B-gnull) -- (D-gnull);

\node[below right=of E-gnull] {\textbf{(III)}};

%
%
\node[draw, circle, minimum size=15pt, inner sep=0pt] at (12,0) (A-gmpdag) {\footnotesize $A$};
\node[draw, circle, minimum size=15pt, inner sep=0pt] at ($(A-gmpdag) + (2,0)$) (B-gmpdag) {\footnotesize $B$};
\node[draw, circle, minimum size=15pt, inner sep=0pt] at ($(A-gmpdag) + (-1,-1)$) (C-gmpdag) {\footnotesize $C$};
\node[draw, circle, minimum size=15pt, inner sep=0pt] at ($(A-gmpdag) + (1,-1)$) (D-gmpdag) {\footnotesize $D$};
\node[draw, circle, minimum size=15pt, inner sep=0pt] at ($(A-gmpdag) + (0,-2)$) (E-gmpdag) {\footnotesize $E$};
\node[draw, circle, minimum size=15pt, inner sep=0pt] at ($(E-gmpdag) + (0,-2)$) (F-gmpdag) {\footnotesize $F$};

\draw[thick, -{Stealth[scale=1.5]}] (B-gmpdag) -- (A-gmpdag);
\draw[thick, -{Stealth[scale=1.5]}] (A-gmpdag) -- (C-gmpdag);
\draw[thick] (A-gmpdag) -- (D-gmpdag);
\draw[thick, -{Stealth[scale=1.5]}] (B-gmpdag) -- (D-gmpdag);
\draw[thick, -{Stealth[scale=1.5]}] (C-gmpdag) -- (E-gmpdag);
\draw[thick, -{Stealth[scale=1.5]}] (D-gmpdag) -- (E-gmpdag);
\draw[thick, -{Stealth[scale=1.5]}] (D-gmpdag) to[in=45,out=270] (F-gmpdag);
\draw[thick, -{Stealth[scale=1.5]}] (E-gmpdag) -- (F-gmpdag);

\node[below right=of E-gmpdag] {\textbf{(IV)}};
\end{tikzpicture}
}
\caption{
\textbf{(I)} Ground truth DAG $G^*$;
\textbf{(II)} Observational essential graph $\cE(G^*)$ where $C \to E \gets D$ is a v-structure and Meek rules orient arcs $D \to F$ and $E \to F$;
\textbf{(III)} $G^\emptyset = G[E \setminus R(G, \emptyset)]$ where oriented arcs in $\cE(G^*)$ are removed from $G^*$;
\textbf{(IV)} MPDAG $\tilde{G} \in [G^*]$ incorporating the following partial order advice $(S_1 = \{B\}, S_2 = \{A,D\}, S_3 = \{C,E,F\})$, which can be converted to required arcs $B \to A$ and $B \to D$.
Observe that $A \to C$ is oriented by Meek R1 via $B \to A \sim C$, the arc $A \sim D$ is still unoriented, the arc $B \to A$ disagrees with $G^*$, and there are two possible DAGs consistent with the resulting MPDAG.
}
\label{fig:prelim-example}
\end{figure}

\begin{theorem}[\cite{choo2023subset}]
\label{thm:subsetsearch}
Fix an unknown underlying DAG $G^*$.
Given an interventional essential graph $\cE_{\cI}(G^*)$, node-induced subgraph $H$ with relevant nodes $\rho(\cI, V(H))$ and intervention set bound $k \geq 1$, there is a deterministic polynomial time algorithm that computes an intervention set $\cI$ adaptively such that $\cE_{\cI \cup \cI'}(G^*)[V(H)] = G^*[V(H)]$, and $|\cI'|$ has size\\
1. $\cO(\log(|\rho(\cI, V(H))|) \cdot \nu_1(G^*))$ when $k = 1$\\
2. $\cO(\log(|\rho(\cI, V(H))|) \cdot \log (k) \cdot \nu_k(G^*))$ when $k > 1$.
\end{theorem}

Note that $k=1$ refers to the setting of atomic interventions and we always have $0 \leq |\rho(\cI, V(H))| \leq n$.

\subsection{Expert advice in causal graph discovery}
\label{sec:expert-advice-prelim}

There are three main types of information that a domain expert may provide (e.g.\ see the references given in \cref{sec:introduction}):
\begin{enumerate}[(I)]
    \item Required parental arcs: $X \to Y$
    \item Forbidden parental arcs: $X \not\to Y$
    \item Partial order or tiered knowledge: A partition of the $n$ variables into $1 \leq t \leq n$ sets $S_1, \ldots, S_t$ such that variables in $S_i$ \emph{cannot come after} $S_j$, for all $i < j$.
\end{enumerate}
In the context of orienting unoriented $X \sim Y$ edges in an essential graph, it suffices to consider only information of type (I): $X \not\to Y$ implies $Y \to X$, and a partial order can be converted to a collection of required parental arcs.\footnote{For every edge $X \sim Y$ with $X \in S_i$ and $Y \in S_j$, enforce the required parental arc $X \to Y$ if and only if $i < j$.}

Maximally oriented partially directed acyclic graphs (MPDAGs), a refinement of essential graphs under additional causal information, are often used to model such expert advice and there has been a recent growing interest in understanding them better \cite{perkovic2017interpreting,perkovic2020identifying,guo2021minimal}.
MPDAGs are obtained by orienting additional arc directions in the essential graph due to background knowledge, and then applying Meek rules.
See \cref{fig:prelim-example} for an example.

\subsection{Other related work}
\label{sec:apsurvey}

\paragraph{Causal Structure Learning}
Algorithms for causal structure learning can be grouped into three broad categories, constraint-based, score-based, and Bayesian.
Previous works on the first two approaches are described in \cref{sec:apprelated}. 
In Bayesian methods, a prior distribution is assumed on the space of all structures, and the posterior is updated as more data come in.
\cite{heckerman1995bayesian} was one of the first works on learning from interventional data in this context, which spurred a series of papers (e.g.\ \cite{heckerman1995learning, cooper1999causal, friedman2000being, heckerman2006bayesian}).
Research on active experimental design for causal structure learning with Bayesian updates was initiated by \cite{tong2000active, tong2001active} and \cite{murphy2001active}.
\cite{masegosa2013interactive} considered a combination of Bayesian and constraint-based approaches.
\cite{cho2016reconstructing} and \cite{agrawal2019abcd} have used active learning and Bayesian updates to help recover biological networks.
While possibly imperfect expert advice may be used to guide the prior in the Bayesian approach, the works mentioned above do not provide rigorous guarantees about the number of interventions performed or about optimality, and so they are not directly comparable to our results here.

\paragraph{Algorithms with predictions}
Learning-augmented algorithms have received significant attention since the seminal work of \cite{lykouris2021competitive}, where they investigated the online caching problem with predictions.
Based on that model, \cite{purohit2018improving} proposed algorithms for the ski-rental problem as well as non-clairvoyant scheduling.
Subsequently, \cite{gollapudi2019online}, \cite{wang2020online}, and \cite{angelopoulos2020online} improved the initial results for the ski-rental problem.
Several works, including \cite{rohatgi2020near,antoniadis2020online,wei2020better}, improved the initial results regarding the caching problem.
Scheduling problems with machine-learned advice have been extensively studied in the literature \cite{lattanzi2020online,bamas2020learning,antoniadis2022novel}.
There are also results for augmenting classical data structures with predictions (e.g.\ indexing \cite{kraska2018case} and Bloom filters \cite{mitzenmacher2018model}), online selection and matching problems \cite{antoniadis2020secretary, dutting2021secretaries}, online TSP \cite{bernardiniuniversal,gouleakis2023learning}, and a more general framework of online primal-dual algorithms \cite{bamas2020primal}.

In the above line of work, the extent to which the predictions are helpful in the design of the corresponding online algorithms, is quantified by the following two properties.
The algorithm is called (i) $\alpha$-\textit{consistent} if it is $\alpha$-\textit{competitive} with no prediction error and (ii) $\beta$-\textit{robust} if it is $\beta$-\textit{competitive} with any prediction error.
In the language of learning augmented algorithms or algorithms with predictions, our causal graph discovery algorithm is 1-consistent and $\cO(\log n)$-robust when competing against the verification number $\nu_1(G^*)$, the minimum number of interventions necessary needed to recover $G^*$.
Note that even with arbitrarily bad advice, our algorithm uses asymptotically the same number of interventions incurred by the best-known advice-free adaptive search algorithm \cite{choo2022verification}.

\section{Results}
\label{sec:results}

Our exposition here focuses on interpreting and contextualizing our main results while deferring technicalities to \cref{sec:techniques}.
We first focus on the setting where the advice is a fully oriented DAG $\wt{G} \in [G^*]$ within the Markov equivalence class $[G^*]$ of the true underlying causal graph $G^*$, and explain in \cref{sec:partial-advice} how to handle the case of partial advice.
Full proofs are provided in the appendix.

\subsection{Structural property of verification numbers}

We begin by stating a structural result about verification numbers of DAGs within the same Markov equivalence class (MEC) that motivates the definition of a metric between DAGs in the same MEC our algorithmic guarantees (\cref{thm:search-with-advice}) are based upon.

\begin{restatable}{mytheorem}{twoapproxratio}
\label{thm:two-approx-ratio}
For any DAG $G^*$ with MEC $[G^*]$, we have that $\max_{G \in [G^*]} \nu_1(G) \leq 2 \cdot \min_{G \in [G^*]} \nu_1(G)$.
\end{restatable}

\cref{thm:two-approx-ratio} is the first known result relating the minimum and maximum verification numbers of DAGs given a fixed MEC.
The next result tells us that the ratio of two is tight.

\begin{restatable}[Tightness of \cref{thm:two-approx-ratio}]{mylemma}{tightnessoftwoapprox}
\label{lem:tightness-of-two-approx}
There exist DAGs $G_1$ and $G_2$ from the same MEC with $\nu_1(G_1) = 2 \cdot \nu_1(G_2)$.
\end{restatable}

\cref{thm:two-approx-ratio} tells us that we can blindly intervene on any minimum verifying set $\wt{V} \in \cV(\wt{G})$ of any given advice DAG $\wt{G}$ while incurring only at most a constant factor of 2 more interventions than the minimum verification number $\nu(G^*)$ of the unknown ground truth DAG $G^*$.

\subsection{Adaptive search with imperfect DAG advice}

Recall the definition of $r$-hop from \cref{sec:preliminaries}.
To define the quality of the advice DAG $\wt{G}$, we first define the notion of \emph{min-hop-coverage} which measures how ``far'' a given verifying set of $\wt{G}$ is from the set of covered edges of $G^*$.

\begin{definition}[Min-hop-coverage]
Fix a DAG $G^*$ with MEC $[G^*]$ and consider any DAG $\wt{G} \in [G^*]$.
For any minimum verifying set $\wt{V} \in \cV(\wt{G})$, we define the \emph{min-hop-coverage} $h(G^*, \wt{V}) \in \{0, 1, 2, \ldots, n\}$ as the minimum number of hops such that \emph{both} endpoints of covered edges $\cC(G^*)$ of $G^*$ belong in $N^{h(G^*, \wt{V})}_{\skel(\cE(G^*))}(\wt{V})$.
\end{definition}

Using min-hop-coverage, we now define a quality measure $\psi(G^*, \wt{G})$ for DAG $\wt{G} \in [G^*]$ as an advice for DAG $G^*$.

\begin{restatable}[Quality measure]{mydefinition}{qualitymeasure}
\label{def:quality-measure}
Fix a DAG $G^*$ with MEC $[G^*]$ and consider any DAG $\wt{G} \in [G^*]$.
We define $\psi(G^*, \wt{G})$ as follows:
\[
\psi(G^*, \wt{G}) = \max_{\wt{V} \in \cV(\wt{G})} \left| \rho \left( \wt{V}, N^{h(G^*, \wt{V})}_{\skel(
\cE(G^*))}(\wt{V}) \right) \right|
\]
\end{restatable}

By definition, $\psi(G^*, G^*) = 0$ and $\max_{G \in [G^*]} \psi(G^*, G) \leq n$.
In words, $\psi(G^*, \wt{G})$ only counts the relevant nodes within the min-hop-coverage neighborhood after intervening on the \emph{worst} possible verifying set $\wt{V}$ of $\wt{G}$.
We define $\psi$ via the worst set because any search algorithm \emph{cannot} evaluate $h(G^*, \wt{V})$, since $G^*$ is unknown, and can only consider an \emph{arbitrary} $\wt{V} \in \cV(\wt{G})$.
See \cref{fig:psi-example} for an example.

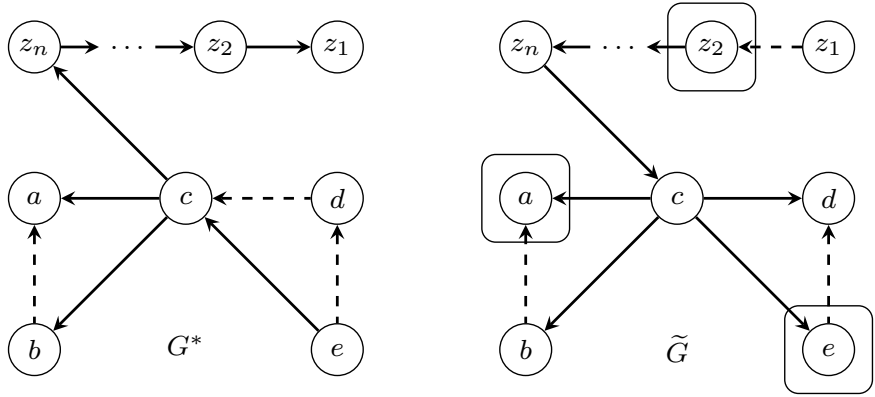
\begin{figure}[tb]
\centering
\resizebox{0.7\linewidth}{!}{%
\begin{tikzpicture}
%
%
\node[draw, circle, minimum size=15pt, inner sep=0pt] at (0,0) (a-gstar) {\footnotesize $a$};
\node[draw, circle, minimum size=15pt, inner sep=0pt, right=of a-gstar] (c-gstar) {\footnotesize $c$};
\node[draw, circle, minimum size=15pt, inner sep=0pt, below=of a-gstar] (b-gstar) {\footnotesize $b$};
\node[draw, circle, minimum size=15pt, inner sep=0pt, right=of c-gstar] (d-gstar) {\footnotesize $d$};
\node[draw, circle, minimum size=15pt, inner sep=0pt, below=of d-gstar] (e-gstar) {\footnotesize $e$};
\node[draw, circle, minimum size=15pt, inner sep=0pt, above=of a-gstar] (zn-gstar) {\footnotesize $z_n$};
\node[minimum size=15pt, inner xsep=3pt, above=of c-gstar, xshift=-17pt] (zdots-gstar) {\footnotesize $\ldots$};
\node[draw, circle, minimum size=15pt, inner sep=0pt, above=of c-gstar, xshift=10pt] (z2-gstar) {\footnotesize $z_2$};
\node[draw, circle, minimum size=15pt, inner sep=0pt, above=of d-gstar] (z1-gstar) {\footnotesize $z_1$};

\draw[thick, stealth-, dashed] (a-gstar) -- (b-gstar);
\draw[thick, stealth-] (a-gstar) -- (c-gstar);
\draw[thick, stealth-] (b-gstar) -- (c-gstar);
\draw[thick, stealth-, dashed] (c-gstar) -- (d-gstar);
\draw[thick, stealth-] (c-gstar) -- (e-gstar);
\draw[thick, stealth-, dashed] (d-gstar) -- (e-gstar);
\draw[thick, stealth-] (z1-gstar) -- (z2-gstar);
\draw[thick, stealth-] (z2-gstar) -- (zdots-gstar);
\draw[thick, stealth-] (zdots-gstar) -- (zn-gstar);
\draw[thick, stealth-] (zn-gstar) -- (c-gstar);

\node[below=of c-gstar] {\footnotesize $G^*$};

%
%
\node[draw, circle, minimum size=15pt, inner sep=0pt] at (5,0) (a-g) {\footnotesize $a$};
\node[draw, circle, minimum size=15pt, inner sep=0pt, right=of a-g] (c-g) {\footnotesize $c$};
\node[draw, circle, minimum size=15pt, inner sep=0pt, below=of a-g] (b-g) {\footnotesize $b$};
\node[draw, circle, minimum size=15pt, inner sep=0pt, right=of c-g] (d-g) {\footnotesize $d$};
\node[draw, circle, minimum size=15pt, inner sep=0pt, below=of d-g] (e-g) {\footnotesize $e$};
\node[draw, circle, minimum size=15pt, inner sep=0pt, above=of a-g] (zn-g) {\footnotesize $z_n$};
\node[minimum size=15pt, inner xsep=3pt, above=of c-g, xshift=-17pt] (zdots-g) {\footnotesize $\ldots$};
\node[draw, circle, minimum size=15pt, inner sep=0pt, above=of c-g, xshift=10pt] (z2-g) {\footnotesize $z_2$};
\node[draw, circle, minimum size=15pt, inner sep=0pt, above=of d-g] (z1-g) {\footnotesize $z_1$};

\draw[thick, stealth-, dashed] (a-g) -- (b-g);
\draw[thick, stealth-] (a-g) -- (c-g);
\draw[thick, stealth-] (b-g) -- (c-g);
\draw[thick, -stealth] (c-g) -- (d-g);
\draw[thick, -stealth] (c-g) -- (e-g);
\draw[thick, stealth-, dashed] (d-g) -- (e-g);
\draw[thick, -stealth, dashed] (z1-g) -- (z2-g);
\draw[thick, -stealth] (z2-g) -- (zdots-g);
\draw[thick, -stealth] (zdots-g) -- (zn-g);
\draw[thick, -stealth] (zn-g) -- (c-g);

\node[draw, rounded corners, fit=(a-g), inner sep=5pt] {};
\node[draw, rounded corners, fit=(e-g), inner sep=5pt] {};
\node[draw, rounded corners, fit=(z2-g), inner sep=5pt] {};
\node[below=of c-g] {\footnotesize $\wt{G}$};
\end{tikzpicture}
}
\caption{
Consider the moral DAGs $G^*$ and $\wt{G} \in [G^*]$ on $n+5$ nodes, where dashed arcs represent the covered edges in each DAG.
A minimum sized verifying set $\wt{V} = \{a, e, z_2\} \in \cV(\wt{G})$ of $\wt{G}$ is given by the boxed vertices on the right.
As $N^1_{\skel(G^*)}(\wt{V}) = \{a, b, c, d, e, z_1, z_2, z_3\}$ includes both endpoints of all covered edges of $G^*$, we see that $h(G^*, \wt{V}) = 1$.
Intervening on $\wt{V} = \{a,e,z_2\}$ in $G^*$ orients the arcs $b \to a \gets c$, $c \gets e \to d$, and $z_3 \to z_2 \to z_1$ respectively which then triggers Meek R1 to orient $c \to b$ via $e \to c \sim b$ and to orient $z_4 \to z_3$ via $e \to c \to \ldots \to z_4 \sim z_3$ (after a few invocations of R1), so $\{a, b, e, z_1, z_2, z_3\}$ will \emph{not} be relevant nodes in $\cE_{\wt{V}}(G^*)$.
Meanwhile, the edge $c \sim d$ remains unoriented in $\cE_{\wt{V}}(G^*)$, so $\rho(\wt{V}, N^1_{\skel(G^*)}(\wt{V})) = |\{c, d\}| = 2$.
One can check that $\psi(G^*, \wt{G}) = 2$ while $n$ could be arbitrarily large.
On the other hand, observe that $\psi$ is \emph{not} symmetric: in the hypothetical situation where we use $G^*$ as an advice for $\wt{G}$, the min-hop-coverage has to extend along the chain $z_1 \sim \ldots \sim z_n$ to reach $\{z_1, z_2\}$, so $h(G^*, V^*) \approx n$ and $\psi(\wt{G}, G^*) \approx n$ since the entire chain remains unoriented with respect to any $V^* \in \cV(G^*)$.
}
\label{fig:psi-example}
\end{figure}

Our main result is that it is possible to design an algorithm that leverages an advice DAG $\wt{G} \in [G^*]$ and performs interventions to fully recover an unknown underlying DAG $G^*$, whose performance depends on the advice quality $\psi(G^*, \wt{G})$.
Our search algorithm only knows $\cE(G^*)$ and $\wt{G} \in [G^*]$ but knows neither $\psi(G^*, \wt{G})$ nor $\nu(G^*)$.

\begin{restatable}{mytheorem}{searchwithadvice}
\label{thm:search-with-advice}
Fix an essential graph $\cE(G^*)$ with an unknown underlying ground truth DAG $G^*$.
Given an advice graph $\wt{G} \in [G^*]$ and intervention set bound $k \geq 1$, there exists a deterministic polynomial time algorithm (\cref{alg:adaptive-search-with-advice-algo}) that computes an intervention set $\cI$ adaptively such that $\cE_{\cI}(G^*) = G^*$, and $|\cI|$ has size\\
1. $\cO( \max\{1, \log \psi(G^*, \wt{G}) \} \cdot \nu_1(G^*))$ when $k = 1$\\
2. $\cO( \max\{1, \log \psi(G^*, \wt{G}) \} \cdot \log k \cdot \nu_k(G^*))$ when $k > 1$.
\end{restatable}

Consider first the setting of $k=1$.
Observe that when the advice is perfect (i.e. $\wt{G} = G^*$), we use $\cO(\nu(G^*))$ interventions, i.e.\ a constant multiplicative factor of the minimum number of interventions necessary.
Meanwhile, even with low quality advice, we still use $\cO(\log n \cdot \nu(G^*))$ interventions, asymptotically matching the best known guarantees for adaptive search without advice.
To the best of our knowledge, \cref{thm:search-with-advice} is the first known result that principally employs imperfect expert advice with provable guarantees in the context of causal graph discovery via interventions.

Consider now the setting of bounded size interventions where $k > 1$.
The reason why we can obtain such a result is precisely because of our algorithmic design: we deliberately designed an algorithm that invokes \texttt{SubsetSearch} as a black-box subroutine.
Thus, the bounded size guarantees of \texttt{SubsetSearch} given by \cref{thm:subsetsearch} carries over to our setting with a slight modification of the analysis.

\section{Techniques}
\label{sec:techniques}

Here, we discuss the high-level technical ideas and intuition behind how we obtain our adaptive search algorithm with imperfect DAG advice.
See the appendix for full proofs; in particular, see \cref{sec:two-overview} for an overview of \cref{thm:two-approx-ratio}.

For brevity, we write $\psi$ to mean $\psi(G^*, \wt{G})$ and drop the subscript $\skel(\cE(G^*))$ of $r$-hop neighborhoods in this section.
We also focus our discussion to the atomic interventions.
Our adaptive search algorithm (\cref{alg:adaptive-search-with-advice-algo}) uses \texttt{SubsetSearch} as a subroutine.

We begin by observing that $\texttt{SubsetSearch}(\cE(G^*), A)$ fully orients $\cE(G^*)$ into $G^*$ if the covered edges of $G^*$ lie within the node-induced subgraph induced by $A$.

\begin{restatable}{mylemma}{ifOPTreachedthenblindsearchfullyorients}
\label{lem:if-OPT-reached-then-blind-search-fully-orients}
Fix a DAG $G^* = (V,E)$ and let $V' \subseteq V$ be any subset of vertices.
Suppose $\cI_{V'} \subseteq V$ is the set of nodes intervened by $\texttt{SubsetSearch}(\cE(G^*), V')$.
If $\cC(G^*) \subseteq E(G^*[V'])$, then $\cE_{\cI_{V'}}(G^*) = G^*$.
\end{restatable}

Motivated by \cref{lem:if-OPT-reached-then-blind-search-fully-orients}, we design \cref{alg:adaptive-search-with-advice-algo} to repeatedly invoke \texttt{SubsetSearch} on node-induced subgraphs $N^r(\wt{V})$, starting from an \emph{arbitrary} verifying set $\wt{V} \in \cV(\wt{G})$ and for \emph{increasing} values of $r$.

For $i \in \mathbb{N} \cup \{0\}$, let us denote $r(i) \in \mathbb{N} \cup \{0\}$ as the value of $r$ in the $i$-th invocation of \texttt{SubsetSearch}, where we insist that $r(0) = 0$ and $r(j) > r(j-1)$ for any $j \in \mathbb{N}$.
Note that $r = 0$ simply implies that we intervene on the verifying set $\wt{V}$, which only incurs $\cO(\nu_1(G^*))$ interventions due to \cref{thm:two-approx-ratio}.
Then, we can appeal to \cref{lem:if-OPT-reached-then-blind-search-fully-orients} to conclude that $\cE(G^*)$ is completely oriented into $G^*$ in the $t$-th invocation if $r(t) \geq h(G^*, \wt{V})$.

While the high-level subroutine invocation idea seems simple, one needs to invoke \texttt{SubsetSearch} at \emph{suitably chosen intervals} in order to achieve our theoretical guarantees we promise in \cref{thm:search-with-advice}.
We now explain how to do so in three successive attempts while explaining the algorithmic decisions behind each modification introduced.

As a reminder, we \emph{do not} know $G^*$ and thus \emph{do not} know $h(G^*, \wt{V})$ for any verifying set $\wt{V} \in \cV(\wt{G})$ of $\wt{G} \in [G^*]$.

\subsubsection*{Naive attempt: Invoke for $r = 0, 1, 2, 3, \ldots$}

The most straightforward attempt would be to invoke \texttt{SubsetSearch} repeatedly each time we increase $r$ by 1 until the graph is fully oriented -- in the worst case, $t = h(G^*, \wt{V})$.
However, this may cause us to incur way too many interventions.
Suppose there are $n_i$ relevant nodes in the $i$-th invocation.
Using \cref{thm:subsetsearch}, one can only argue that the overall number interventions incurred is $\cO(\sum_{i=0}^{t} \log n_i \cdot \nu(G^*))$.
However, $\sum_i \log n_i$ could be significantly larger than $\log (\sum_i n_i)$ in general, e.g.\ $\log 2 + \ldots + \log 2 = (n/2) \cdot \log 2 \gg \log n$.
In fact, if $G^*$ was a path on $n$ vertices $v_1 \to v_2 \to \ldots \to v_n$ and $\wt{G} \in [G^*]$ misleads us with $v_1 \gets v_2 \gets \ldots \gets v_n$, then this approach incurs $\Omega(n)$ interventions in total.

\subsubsection*{Tweak 1: Only invoke periodically}

Since \cref{thm:subsetsearch} provides us a logarithmic factor in the analysis, we could instead consider only invoking \texttt{SubsetSearch} after the number of nodes in the subgraph \emph{increases by a polynomial factor}.
For example, if we invoked \texttt{SubsetSearch} with $n_i$ previously, then we will wait until the number of relevant nodes surpasses $n_i^2$ before invoking \texttt{SubsetSearch} again, where we define $n_0 \geq 2$ for simplicity.
Since $\log n_i \geq 2 \log n_{i-1}$, we can see via an inductive argument that the number of interventions used in the final invocation will dominate the total number of interventions used so far: $n_t \geq 2 \log n_{t-1} \geq \log n_{t-1} + 2 \log n_{t-2} \geq \ldots \geq \sum_{i=0}^{t-1} \log n_i$.
Since $n_i \leq n$ for any $i$, we can already prove that $\cO(\log n \cdot \nu_1(G^*))$ interventions suffice, matching the advice-free bound of \cref{thm:fullsearch}.
However, this approach and analysis does \emph{not} take into account the quality of $\wt{G}$ and is \emph{insufficient} to relate $n_t$ with the advice measure $\psi$.

\subsubsection*{Tweak 2: Also invoke one round before}

Suppose the final invocation of \texttt{SubsetSearch} is on $r(t)$-hop neighborhood while incurring $\cO(\log n_t \cdot \nu_1(G^*))$ interventions.
This means that $\cC(G^*)$ lies within $N^{r(t)}(\wt{V})$ but \emph{not} within $N^{r(t-1)}(\wt{V})$.
That is, $N^{r(t-1)}(\wt{V}) \subsetneq N^{h(G^*, \wt{V})}(\wt{V}) \subseteq N^{r(t)}(\wt{V})$.
While this tells us that $n_{t-1} \leq |\rho(\wt{V}, N^{r(t-1)}(\wt{V}))| < |\rho(\wt{V}, N^{h(G^*, \wt{V})}(\wt{V}))| = \psi$, what we want is to conclude that $n_t \in \cO(\psi)$.
Unfortunately, even when $\psi = r(t-1) + 1$, it could be the case that $|\rho(\wt{V}, N^{h(G^*, \wt{V})}(\wt{V}))| \ll |N^{r(t)}(\wt{V})|$ as the number of relevant nodes could blow up within a single hop (see \cref{fig:blowup}).
To control this potential blow up in the analysis, we can introduce the following technical fix: whenever we want to invoke \texttt{SubsetSearch} on $r(i)$, first invoke \texttt{SubsetSearch} on $r(i)-1$ and terminate earlier if the graph is already fully oriented into $G^*$.

\begin{figure}[tb]
\centering
\resizebox{0.7\linewidth}{!}{%
\begin{tikzpicture}
%
%
\node[draw, circle, minimum size=10pt, inner sep=1pt] at (0,0) (star-v1) {$v_1$};
\node[draw, circle, minimum size=10pt, inner sep=1pt] at (1,0) (star-v2) {$v_2$};
\node[draw, circle, minimum size=10pt, inner sep=1pt] at (2,0) (star-v3) {$v_3$};
\node[draw, circle, minimum size=10pt, inner sep=1pt] at (3,1) (star-v4) {$v_4$};
\node[draw, circle, minimum size=10pt, inner sep=1pt] at (3,0.25) (star-v5) {$v_5$};
\node[minimum size=10pt, inner sep=1pt] at (3,-0.25) (star-vdots) {\small $\vdots$};
\node[draw, circle, minimum size=10pt, inner sep=1pt] at (3,-1) (star-vn) {$v_n$};

\draw[thick, -stealth, dashed] (star-v2) -- (star-v1);
\draw[thick, -stealth, dashed] (star-v2) -- (star-v3);
\draw[thick, -stealth] (star-v3) -- (star-v4);
\draw[thick, -stealth] (star-v3) -- (star-v5);
\draw[thick, -stealth] (star-v3) -- (star-vn);

\node[] at (0,1) {$G^*$};
\node[fit=(star-v2), draw, rounded corners] {};

%
%
\node[draw, circle, minimum size=10pt, inner sep=1pt] at (4,0) (tilde-v1) {$v_1$};
\node[draw, circle, minimum size=10pt, inner sep=1pt] at (5,0) (tilde-v2) {$v_2$};
\node[draw, circle, minimum size=10pt, inner sep=1pt] at (6,0) (tilde-v3) {$v_3$};
\node[draw, circle, minimum size=10pt, inner sep=1pt] at (7,1) (tilde-v5) {$v_4$};
\node[draw, circle, minimum size=10pt, inner sep=1pt] at (7,0.25) (tilde-v6) {$v_5$};
\node[minimum size=10pt, inner sep=1pt] at (7,-0.25) (tilde-vdots) {\small $\vdots$};
\node[draw, circle, minimum size=10pt, inner sep=1pt] at (7,-1) (tilde-vn) {$v_n$};

\draw[thick, -stealth, dashed] (tilde-v1) -- (tilde-v2);
\draw[thick, -stealth] (tilde-v2) -- (tilde-v3);
\draw[thick, -stealth] (tilde-v3) -- (tilde-v5);
\draw[thick, -stealth] (tilde-v3) -- (tilde-v6);
\draw[thick, -stealth] (tilde-v3) -- (tilde-vn);

\node[] at (4,1) {$\wt{G}$};
\node[] at (4.5,-1) {$\wt{V} = \{v_1\}$};
\node[fit=(tilde-v1), draw, rounded corners] {};

%
%
\draw[thick] (3.5, 1.5) -- (3.5, -1.5);
\end{tikzpicture}
}
\caption{
Consider the ground truth DAG $G^*$ with unique minimum verifying set $\{v_2\}$ and an advice DAG $\wt{G} \in [G^*]$ with chosen minimum verifying set $\wt{V} = \{v_1\}$.
So, $h(G^*, \wt{V}) = 1$ and ideally we want to argue that our algorithm uses a constant number of interventions.
Without tweak 2 and $n_0 = 2$, an algorithm that increases hop radius until the number of relevant nodes is squared will \emph{not} invoke \texttt{SubsetSearch} until $r = 3$ because $\rho(\wt{V}, N^1) = 1 < n_0^2$ and $\rho(\wt{V}, N^2) = 2 < n_0^2$.
However, $\rho(\wt{V}, N^3) = n-1$ and we can only conclude that the algorithm uses $\cO(\log n)$ interventions by invoking \texttt{SubsetSearch} on a subgraph on $n-1$ nodes.
\vspace{-10pt}
}
\label{fig:blowup}
\end{figure}
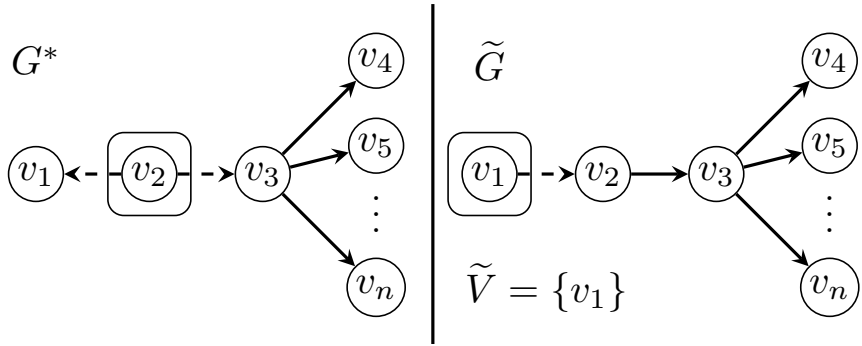

\subsubsection*{Putting together}

\cref{alg:adaptive-search-with-advice-algo} presents our full algorithm where the inequality $\rho(\cI_{i}, N^r_{\skel(\cE(G^*))}(\wt{V})) \geq n_i^2$ corresponds to the first tweak while the terms $C_i$ and $C_i'$ correspond to the second tweak.

In \cref{sec:appendix-path}, we explain why our algorithm (\cref{alg:adaptive-search-with-advice-algo}) is simply the classic ``binary search with prediction''\footnote{e.g.\ see \url{https://en.wikipedia.org/wiki/Learning_augmented_algorithm\#Binary_search}} when the given essential graph $\cE(G^*)$ is an undirected path.
So, another way to view our result is a \emph{generalization} that works on essential graphs of arbitrary moral DAGs.

\begin{algorithm}[htb]
\caption{Adaptive search algorithm with advice.}
\label{alg:adaptive-search-with-advice-algo}
\begin{algorithmic}[1]
    \Statex \textbf{Input}: Essential graph $\cE(G^*)$, advice DAG $\wt{G} \in [G^*]$, intervention size $k \in \N$
    \Statex \textbf{Output}: An intervention set $\cI$ such that each intervention involves at most $k$ nodes and $\cE_{\cI}(G^*) = G^*$.
    \State Let $\wt{V} \in \cV(\wt{G})$ be any atomic verifying set of $\wt{G}$.
    \If{$k = 1$}
        \State Define $\cI_0 = \wt{V}$ as an atomic intervention set.
    \Else
        \State Define $k' = \min\{k, |\wt{V}|/2\}$, $a = \lceil |\wt{V}|/k' \rceil \geq 2$, and $\ell = \lceil \log_a |C| \rceil$. Compute labelling scheme on $\wt{V}$ with $(|\wt{V}|, k, a)$ via \cref{lem:labelling-scheme} and define $\cI_0 = \{S_{x,y}\}_{x \in [\ell], y \in [a]}$, where $S_{x,y} \subseteq \wt{V}$ is the subset of vertices whose $x^{th}$ letter in the label is $y$.
    \EndIf
    \State Intervene on $\cI_0$ and initialize $r \gets 0$, $i \gets 0$, $n_0 \gets 2$.
	\While{$\cE_{\cI_{i}}(G^*)$ still has undirected edges}
        \If{$\rho(\cI_{i}, N^r_{\skel(\cE(G^*))}(\wt{V})) \geq n_i^2$}
            \State Increment $i \gets i + 1$ and record $r(i) \gets r$.
            \State Update $n_i \gets \rho(\cI_{i}, N^r_{\skel(\cE(G^*))}(\wt{V}))$
            \State $C_{i} \gets \texttt{SubsetSearch}(\cE_{\cI_{i}}(G^*), N^{r-1}_{\skel(\cE(G^*))}(\wt{V}), k)$
            \If{$\cE_{\cI_{i-1} \;\cup\; C_{i}}(G^*)$ still has undirected edges}
                \State $C'_{i} \gets \texttt{SubsetSearch}(\cE_{\cI_{i-1} \,\cup\, C_{i}}(G^*), N^{r}_{\skel(\cE(G^*))}(\wt{V}), k)$
                \State Update $\cI_{i} \gets \cI_{i-1} \cup C_{i} \cup C'_{i}$.
            \Else
                \State Update $\cI_{i} \gets \cI_{i-1} \cup C_{i}$.
            \EndIf
        \EndIf
        \State Increment $r \gets r + 1$.
	\EndWhile=
	\State \textbf{return} $\cI_i$
\end{algorithmic}
\end{algorithm}

For bounded size interventions, we rely on the following known results.

\begin{theorem}[Theorem 12 of \cite{choo2022verification}]
\label{thm:efficient-near-optimal-bounded}
Fix an essential graph $\cE(G^*)$ and $G \in [G^*]$.
If $\nu_1(G) = \ell$, then $\nu_k(G) \geq \lceil \frac{\ell}{k} \rceil$ and there exists a polynomial time algo.\ to compute a bounded size intervention set $\cI$ of size $|\cI| \leq \lceil \frac{\ell}{k} \rceil + 1$.
\end{theorem}

\begin{lemma}[Lemma 1 of \cite{shanmugam2015learning}]
\label{lem:labelling-scheme}
Let $(n,k,a)$ be parameters where $k \leq n/2$.
There exists a polynomial time labeling scheme that produces distinct $\ell$ length labels for all elements in $[n]$ using letters from the integer alphabet $\{0\} \cup [a]$ where $\ell = \lceil \log_a n \rceil$.
Further, in every digit (or position), any integer letter is used at most $\lceil n/a \rceil$ times.
This labelling scheme is a separating system: for any $i,j \in [n]$, there exists some digit $d \in [\ell]$ where the labels of $i$ and $j$ differ.
\end{lemma}

\cref{thm:efficient-near-optimal-bounded} enables us to easily relate $\nu_1(G)$ with $\nu_k(G)$ while \cref{lem:labelling-scheme} provides an efficient labelling scheme to partition a set of $n$ nodes into a set $S = \{S_1, S_2, \ldots \}$ of bounded size sets, each $S_i$ involving at most $k$ nodes.
By invoking \cref{lem:labelling-scheme} with $a \approx n'/k$ where $n'$ is related to $\nu_1(G)$, we see that $|S| \approx \frac{n'}{k} \cdot \log k$.
As $\nu_k(G) \approx \nu_1(G)/k$, this is precisely why the bounded intervention guarantees in \cref{thm:fullsearch}, \cref{thm:subsetsearch} and \cref{thm:search-with-advice} have an additional multiplicative $\log k$ factor.

\section{Empirical validation}
\label{sec:experiments}

While our main contributions are theoretical, we also performed some experiments to empirically validate that our algorithm is practical, outperforms the advice-free baseline when the advice quality is good, and still being at most a constant factor worse when the advice is poor.

Motivated by \cref{thm:moral-suffices}, we experimented on synthetic moral DAGs from \cite{wienobst2021polynomial}:
For each undirected chordal graph, we use the uniform sampling algorithm of \cite{wienobst2021polynomial} to uniformly sample 1000 moral DAGs $\wt{G}_1, \ldots, \wt{G}_{1000}$ and randomly choose one of them as $G^*$.
Then, we give $\{(\cE(G^*), \wt{G}_i)\}_{i \in [1000]}$ as input to \cref{alg:adaptive-search-with-advice-algo}.

\cref{fig:one-of-the-experiments} shows one of the experimental plots; more detailed experimental setup and results are given in \cref{sec:appendix-experiments}.
On the X-axis, we plot $\psi(G^*, \wt{V}) = \left| \rho \left( \wt{V}, N_{\skel(\cE(G^*))}^{h(G^*, \wt{V})}(\wt{V}) \right) \right|$, which is a \emph{lower bound} and proxy\footnote{We do not know if there is an efficient way to compute $\psi(G^*, \wt{G})$ besides the naive (possibly exponential time) enumeration over all possible minimum verifying sets.} for $\psi(G^*, \wt{G})$.
On the Y-axis, we aggregate advice DAGs based on their quality measure and also show (in dashed lines) the empirical distribution of quality measures of all DAGs within the Markov equivalence class.

As expected from our theoretical analyses, we see that the number of interventions by our advice search starts from $\nu_1(G^*)$, is lower than advice-free search of \cite{choo2022verification} when $\psi(G^*, \wt{V})$ is low, and gradually increases as the advice quality degrades.
Nonetheless, the number of interventions used is always theoretically bounded below $\cO(\psi(G^*, \wt{V}) \cdot \nu_1(G^*))$; we do not plot $\psi(G^*, \wt{V}) \cdot \nu_1(G^*)$ since plotting it yields a ``squashed'' graph as the empirical counts are significantly smaller.
In this specific graph instance, \cref{fig:one-of-the-experiments} suggests that our advice search outperforms its advice-free counterpart when given an advice DAG $\wt{G}$ that is better than $\sim 40\%$ of all possible DAGs consistent with the observational essential graph $\cE(G^*)$.

\begin{figure}[tb]
    \centering
    \includegraphics[width=0.7\linewidth]{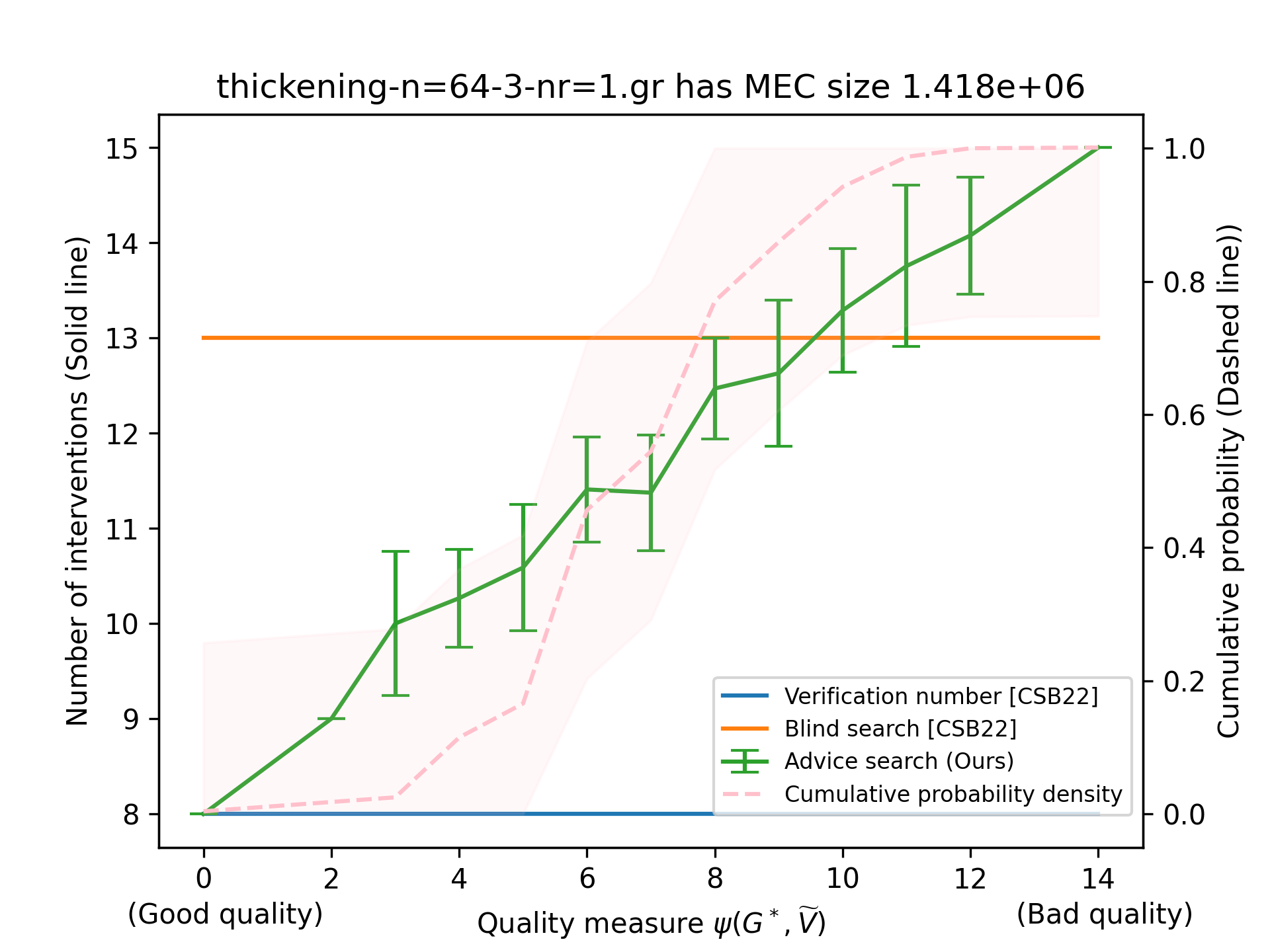}
    \caption{
    Experimental plot for one of the synthetic graphs $G^*$, with respect to $1000 \ll |[G^*]| \approx 1.4 \times 10^{6}$ uniformly sampled advice DAGs $\wt{G}$ from the MEC $[G^*]$. The solid lines indicate the number of atomic interventions used while the dotted lines indicate the empirical cumulative probability density of $\wt{G}$. The true cumulative probability density lies within the shaded area with probability at least 0.99 (see \cref{sec:appendix-experiments} for details).
    \vspace{-10pt}
    }
    \label{fig:one-of-the-experiments}
\end{figure}

\section{Conclusion and discussion}
\label{sec:conclusion}

In this work, we gave the first result that utilizes imperfect advice in the context of causal discovery.
We do so in a way that the performance (i.e.\ the number of interventions in our case) does not degrade significantly even when the advice is inaccurate, which is consistent with the objectives of learning-augmented algorithms. 
Specifically, we show a smooth bound that matches the number of interventions needed for verification of the causal relationships in a graph when the advice is completely accurate and also depends logarithmically on the distance of the advice to the ground truth.
This ensures robustness to ``bad'' advice, the number of interventions needed is asymptotically the same as in the case where no advice is available.    

Our results do rely on the widely-used assumptions of sufficiency and faithfulness as well as access to ideal iterventions; see \cref{sec:appendix-assumptions} for a more detailed discussion.
Since wrong causal conclusions may be drawn when these assumptions are violated by the data, thus it is of great interest to remove/weaken these assumptions while maintaining strong theoretical guarantees in future work.

\subsection{Interesting future directions to explore}

\paragraph{Partial advice}

In \cref{sec:partial-advice}, we explain why having a DAG $\wt{G}$ as advice may not always be possible and explain how to extend our results to the setting of \emph{partial advice} by considering the worst case DAG consistent with the given partial advice $\cA$.
The question is whether one can design and analyze a better algorithm than a trivial $\max_{\wt{G} \in \cA}$.
For example, maybe one could pick $\wt{G} = \argmin_{G \in \cA} \max_{H \in [G^*]} \psi(H, G)$?
The motivation is as follows: If $[G^*]$ is a disc in $\mathbb{R}^2$ and $\psi$ is the Euclidean distance, then $\wt{G}$ should be the point within $\cA$ that is closest to the center of the disc. Note that we can only optimize with respect to $\max_{H \in [G^*]}$ because we do not actually know $G^*$.
It remains to be seen if such an object can be efficiently computed and whether it gives a better bound than $\max_{\wt{G} \in \cA}$.

\paragraph{Incorporating expert confidence}

The notion of ``confidence level'' and ``correctness'' of an advice are orthogonal issues -- an expert can be confidently wrong.
In this work, we focused on the case where the expert is fully confident but may be providing imperfect advice.
It is an interesting problem to investigate how to principally handle both issues simultaneously; for example, what if the advice is not a DAG $\wt{G} \in [G^*]$ in the essential graph but a distribution over all DAGs in $[G^*]$?
Bayesian ideas may apply here.

\paragraph{Better analysis?}

Empirically, we see that the log factor is a rather loose upper bound both for blind search and advice search.
\emph{Can there be a tighter analysis?}
\cite{choo2022verification} tells us that $\Omega(\log n \cdot \nu_1(G^*))$ is unavoidable when $\cE(G^*)$ is a path on $n$ vertices with $\nu_1(G^*) = 1$ but this is a special class of graphs.
What if $\nu_1(G^*) > 1$?
Can we give tighter bounds in other graph parameters?
Furthermore, in some preliminary testing, we observed that implementing tweak 2 or ignoring it yield similar empirical performance and we wonder if there is a tighter analysis without tweak 2 that has similar guarantees.

\section*{Acknowledgements}
This research/project is supported by the National Research Foundation, Singapore under its AI Singapore Programme (AISG Award No: AISG-PhD/2021-08-013).
TG and AB are supported by the National Research Foundation Fellowship for AI (Award NRF-NRFFAI-0002), an Amazon Research Award, and a Google South \& Southeast Asia Research Award.
Part of this work was done while the authors were visiting the Simons Institute for the Theory of Computing.
We would like to thank Kirankumar Shiragur and Joy Qiping Yang for valuable feedback and discussions.

\bibliography{refs}
\bibliographystyle{alpha}

\newpage
\appendix

\section{Remark about assumptions}
\label{sec:appendix-assumptions}

Under \emph{causal sufficiency}, there are no hidden confounders (i.e.\ unobserved common causes to the observed variables).
While causal sufficiency may not always hold, it is still a reasonable assumption to make in certain applications such as studying gene regulatory networks (e.g.\ see \cite{wang2017permutation}).

\emph{Faithfulness} assumes that independencies that occur in the data do not occur due to ``cancellations'' in the functional relationships, but rather due to the causal graph structure.
It is known \cite{meek1995strong,spirtes2000causation} that, under many natural parameterizations and settings, the set of unfaithful parameters for any given causal DAG has zero Lebesgue measure (i.e.\ faithfulness holds; see also Section 3.2 of \cite{zhang2002strong} for a discussion about faithfulness).
However, one should be aware that the faithfulness assumption may be violated in reality \cite{andersen2013expect,uhler2013geometry}, especially in the presence of sampling errors in the finite sample regime.

\emph{Ideal interventions} assume hard interventions (forcefully setting a variable value) and the ability to obtain as many interventional samples as desired, ensuring that we always recover the directions of all edges cut by interventions.
Without this assumption, we may fail to correctly infer some arc directions and our algorithms will only succeed with some success probability.

Our assumption that the given expert advice is consistent with observational essential graph is purely for simplicity and can be removed by deciding which part of the given advice to discard so that the remaining advice is consistent.
However, we feel that deciding which part of the inconsistent advice to discard will unnecessarily complicate our algorithmic contributions without providing any useful insights, and thus we made such an assumption.
\section{Additional Preliminaries}\label{sec:appprelim}

For any set $A$, we denote its powerset by $2^A$.
We write $\{1, \ldots, n\}$ as $[n]$ and hide absolute constant multiplicative factors in $n$ using standard asymptotic notations $\cO(\cdot)$, $\Omega(\cdot)$, and $\Theta(\cdot)$.
The indicator function $\mathbbm{1}_{\text{predicate}}$ is 1 if the predicate is true and 0 otherwise.
Throughout, we use $G^*$ to denote the (unknown) ground truth DAG, its Markov equivalence class by $[G^*]$ and the corresponding essential graph by $\cE(G^*)$.
We write $A \,\dot\cup\, B$ and $A \setminus B$ to represent the disjoint union and set difference of two sets $A$ and $B$ respectively.

\subsection{Graph basics}

We consider partially oriented graphs without parallel edges.

Let $G = (V,E)$ be a graph on $|V| = n$ nodes/vertices where $V(G)$, $E(G)$, and $A(G) \subseteq E(G)$ denote nodes, edges, and arcs of $G$ respectively.
The graph $G$ is said to be fully oriented if $A(G) = E(G)$, fully unoriented if $A(G) = \emptyset$, and partially oriented otherwise.
For any subset $V' \subseteq V$ and $E' \subseteq E$, we use $G[V']$ and $G[E']$ to denote the node-induced and edge-induced subgraphs respectively.
We write $u \sim v$ to denote that two nodes $u,v \in V$ are connected in $G$, and write $u \to v$ or $u \gets v$ when specifying a certain direction.
The \emph{skeleton} $\skel(G)$ refers to the underlying graph where all edges are made undirected.
A \emph{v-structure} in $G$ refers to a collection of three distinct vertices $u,v,w \in V$ such that $u \to v \gets w$ and $u \not\sim w$.
A directed cycle refers to a sequence of $k \geq 3$ vertices where $v_1 \to v_2 \to \ldots \to v_k \to v_1$.
An \emph{acyclic completion / consistent extension} of a partially oriented graph refers to an assignment of edge directions to the unoriented edges $E(G) \setminus A(G)$ such that the resulting fully oriented graph has no directed cycles.

Suppose $G = (V,E)$ is fully unoriented.
For vertices $u,v \in V$, subset of vertices $V' \subseteq V$ and integer $r \geq 0$, define $\dist_{G}(u,v)$ as the shortest path length between $u$ and $v$, $\dist_{G}(V',v) = \min_{u \in V'} \dist_{G}(u,v)$, and $N_{G}^r(V') = \{v \in V: \dist_{G}(v,V') \leq r \} \subseteq V$ as the set of vertices that are $r$-hops away from $V'$, i.e.\ $r$-hop neighbors of $V'$.
We omit the subscript $G$ when it is clear from context.

Suppose $G = (V,E)$ is fully oriented.
For any vertex $v \in V$, we write $\Pa(v), \Anc(v), \Des(v)$ to denote the parents, ancestors and descendants of $v$ respectively and we write $\Des[v] = \Des(v) \cup \{v\}$ and $\Anc[v] = \Anc(v) \cup \{v\}$ to include $v$ itself.
We define $\Ch(v) \subseteq \Des(v)$ as the set of \emph{direct children} of $v$, that is, for any $w \in \Ch(v)$ there does \emph{not} exists $z \in V \setminus \{v,w\}$ such that $z \in \Des(v) \cap \Anc(w)$.
Note that, $\Ch(v) \subseteq \{w \in V : v \to w\} \subseteq \Des(v)$.

\subsection{Causal graph basics}
\label{sec:causal-graph-basics}

A directed acyclic graph (DAG) is a fully oriented graph without directed cycles.
By representing random variables by nodes, DAGs are commonly used as graphical causal models \cite{pearl2009causality}, where the joint probability density $f$ factorizes according to the Markov property: $f(v_1, \ldots, v_n) = \prod_{i=1}^n f(v_i \mid pa(v))$, where $pa(v)$ denotes the values taken by $v$'s parents.
One can associate a (not necessarily unique) \emph{valid permutation / topological ordering} $\pi : V \to [n]$ to any (partially directed) DAG such that oriented arcs $(u,v)$ satisfy $\pi(u) < \pi(v)$ and unoriented arcs $\{u,v\}$ can be oriented as $u \to v$ without forming directed cycles when $\pi(u) < \pi(v)$.

For any DAG $G$, we denote its Markov equivalence class (MEC) by $[G]$ and essential graph by $\cE(G)$.
DAGs in the same MEC have the same skeleton and the essential graph is a partially directed graph such that an arc $u \to v$ is directed if $u \to v$ in \emph{every} DAG in MEC $[G]$, and an edge $u \sim v$ is undirected if there exists two DAGs $G_1, G_2 \in [G]$ such that $u \to v$ in $G_1$ and $v \to u$ in $G_2$.
It is known that two graphs are Markov equivalent if and only if they have the same skeleton and v-structures \cite{verma1990,andersson1997characterization}.
In fact, the essential graph $\cE(G)$ can be computed from $G$ by orienting v-structures in the skeleton $\skel(G)$ and applying Meek rules (see \cref{sec:appendix-meek-rules}).
An edge $u \to v$ is a \emph{covered edge} \cite{chickering2013transformational} if $\Pa(u) = \Pa(v) \setminus \{u\}$.
We use $\cC(G) \subseteq E(G)$ to denote the set of covered edges of $G$.
The following is a well-known result relating covered edges and MECs.

\begin{lemma}[\cite{chickering2013transformational}]
\label{lem:sequence}
If $G$ and $G'$ belong in the same MEC if and only if there exists a sequence of covered edge reversals to transform between them.
\end{lemma}

\section{Additional Related Works on Causal Structure Learning}
\label{sec:apprelated}

Constraint-based algorithms, such as ours, use information about conditional independence relations to identify the underlying structure.
From purely observational data, the PC \cite{spirtes2000causation}, FCI \cite{spirtes2000causation} and RFCI algorithms \cite{colombo2012learning} have been shown to consistently recover the essential graph, assuming causal sufficiency, faithfulness, and i.i.d.\ samples.
The problem of recovering the DAG using constraints from interventional data was first studied by \cite{eberhardt2006n, eberhardt2005number, eberhardt2007causation}.
Many recent works \cite{hu2014randomized, shanmugam2015learning, kocaoglu2017cost, lindgren2018experimental,greenewald2019sample,squires2020active,choo2022verification, choo2023subset} have followed up on these themes.

Score-based methods maximize a particular score function over the space of graphs. For observational data, the GES algorithm \cite{chickering2002optimal} uses the BIC to iteratively add edges.
Extending the GES, \cite{hauser2012characterization} proposed the GIES algorithm that uses passive interventional data to orient more edges.
Hybrid methods, like \cite{solus2021consistency} for observational and \cite{wang2017permutation} for interventional data, use elements of both approaches. 

\section{Meek rules}
\label{sec:appendix-meek-rules}

Meek rules are a set of 4 edge orientation rules that are sound and complete with respect to any given set of arcs that has a consistent DAG extension \cite{meek1995}.
Given any edge orientation information, one can always repeatedly apply Meek rules till a unique fixed point (where no further rules trigger) to maximize the number of oriented arcs.

\begin{definition}[The four Meek rules \cite{meek1995}, see \cref{fig:meek-rules} for an illustration]
\hspace{0pt}
\begin{description}
    \item [R1] Edge $\{a,b\} \in E(G) \setminus A(G)$ is oriented as $a \to b$ if $\exists$ $c \in V$ such that $c \to a$ and $c \not\sim b$.
    \item [R2] Edge $\{a,b\} \in E(G) \setminus A(G)$ is oriented as $a \to b$ if $\exists$ $c \in V$ such that $a \to c \to b$.
    \item [R3] Edge $\{a,b\} \in E(G) \setminus A(G)$ is oriented as $a \to b$ if $\exists$ $c,d \in V$ such that $d \sim a \sim c$, $d \to b \gets c$, and $c \not\sim d$.
    \item [R4] Edge $\{a,b\} \in E(G) \setminus A(G)$ is oriented as $a \to b$ if $\exists$ $c,d \in V$ such that $d \sim a \sim c$, $d \to c \to b$, and $b \not\sim d$.
\end{description}
\end{definition}

\begin{figure}[htbp]
\centering
\resizebox{\linewidth}{!}{%
\begin{tikzpicture}
%
%
\node[draw, circle, inner sep=2pt] at (0,0) (R1a-before) {\small $a$};
\node[draw, circle, inner sep=2pt, right=of R1a-before] (R1b-before) {\small $b$};
\node[draw, circle, inner sep=2pt, above=of R1a-before](R1c-before) {\small $c$};
\draw[thick, -stealth] (R1c-before) -- (R1a-before);
\draw[thick] (R1a-before) -- (R1b-before);

\node[draw, circle, inner sep=2pt] at (3,0) (R1a-after) {\small $a$};
\node[draw, circle, inner sep=2pt, right=of R1a-after] (R1b-after) {\small $b$};
\node[draw, circle, inner sep=2pt, above=of R1a-after](R1c-after) {\small $c$};
\draw[thick, -stealth] (R1c-after) -- (R1a-after);
\draw[thick, -stealth] (R1a-after) -- (R1b-after);

\node[single arrow, draw, minimum height=2em, single arrow head extend=1ex, inner sep=2pt] at (2.2,0.75) (R1arrow) {};
\node[above=5pt of R1arrow] {\footnotesize R1};

%
%
\node[draw, circle, inner sep=2pt] at (6,0) (R2a-before) {\small $a$};
\node[draw, circle, inner sep=2pt, right=of R2a-before] (R2b-before) {\small $b$};
\node[draw, circle, inner sep=2pt, above=of R2a-before](R2c-before) {\small $c$};
\draw[thick, -stealth] (R2a-before) -- (R2c-before);
\draw[thick, -stealth] (R2c-before) -- (R2b-before);
\draw[thick] (R2a-before) -- (R2b-before);

\node[draw, circle, inner sep=2pt] at (9,0) (R2a-after) {\small $a$};
\node[draw, circle, inner sep=2pt, right=of R2a-after] (R2b-after) {\small $b$};
\node[draw, circle, inner sep=2pt, above=of R2a-after](R2c-after) {\small $c$};
\draw[thick, -stealth] (R2a-after) -- (R2c-after);
\draw[thick, -stealth] (R2c-after) -- (R2b-after);
\draw[thick, -stealth] (R2a-after) -- (R2b-after);

\node[single arrow, draw, minimum height=2em, single arrow head extend=1ex, inner sep=2pt] at (8.2,0.75) (R2arrow) {};
\node[above=5pt of R2arrow] {\footnotesize R2};

%
%
\node[draw, circle, inner sep=2pt] at (12,0) (R3d-before) {\small $d$};
\node[draw, circle, inner sep=2pt, above=of R3d-before](R3a-before) {\small $a$};
\node[draw, circle, inner sep=2pt, right=of R3a-before] (R3c-before) {\small $c$};
\node[draw, circle, inner sep=2pt, right=of R3d-before](R3b-before) {\small $b$};
\draw[thick, -stealth] (R3c-before) -- (R3b-before);
\draw[thick, -stealth] (R3d-before) -- (R3b-before);
\draw[thick] (R3c-before) -- (R3a-before) -- (R3d-before);
\draw[thick] (R3a-before) -- (R3b-before);

\node[draw, circle, inner sep=2pt] at (15,0) (R3d-after) {\small $d$};
\node[draw, circle, inner sep=2pt, above=of R3d-after](R3a-after) {\small $a$};
\node[draw, circle, inner sep=2pt, right=of R3a-after] (R3c-after) {\small $c$};
\node[draw, circle, inner sep=2pt, right=of R3d-after](R3b-after) {\small $b$};
\draw[thick, -stealth] (R3c-after) -- (R3b-after);
\draw[thick, -stealth] (R3d-after) -- (R3b-after);
\draw[thick] (R3c-after) -- (R3a-after) -- (R3d-after);
\draw[thick, -stealth] (R3a-after) -- (R3b-after);

\node[single arrow, draw, minimum height=2em, single arrow head extend=1ex, inner sep=2pt] at (14.2,0.75) (R3arrow) {};
\node[above=5pt of R3arrow] {\footnotesize R3};

%
%
\node[draw, circle, inner sep=2pt] at (18,0) (R4a-before) {\small $a$};
\node[draw, circle, inner sep=2pt, above=of R4a-before](R4d-before) {\small $d$};
\node[draw, circle, inner sep=2pt, right=of R4d-before] (R4c-before) {\small $c$};
\node[draw, circle, inner sep=2pt, right=of R4a-before](R4b-before) {\small $b$};
\draw[thick, -stealth] (R4d-before) -- (R4c-before);
\draw[thick, -stealth] (R4c-before) -- (R4b-before);
\draw[thick] (R4d-before) -- (R4a-before) -- (R4c-before);
\draw[thick] (R4a-before) -- (R4b-before);

\node[draw, circle, inner sep=2pt] at (21,0) (R4a-after) {\small $a$};
\node[draw, circle, inner sep=2pt, above=of R4a-after](R4d-after) {\small $d$};
\node[draw, circle, inner sep=2pt, right=of R4d-after] (R4c-after) {\small $c$};
\node[draw, circle, inner sep=2pt, right=of R4a-after](R4b-after) {\small $b$};
\draw[thick, -stealth] (R4d-after) -- (R4c-after);
\draw[thick, -stealth] (R4c-after) -- (R4b-after);
\draw[thick] (R4d-after) -- (R4a-after) -- (R4c-after);
\draw[thick, -stealth] (R4a-after) -- (R4b-after);

\node[single arrow, draw, minimum height=2em, single arrow head extend=1ex, inner sep=2pt] at (20.2,0.75) (R4arrow) {};
\node[above=5pt of R4arrow] {\footnotesize R4};

\draw[thick] (5.25,1.75) -- (5.25,-0.25);
\draw[thick] (11.25,1.75) -- (11.25,-0.25);
\draw[thick] (17.25,1.75) -- (17.25,-0.25);
\end{tikzpicture}
}
\caption{An illustration of the four Meek rules}
\label{fig:meek-rules}
\end{figure}
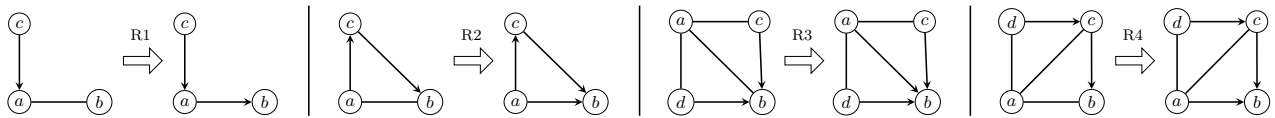

There exists an algorithm (Algorithm 2 of \cite{wienobst2021extendability}) that runs in $\cO(d \cdot | E(G) |)$ time and computes the closure under Meek rules, where $d$ is the degeneracy of the graph skeleton\footnote{A $d$-degenerate graph is an undirected graph in which every subgraph has a vertex of degree at most $d$. Note that the degeneracy of a graph is typically smaller than the maximum degree of the graph.}.

\section{Imperfect \emph{partial} advice via MPDAGs}
\label{sec:partial-advice}

In the previous sections, we discuss advice that occurs in the form of a DAG $\wt{G} \in [G^*]$.
However, this may be too much to ask for in certain situations.
For example:
\begin{itemize}
    \item The Markov equivalence class may be too large for an expert to traverse through and propose an advice DAG.
    \item The expert only has opinions about a subset of a very large causal graph involving millions of nodes / edges.
\end{itemize}

As discussed in \cref{sec:expert-advice-prelim}, we can formulate such partial advice as MPDAGs.
Given a MPDAG as expert advice, a natural attempt would be to sample a DAG $\wt{G}$ from it to use the full advice.
Unfortunately, it is \#P-complete even to count the number of DAGs consistent with a given MPDAG in general \cite{wienobst2021polynomial} and we are unaware of any efficient way to sample uniformly at random from it.
Instead, we propose to pick an arbitrary DAG $\wt{G}$ as advice within the given MPDAG: pick any unoriented edge, orient arbitrarily, apply Meek rules, repeat until fully oriented.
The following result follows naturally by maximizing over all possible DAGs consistent with the given partial advice.

\begin{restatable}{mytheorem}{searchwithpartialadvice}
\label{thm:search-with-partial-advice}
Fix an essential graph $\cE(G^*)$ with an unknown underlying ground truth DAG $G^*$.
Given a set $\cA$ of DAGs consistent with the given partial advice and intervention set bound $k \geq 1$, there exists a deterministic polynomial time algorithm that computes an intervention set $\cI$ adaptively such that $\cE_{\cI}(G^*) = G^*$, and $|\cI|$ has size\\
1. $\cO( \max\{1, \log \max_{\wt{G} \in \cA} \psi(G^*, \wt{G}) \} \cdot \nu_1(G^*))$\\
2. $\cO( \max\{1, \log \max_{\wt{G} \in \cA} \psi(G^*, \wt{G}) \} \cdot \log k \cdot \nu_k(G^*))$\\
when $k = 1$ and $k > 1$ respectively.
\end{restatable}

\section{Technical Overview for \cref{thm:two-approx-ratio}}
\label{sec:two-overview}

As discussed in \cref{sec:preliminaries}, it suffices to prove \cref{thm:two-approx-ratio} with respect to moral DAGs.

Our strategy for proving \cref{thm:two-approx-ratio} is to consider two arbitrary DAGs $G_s$ (source) and $G_t$ (target) in the same equivalence class and transform a verifying set for $G_s$ into a verifying set for $G_t$ using \cref{lem:sequence} (see \cref{alg:repeatedfindedge} for the explicit algorithm\footnote{Lemma 2 of \cite{chickering2013transformational} guarantees that $x \to y$ is a covered edge of the current $G_s$ whenever step 9 is executed.}).
Instead of proving \cref{thm:two-approx-ratio} by analyzing the exact sequence of covered edges produced by \cref{alg:repeatedfindedge}\footnote{The correctness of \cref{alg:repeatedfindedge} is given in \cite{chickering2013transformational} where the key idea is to show that $x \to y$ found in this manner is a covered edge. This is proven in Lemma 2 of \cite{chickering2013transformational}.} when transforming between the DAGs $G_{\min} = \argmin_{G \in [G^*]} \nu_1(G)$ and $G_{\max} = \argmax_{G \in [G^*]} \nu_1(G)$, we will prove something more general.

\begin{algorithm}[ht]
\caption{\cite{chickering2013transformational}: Transforms between two DAGs within the same MEC via covered edge reversals}
\label{alg:repeatedfindedge}
\begin{algorithmic}[1]
    \Statex \textbf{Input}: Two DAGs $G_s = (V, E_s)$ and $G_t = (V, E_t)$
    \Statex \textbf{Output}: A sequence $\texttt{seq}$ of covered edge reversals that transforms $G_s$ to $G_t$
    \State $\texttt{seq} \gets \emptyset$
    \While{$G_s \neq G_t$}
        \State Fix an arbitrary valid ordering $\pi$ for $G_s$.
        \State Let $A \gets A(G_s) \setminus A(G_t)$ be the set of differing arcs.
        \State Let $y \gets \argmin_{y \;\in\; V \;:\; \Pa_{A}(y) \neq \emptyset} \{ \pi(y) \}$.
        \State Let $x \gets \argmax_{z \;\in\; \Pa_{A}(y)} \{ \pi(z) \}$.
        \State Add $x \to y$ to $\texttt{seq}$. \Comment{\cite[Lemma 2]{chickering2013transformational}: $x \to y \in \cC(G_s)$}
        \State Update $G_s$ by replacing $x \to y$ with $y \to x$.
    \EndWhile
	\State \textbf{return} $\texttt{seq}$
\end{algorithmic}
\end{algorithm}

Observe that taking both endpoints of any maximal matching of covered edges is a valid verifying set that is at most \emph{twice} the size of the minimum verifying set.
This is because maximal matching is a 2-approximation to the minimum vertex cover.
Motivated by this observation, our proof for \cref{thm:two-approx-ratio} uses the following transformation argument (\cref{lem:same-mm-size}): for two DAGs $G$ and $G'$ that differ only on the arc direction of a single covered edge $x \sim y$, we show that given a conditional-root-greedy (CRG) maximal matching\footnote{A special type of maximal matching (see \cref{def:greedymm}).} on the covered edges of $G$, we can obtain another CRG maximal matching \emph{of the same size} on the covered edges of $G'$, after reversing $x \sim y$ and transforming $G$ to $G'$.

So, starting from $G_s$, we compute a CRG maximal matching, then we apply the transformation argument above on the sequence of covered edges given by \cref{alg:repeatedfindedge} until we get a CRG maximal matching of $G_t$ \emph{of the same size}.
Thus, we can conclude that the minimum vertex cover sizes of $G_s$ and $G_t$ differ by a factor of at most two.
This argument holds for \emph{any} pair of DAGs $(G_s, G_t)$ from the same MEC.

We now define what is a conditional-root-greedy (CRG) maximal matching.
As the set of covered edges $\cC(G)$ of any DAG $G$ induces a forest (see \cref{thm:verification-characterization}), we define the CRG maximal matching using a particular greedy process on the tree structure of $\cC(G)$.
The CRG maximal matching is unique with respect to a fixed valid ordering $\pi$ of $G$ and subset $S$.
We will later consider CRG maximal matchings with $S = A(G_s) \cap A(G_t)$, where the arc set $S$ \emph{remains unchanged throughout the entire transformation process}.

\begin{definition}[Conditional-root-greedy (CRG) maximal matching]
\label{def:greedymm}
Given a DAG $G = (V,E)$ with a valid ordering $\pi_{G}$ and a subset of edges $S \subseteq E$, we define the conditional-root-greedy (CRG) maximal matching $M_{G, \pi_{G}, S}$ as the \emph{unique} maximal matching on $\cC(G)$ computed via \cref{alg:condition-root-greedy}: greedily choose arcs $x \to y$ where the $x$ has no incoming arcs by minimizing $\pi_{G}(y)$, conditioned on \emph{favoring arcs outside of $S$}.
\end{definition}

\begin{algorithm}[htbp]
\caption{Conditional-root-greedy maximal matching}
\label{alg:condition-root-greedy}
\begin{algorithmic}[1]
    \Statex \textbf{Input}: A DAG $G = (V,E)$, a valid ordering $\pi_{G}$, a subset of edges $S \subseteq E$
    \Statex \textbf{Output}: A CRG maximal matching $M_{G, \pi_{G}, S}$
    \State Initialize $M_{G, \pi_{G}, S} \gets \emptyset$ and $C \gets \cC(G)$
    \While{$C \neq \emptyset$}
        \State $x \gets \argmin_{z \;\in\; \{ u \in V \;\mid\; u \to v \in C \}} \{ \pi_{G}(z) \}$ \Comment{$x$ is a root (i.e.\ no incoming arcs)}
        \State $y \gets \argmin_{z \in V \;:\; x \to z \;\in\; C} \{ \pi_{G}(z) + n^2 \cdot \mathbbm{1}_{x \to z \in S} \}$
        \State Add the arc $x \to y$ to $M_{G, \pi_{G}, S}$
        \State Remove all arcs with $x$ or $y$ as endpoints from $C$
    \EndWhile
	\State \textbf{return} $M_{G, \pi_{G}, S}$
\end{algorithmic}
\end{algorithm}

To prove the transformation argument (\cref{lem:same-mm-size}), we need to first understand how the status of covered edges evolve when we perform a single edge reversal.
The following lemma may be of independent interest beyond this work.

\begin{restatable}[Covered edge status changes due to covered edge reversal]{mylemma}{statuschanges}
\label{lem:status-changes}
Let $G^*$ be a moral DAG with MEC $[G^*]$ and consider any DAG $G \in [G^*]$.
Suppose $G = (V,E)$ has a covered edge $x \to y \in \cC(G) \subseteq E$ and we reverse $x \to y$ to $y \to x$ to obtain a new DAG $G' \in [G^*]$.
Then, all of the following statements hold:
\begin{enumerate}
    \item $y \to x \in \cC(G')$. Note that this is the covered edge that was reversed.
    \item If an edge $e$ does not involve $x$ or $y$, then $e \in \cC(G)$ if and only if $e \in \cC(G')$.
    \item If $x \in \Ch_{G}(a)$ for some $a \in V \setminus \{x,y\}$, then $a \to x \in \cC(G)$ if and only if $a \to y \in \cC(G')$.
    \item If $b \in \Ch_{G}(y)$ and $x \to b \in E(G)$ for some $b \in V \setminus \{x,y\}$, then $y \to b \in \cC(G)$ if and only if $x \to b \in \cC(G')$.
\end{enumerate}
\end{restatable}

Using \cref{lem:status-changes}, we derive our transformation argument.

\begin{restatable}{mylemma}{samemmsize}
\label{lem:same-mm-size}
Consider two moral DAGs $G_1$ and $G_2$ from the same MEC such that they differ only in one covered edge direction: $x \to y \in E(G_1)$ and $y \to x \in E(G_2)$.

Let vertex $a$ be the direct parent of $x$ in $G_1$, if it exists.
Let $S \subseteq E$ be a subset such that $a \to x \in S$ and $x \to y, y \to x \not\in S$ (if $a$ does not exist, ignore condition $a \to x \in S$).

Suppose $\pi_{G_1}$ is an ordering for $G_1$ such that $y = \argmin_{z \;:\; x \to z \in \cC(G_1)} \{ \pi_{G_1}(z) + n^2 \cdot \mathbbm{1}_{x \to z \in S} \}$ and denote $M_{G_1, \pi_{G_1}, S}$ as the corresponding CRG maximal matching for $\cC(G_1)$.
Then, there exists an explicit modification of $\pi_{G_1}$ to $\pi_{G_2}$, and $M_{G_1, \pi_{G_1}, S}$ to a CRG maximal matching $M_{G_2, \pi_{G_2}, S}$ for $\cC(G_2)$ such that $|M_{G_1, \pi_{G_1}, S}| = |M_{G_2, \pi_{G_2}, S}|$.
\end{restatable}

To be precise, given $\pi_{G_1}$, we will define $\pi_{G_2}$ in our proofs as follows:
\begin{equation}
\label{eq:pi-g2}
\pi_{G_2}(v) = 
\begin{cases}
\pi_{G_1}(x) & \text{if $v = y$}\\
\pi_{G_1}(u) & \text{if $v = x$}\\
\pi_{G_1}(y) & \text{if $v = u$}\\
\pi_{G_1}(v) & \text{else}
\end{cases}
\end{equation}

As discussed earlier, \cref{thm:two-approx-ratio} follows by picking $G_s = \argmax_{G \in [G^*]} \nu_1(G)$ and $G_t = \argmin_{G \in [G^*]} \nu_1(G)$, applying \cref{alg:repeatedfindedge} to find a transformation sequence of covered edge reversals between them, and repeatedly applying \cref{lem:same-mm-size} with the conditioning set $S = A(G_s) \cap A(G_t)$ to conclude that $G_s$ and $G_t$ have the same sized CRG maximal matchings, and thus implying that $\min_{G \in [G^*]} \nu_1(G) = \nu_1(G_s) \leq 2 \cdot \nu_1(G_t) = 2 \cdot \argmax_{G \in [G^*]} \nu_1(G)$.
Note that we keep the conditioning set $S$ \emph{unchanged throughout the entire transformation process} from $G_s$ to $G_t$.

For an illustrated example of conditional-root-greedy (CRG) maximal matchings and how we update the permutation ordering, see \cref{fig:two-approx-example} and \cref{fig:two-approx-example-without-a}.

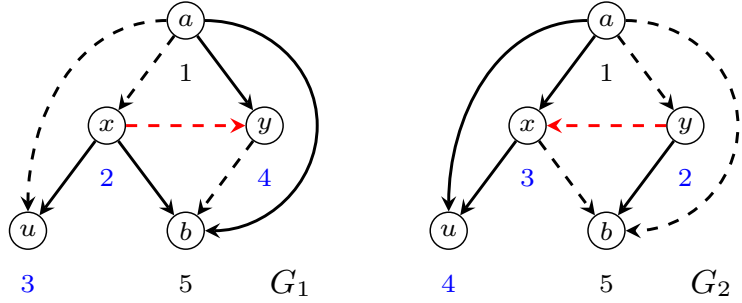
\begin{figure}[htb]
\centering
\resizebox{0.6\linewidth}{!}{%
\begin{tikzpicture}
%
%
\node[draw, circle, minimum size=10pt, inner sep=0pt] at (0,0) (a-g1) {\scriptsize $a$};
\node[draw, circle, minimum size=10pt, inner sep=0pt] at ($(a-g1) + (-0.75,-1)$) (x-g1) {\scriptsize $x$};
\node[draw, circle, minimum size=10pt, inner sep=0pt] at ($(a-g1) + (0.75,-1)$) (y-g1) {\scriptsize $y$};
\node[draw, circle, minimum size=10pt, inner sep=0pt] at ($(a-g1) + (0,-2)$) (b-g1) {\scriptsize $b$};
\node[draw, circle, minimum size=10pt, inner sep=0pt] at ($(b-g1) + (-1.5,0)$) (u-g1) {\scriptsize $u$};

\draw[thick, -stealth, dashed] (a-g1) -- (x-g1);
\draw[thick, -stealth] (a-g1) -- (y-g1);
\draw[thick, -stealth, dashed] (a-g1) to[out=180,in=90] (u-g1);
\draw[thick, -stealth] (a-g1) to[out=0,in=90] ($(y-g1) + (0.5,0)$) to[out=270,in=0] (b-g1);
\draw[thick, -stealth] (x-g1) -- (b-g1);
\draw[red, thick, -stealth, dashed] (x-g1) -- (y-g1);
\draw[thick, -stealth] (x-g1) -- (u-g1);
\draw[thick, -stealth, dashed] (y-g1) -- (b-g1);

\node[] at ($(a-g1) + (0,-0.5)$) {\scriptsize 1};
\node[blue] at ($(x-g1) + (0,-0.5)$) {\scriptsize 2};
\node[blue] at ($(y-g1) + (0,-0.5)$) {\scriptsize 4};
\node[blue] at ($(u-g1) + (0,-0.5)$) {\scriptsize 3};
\node[] at ($(b-g1) + (0,-0.5)$) {\scriptsize 5};

\node[] at ($(y-g1) + (0.25,-1.5)$) {\small $G_1$};

%
%
\node[draw, circle, minimum size=10pt, inner sep=0pt] at (4,0) (a-g2) {\scriptsize $a$};
\node[draw, circle, minimum size=10pt, inner sep=0pt] at ($(a-g2) + (-0.75,-1)$) (x-g2) {\scriptsize $x$};
\node[draw, circle, minimum size=10pt, inner sep=0pt] at ($(a-g2) + (0.75,-1)$) (y-g2) {\scriptsize $y$};
\node[draw, circle, minimum size=10pt, inner sep=0pt] at ($(a-g2) + (0,-2)$) (b-g2) {\scriptsize $b$};
\node[draw, circle, minimum size=10pt, inner sep=0pt] at ($(b-g2) + (-1.5,0)$) (u-g2) {\scriptsize $u$};

\draw[thick, -stealth] (a-g2) -- (x-g2);
\draw[thick, -stealth, dashed] (a-g2) -- (y-g2);
\draw[thick, -stealth] (a-g2) to[out=180,in=90] (u-g2);
\draw[thick, -stealth, dashed] (a-g2) to[out=0,in=90] ($(y-g2) + (0.5,0)$) to[out=270,in=0] (b-g2);
\draw[thick, -stealth, dashed] (x-g2) -- (b-g2);
\draw[thick, -stealth] (x-g2) -- (u-g2);
\draw[thick, -stealth] (y-g2) -- (b-g2);
\draw[red, thick, -stealth, dashed] (y-g2) -- (x-g2);

\node[] at ($(a-g2) + (0,-0.5)$) {\scriptsize 1};
\node[blue] at ($(x-g2) + (0,-0.5)$) {\scriptsize 3};
\node[blue] at ($(y-g2) + (0,-0.5)$) {\scriptsize 2};
\node[blue] at ($(u-g2) + (0,-0.5)$) {\scriptsize 4};
\node[] at ($(b-g2) + (0,-0.5)$) {\scriptsize 5};

\node[] at ($(y-g2) + (0.25,-1.5)$) {\small $G_2$};
\end{tikzpicture}
}
\caption{
Consider the following simple setup of two DAGs $G_1$ and $G_2$ which agree on all arc directions except for $x \to y$ in $G_1$ and $y \to x$ in $G_2$.
Dashed arcs represent the covered edges in each DAG.
The numbers below each vertex indicate the $\pi_{G_1}$ and $\pi_{G_2}$ orderings respectively.
In $G_1$, $u = \argmin_{z \in \Ch_{G_1}(x)} \{ \pi_{G_1}(z) \}$.
Observe that \cref{eq:pi-g2} modifies the ordering only for $\{x,y,u\}$ (in blue) while keeping the ordering of all other vertices fixed.
Suppose $S = A(G_1) \cap A(G_2) = \{a \to b, a \to x, a \to y, a \to u, x \to b, x \to u, y \to b\}$.
With respect to $\pi_{G_1}$ and $S$,
The conditional-root-greedy maximal matchings (see \cref{alg:condition-root-greedy}) are $M_{G_1, \pi_{G_1}, S} = \{a \to x, y \to b\}$ and $M_{G_2, \pi_{G_2}, S} = \{a \to y, x \to b\}$.
}
\label{fig:two-approx-example}
\end{figure}

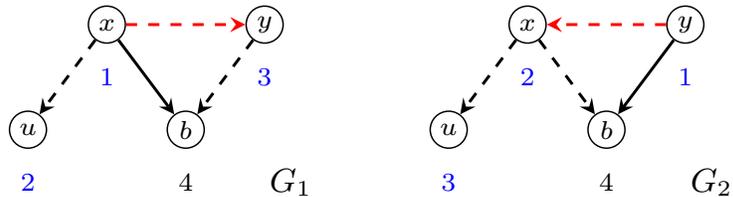
\begin{figure}[htb]
\centering
\resizebox{0.6\linewidth}{!}{%
\begin{tikzpicture}
%
%
\node[minimum size=10pt, inner sep=0pt] at (0,0) (a-g1) {};
\node[draw, circle, minimum size=10pt, inner sep=0pt] at ($(a-g1) + (-0.75,0)$) (x-g1) {\scriptsize $x$};
\node[draw, circle, minimum size=10pt, inner sep=0pt] at ($(a-g1) + (0.75,0)$) (y-g1) {\scriptsize $y$};
\node[draw, circle, minimum size=10pt, inner sep=0pt] at ($(a-g1) + (0,-1)$) (b-g1) {\scriptsize $b$};
\node[draw, circle, minimum size=10pt, inner sep=0pt] at ($(b-g1) + (-1.5,0)$) (u-g1) {\scriptsize $u$};

\draw[thick, -stealth] (x-g1) -- (b-g1);
\draw[red, thick, -stealth, dashed] (x-g1) -- (y-g1);
\draw[thick, -stealth, dashed] (x-g1) -- (u-g1);
\draw[thick, -stealth, dashed] (y-g1) -- (b-g1);

\node[blue] at ($(x-g1) + (0,-0.5)$) {\scriptsize 1};
\node[blue] at ($(y-g1) + (0,-0.5)$) {\scriptsize 3};
\node[blue] at ($(u-g1) + (0,-0.5)$) {\scriptsize 2};
\node[] at ($(b-g1) + (0,-0.5)$) {\scriptsize 4};

\node[] at ($(y-g1) + (0.25,-1.5)$) {\small $G_1$};

%
%
\node[minimum size=10pt, inner sep=0pt] at (4,0) (a-g2) {};
\node[draw, circle, minimum size=10pt, inner sep=0pt] at ($(a-g2) + (-0.75,0)$) (x-g2) {\scriptsize $x$};
\node[draw, circle, minimum size=10pt, inner sep=0pt] at ($(a-g2) + (0.75,0)$) (y-g2) {\scriptsize $y$};
\node[draw, circle, minimum size=10pt, inner sep=0pt] at ($(a-g2) + (0,-1)$) (b-g2) {\scriptsize $b$};
\node[draw, circle, minimum size=10pt, inner sep=0pt] at ($(b-g2) + (-1.5,0)$) (u-g2) {\scriptsize $u$};

\draw[thick, -stealth, dashed] (x-g2) -- (b-g2);
\draw[thick, -stealth, dashed] (x-g2) -- (u-g2);
\draw[thick, -stealth] (y-g2) -- (b-g2);
\draw[red, thick, -stealth, dashed] (y-g2) -- (x-g2);

\node[blue] at ($(x-g2) + (0,-0.5)$) {\scriptsize 2};
\node[blue] at ($(y-g2) + (0,-0.5)$) {\scriptsize 1};
\node[blue] at ($(u-g2) + (0,-0.5)$) {\scriptsize 3};
\node[] at ($(b-g2) + (0,-0.5)$) {\scriptsize 4};

\node[] at ($(y-g2) + (0.25,-1.5)$) {\small $G_2$};
\end{tikzpicture}
}
\caption{
Consider the following simple setup of two DAGs $G_3$ and $G_4$ which agree on all arc directions except for $x \to y$ in $G_3$ and $y \to x$ in $G_4$.
Dashed arcs represent the covered edges in each DAG.
The numbers below each vertex indicate the $\pi_{G_3}$ and $\pi_{G_4}$ orderings respectively.
Observe that $\cC(G_3) = \{ x \to u, x \to y, y \to b\}$.
If we define $S = A(G_3) \cap A(G_4) = \{x \to b, x \to u, y \to b\}$, we see that the conditional-root-greedy maximal matchings (see \cref{alg:condition-root-greedy}) are $M_{G_3, \pi_{G_3}, S} = \{x \to y\}$ and $M_{G_4, \pi_{G_4}, S} = \{y \to x\}$.
Note that \cref{alg:condition-root-greedy} does \emph{not} choose $x \to u \in \cC(G_1)$ despite $\pi(u) < \pi(y)$ because $x \to u \in S$, so $\pi(y) < \pi(u) + n^2$.
}
\label{fig:two-approx-example-without-a}
\end{figure}

\section{Deferred proofs}

\subsection{Preliminaries}

Our proofs rely on some existing results which we first state and explain below.

\begin{lemma}[Lemma 27 of \cite{choo2022verification}]
\label{lem:necessary}
Fix an essential graph $\cE(G^*)$ and $G \in [G^*]$.
If $\cI \subseteq 2^V$ is a verifying set, then $\cI$ separates all unoriented covered edge $u \sim v$ of $G$.
\end{lemma}

\begin{lemma}[Lemma 28 of \cite{choo2022verification}]
\label{lem:sufficient}
Fix an essential graph $\cE(G^*)$ and $G \in [G^*]$.
If $\cI \subseteq 2^V$ is an intervention set that separates every unoriented covered edge $u \sim v$ of $G$, then $\cI$ is a verifying set.
\end{lemma}

\cref{lem:necessary} tells us that we have to intervene on one of the endpoints of \emph{any} covered edge in order to orient it while \cref{lem:sufficient} tells us that doing so for all covered edges suffices to orient the entire causal DAG.

\subsection{Verification numbers of DAGs within same MEC are bounded by a factor of two}

We use the following simple lemma in our proof of \cref{lem:status-changes}.

\begin{restatable}{lemma}{coveredisdirect}
\label{lem:covered-is-direct}
For any covered edge $x \to y$ in a DAG $G = (V,E)$, we have $y \in \Ch_G(x)$.
Furthermore, each vertex only appears as an endpoint in the collection of covered edges $\cC(G)$ at most once.
\end{restatable}
\begin{proof}
For the first statement, suppose, for a contradiction, that $y \not\in \Ch(x)$.
Then, there exists some $z \in V \setminus \{x,y\}$ such that $z \in \Des(x) \cap \Anc(y)$.
Fix an arbitrary ordering $\pi$ for $G$ and let $z^* = \argmax_{z \in \Des(x) \cap \Anc(y)} \{\pi(z)\}$.
Then, we see that $z^* \to y$ while $z^* \not\to x$ since $z^* \in \Des(x)$.
So, $x \to y$ \emph{cannot} be a covered edge.
Contradiction.

For the second statement, suppose, for a contradiction, that there are two covered edges $u \to x, v \to x \in \cC(G)$ that ends with $x$.
Since $u \to x \in \cC(G)$, we must have $v \to u$.
Since $v \to x \in \cC(G)$, we must have $u \to v$.
We cannot have both $u \to v$ and $v \to u$ simultaneously.
Contradiction.
\end{proof}

\statuschanges*
\begin{proof}
The only parental relationships that changed when we reversing $x \to y$ to $y \to x$ are $\Pa_{G'}(y) = \Pa_{G}(y) \setminus \{x\}$ and $\Pa_{G'}(x) = \Pa_{G}(x) \cup \{y\}$.
For any other vertex $u \in V \setminus \{x,y\}$, we have $\Pa_{G'}(u) = \Pa_{G}(u)$.
The first two points have the same proof: as parental relationships of both endpoints are unchanged, the covered edge status is unchanged.
\begin{enumerate}
    \setcounter{enumi}{2}
    \item Since $x \to y \in \cC(G)$, we have $a \to y \in E(G)$.
    We prove both directions separately.
    
    Suppose $a \to x \in \cC(G)$.
    Then, $\Pa_{G}(a) = \Pa_{G}(x) \setminus \{a\}$.
    Since $x \to y \in \cC(G)$, then $\Pa_{G}(x) = \Pa_{G}(y) \setminus \{x\}$.
    So, we have $\Pa_{G'}(a) = \Pa_{G}(a) = \Pa_{G}(x) \setminus \{a\} = \Pa_{G}(y) \setminus \{x,a\} = \Pa_{G'}(y) \setminus \{a\}$.
    Thus, $a \to y \in \cC(G')$.
    
    Suppose $a \to x \not\in \cC(G)$.
    Then, one of the two cases must occur:
    \begin{enumerate}
        \item There exists some vertex $u$ such that $u \to a$ and $u \not\to x$ in $G$.\\
        Since $x \to y$ is a covered edge, $u \not\to x$ implies $u \not\to y$ in $G$.
        Therefore, $a \to y \not\in \cC(G')$ due to $u \to a$.
        \item There exists some vertex $v$ such that $v \to x$ and $v \not\to a$ in $G$.\\
        There are two possibilities for $v \not\to a$: $v \not\sim a$ or $v \gets a$.
        If $v \not\sim a$, then  $v \to x \gets a$ is a v-structure.
        If $v \gets a$, then $x \not\in \Ch(a)$ since we have $a \to v \to x$.
        Both possibilities lead to contradictions.
    \end{enumerate}
    The first case implies $a \to y \not\in \cC(G')$ while the second case cannot happen.
    
    \item We prove both directions separately.
    
    Suppose $y \to b \in \cC(G)$.
    Then, $\Pa_{G}(y) = \Pa_{G}(b) \setminus \{y\}$.
    Since $x \to y \in \cC(G)$, then $\Pa_{G}(x) = \Pa_{G}(y) \setminus \{x\}$.
    So, we have $\Pa_{G'}(b) \setminus \{x\} = \Pa_{G}(b) \setminus \{x\} = \Pa_{G}(y) \cup \{y\} \setminus \{x\} = \Pa_{G}(x) \cup \{y\} = \Pa_{G'}(x)$.
    Thus, $x \to b \in \cC(G')$.
    
    Suppose $y \to b \not\in \cC(G)$.
    Then, one of the two cases must occur:
    \begin{itemize}
        \item There exists some vertex $u \to y$ and $u \not\to b$.\\
        Since $x \to y$ is a covered edge, $u \to y$ implies $u \to x$.
        Therefore, $x \to b \not\in \cC(G')$ due to $u \not\to b$.
        \item There exists some vertex $v \to b$ and $v \not\to y$.\\
        There are two possibilities for $v \not\to y$: $v \not\sim y$ or $v \gets y$.
        If $v \not\sim y$, then $v \to b \gets y$ is a v-structure.
        If $v \gets y$, then $b \not\in \Ch(y)$ since we have $y \to v \to b$.
        Both possibilities lead to contradictions.
    \end{itemize}
    The first case implies $x \to b \not\in \cC(G')$ while the second case cannot happen. \qedhere
\end{enumerate}
\end{proof}

\samemmsize*
\begin{proof}
Define $u = \argmin_{z \in \Ch_{G_1}(x)} \{ \pi_{G_1}(z) \}$ as the lowest ordered child of $x$.
Note that \cref{alg:condition-root-greedy} chooses $x \to y$ instead of $x \to u$ by definition of $y$.
This implies that $x \to u \in S$ whenever $u \neq y$.

Let us define $\pi_{G_2}$ as follows:
\[
\pi_{G_2}(v) = 
\begin{cases}
\pi_{G_1}(x) & \text{if $v = y$}\\
\pi_{G_1}(u) & \text{if $v = x$}\\
\pi_{G_1}(y) & \text{if $v = u$}\\
\pi_{G_1}(v) & \text{else}
\end{cases}
\]
Clearly, $\pi_{G_1}(x) < \pi_{G_1}(y)$ and $\pi_{G_2}(x) > \pi_{G_2}(y)$.
Meanwhile, for any other two adjacent vertices $v$ and $v'$, observe that $\pi_{G_1}(v) < \pi_{G_1}(v') \iff \pi_{G_2}(v) < \pi_{G_2}(v')$ so $\pi_{G_2}$ agrees with the arc orientations of $\pi_{G_1}$ except for $x \sim y$.
See \cref{fig:two-approx-example} for an illustrated example.

Define vertex $b$ as follows:
\[
b = \argmin_{z \in V \;:\; z \in \Des(x) \text{ and } y \to z \in \cC(G_1)} \{ \pi_{G_1}(z) + n^2 \cdot \mathbbm{1}_{x \to z \in S} \}
\]
If vertex $b$ exists, then we know that $b \in \Ch_{G_1}(y)$ and $x \to b \in \cC(G_2)$ by \cref{lem:covered-is-direct} and \cref{lem:status-changes}.
By minimality of $b$, \cref{def:greedymm} will choose $y \to b$ if picking a covered edge starting with $y$ for $M_{G_1, \pi_{G_1}, S}$.
So, we can equivalently define vertex $b$ as follows:
\[
b = \argmin_{z \in V \;:\; z \in \Des(y) \text{ and } x \to z \in \cC(G_2)} \{ \pi_{G_2}(z) + n^2 \cdot \mathbbm{1}_{x \to z \in S} \}
\]
By choice of $\pi_{G_2}$, \cref{def:greedymm} will choose $x \to b$ if picking a covered edge starting with $x$ for $M_{G_2, \pi_{G_2}, S}$.

We will now construct a same-sized maximal matching $M_{G_2, \pi_{G_2}, S}$ from $M_{G_1, \pi_{G_1}, S}$ (Step 1), argue that it is maximal matching of $\cC(G_2)$ (Step 2), and that it is indeed a conditional-root-greedy matching for $\cC(G_2)$ with respect to $\pi_{G_2}$ and $S$ (Step 3).
There are three cases that cover all possibilities:
\begin{description}
    \item[Case 1] Vertex $a$ exists, $a \to x \in M_{G_1, \pi_{G_1}, S}$, and vertex $b$ exists.
    \item[Case 2] Vertex $a$ exists, $a \to x \in M_{G_1, \pi_{G_1}, S}$, and vertex $b$ does not exist.
    \item[Case 3] $a \to x \not\in M_{G_1, \pi_{G_1}, S}$.\\
    This could be due to vertex $a$ not existing, or $a \to x \not\in \cC(G_1)$, or $M_{G_1, \pi_{G_1}, S}$ containing a covered edge ending at $a$ so $a \to x$ was removed from consideration. 
\end{description}

\paragraph{Step 1: Construction of $M_{G_2, \pi_{G_2}, S}$ such that $|M_{G_2, \pi_{G_2}, S}| = |M_{G_1, \pi_{G_1}, S}|$.}\hspace{0pt}

By \cref{lem:status-changes}, covered edge statuses of edges whose endpoints do not involve $x$ or $y$ will remain unchanged.
By definition of $y$, we know that \cref{def:greedymm} will choose $x \to y$ if picking a covered edge starting with $x$ for $M_{G_1, \pi_{G_1}, S}$.

Since $a \to x \in M_{G_1, \pi_{G_1}}$ in cases 1 and 2, we know that there is no arcs of the form $x \to \cdot$ in $M_{G_1, \pi_{G_1}, S}$.
Since there is no arc of the form $\cdot \to x$ in $M_{G_1, \pi_{G_1}, S}$ in case 3, we know that $x \to y \in M_{G_1, \pi_{G_1}, S}$.

\begin{description}
    \item[Case 1] Define $M_{G_2, \pi_{G_2}, S} = M_{G_1, \pi_{G_1}, S} \cup \{a \to y, x \to b\} \setminus \{a \to x, y \to b\}$.
    \item[Case 2] Define $M_{G_2, \pi_{G_2}, S} = M_{G_1, \pi_{G_1}, S} \cup \{a \to y\} \setminus \{a \to x\}$.
    \item[Case 3] Define $M_{G_2, \pi_{G_2}, S} = M_{G_1, \pi_{G_1}, S} \cup \{y \to x\} \setminus \{x \to y\}$.
\end{description}

By construction, we see that $|M_{G_2, \pi_{G_2}, S}| = |M_{G_1, \pi_{G_1}, S}|$.

\paragraph{Step 2: $M_{G_2, \pi_{G_2}, S}$ is a maximal matching of the covered edge $\cC(G_2)$ of $G_2$.}\hspace{0pt}

To prove that $M_{G_2, \pi_{G_2}, S}$ is a maximal matching of $\cC(G_2)$, we argue in three steps:
\begin{enumerate}[2(i)]
    \item Edges of $M_{G_2, \pi_{G_2}, S}$ belong to $\cC(G_2)$.
    \item $M_{G_2, \pi_{G_2}, S}$ is a matching of $\cC(G_2)$.
    \item $M_{G_2, \pi_{G_2}, S}$ is maximal matching of $\cC(G_2)$.
\end{enumerate}

\paragraph{Step 2(i): Edges of $M_{G_2, \pi_{G_2}, S}$ belong to $\cC(G_2)$.}\hspace{0pt}

By \cref{lem:status-changes}, covered edge statuses of edges whose endpoints do not involve $x$ or $y$ will remain unchanged.
Since $M_{G_1, \pi_{G_1}, S}$ is a matching, it has at most one edge $e$ involving endpoint $x$ and at most one edge $e'$ involving endpoint $y$ ($e'$ could be $e$).

\begin{description}
    \item[Case 1] Since $b$ exists, the edges in $M_{G_1, \pi_{G_1}, S}$ with endpoints involving $\{x,y\}$ are $a \to x$ and $y \to b$.
        By \cref{lem:status-changes}, we know that $a \to y, x \to b \in \cC(G_2)$.
    \item[Case 2] Since $b$ does not exist, the only edge in $M_{G_1, \pi_{G_1}, S}$ with endpoints involving $\{x,y\}$ is $a \to x$.
        By \cref{lem:status-changes}, we know that $a \to y \in \cC(G_2)$.
    \item[Case 3] Since $a \to x \not\in M_{G_1, \pi_{G_1}, S}$, we have $x \to y \in M_{G_1, \pi_{G_1}, S}$ by minimality of $y$.
\end{description}

In all cases, we see that $M_{G_2, \pi_{G_2}, S} \subseteq \cC(G_2)$.

\paragraph{Step 2(ii): $M_{G_2, \pi_{G_2}, S}$ is a matching of $\cC(G_2)$.}\hspace{0pt}

It suffices to argue that there are \emph{no} two edges in $M_{G_2, \pi_{G_2}, S}$ sharing an endpoint.
Since $M_{G_1, \pi_{G_1}, S}$ is a matching, this can only happen via newly added endpoints in $M_{G_2, \pi_{G_2}, S}$.

\begin{description}
    \item[Case 1] The endpoints of newly added edges are exactly the endpoints of removed edges.
    \item[Case 2] Since we removed $a \to x$ and added $a \to y$, it suffices to check that there are no edges in $M_{G_1, \pi_{G_1}, S}$ involving $y$.
    This is true since $b$ does not exist in Case 2.
    \item[Case 3] The endpoints of newly added edges are exactly the endpoints of removed edges.
\end{description}

Therefore, we conclude that $M_{G_2, \pi_{G_2}, S}$ is a matching of $\cC(G_2)$.

\paragraph{Step 2(iii): $M_{G_2, \pi_{G_2}, S}$ is a maximal matching of $\cC(G_2)$.}\hspace{0pt}

For any $u \to v \in \cC(G_2)$, we show that there is some edge in $M_{G_2, \pi_{G_2}, S}$ with at least one of $u$ or $v$ is an endpoint.
By \cref{lem:status-changes}, covered edge statuses of edges whose endpoints do not involve $x$ or $y$ will remain unchanged, so it suffices to consider $|\{u,v\} \cap \{x,y\}| \geq 1$.

We check the following 3 scenarios corresponding to $|\{u,v\} \cap \{x,y\}| \geq 1$ below:
\begin{enumerate}[(i)]
    \item $y \in \{u,v\}$.
    
    The endpoints of $M_{G_2, \pi_{G_2}}$ always contains $y$.
    
    \item $y \not\in \{u,v\}$ and $x \to v \in \cC(G_2)$, for some $v \in V \setminus \{x,y\}$.
    
    Since $x \to v \in \cC(G_2)$ and $y \to x$ in $G_2$, it must be the case that $y \to v$ in $G_2$.
    Since $G_1$ and $G_2$ agrees on all arcs except $x \sim y$, we have that $y \to v$ in $G_1$ as well.
    Since $x \to v \in \cC(G_2)$, we know that $v \in \Ch_{G_2}(x)$ via \cref{lem:covered-is-direct}.
    So, we have $y \to v \in \cC(G_1)$ via \cref{lem:status-changes}.
    Since the set $\{ v : y \to v \in \cC(G_1) \}$ is non-empty, vertex $b$ exists.
    In both cases 1 and 3, the endpoints of $M_{G_2, \pi_{G_2}}$ includes $x$.
    
    \item $y \not\in \{u,v\}$ and $u \to x \in \cC(G_2)$, for some $u \in V \setminus \{x,y\}$.
    
    By \cref{lem:covered-is-direct}, we know that $x \in \Ch_{G_2}(u)$.
    Meanwhile, since $y \to x \in \cC(G_2)$, we must have $u \to y$ in $G_2$.
    However, this implies that $x \not\in \Ch_{G_2}(u)$ since $u \to y \to x$ exists.
    This is a contradiction, so this situation cannot happen.
\end{enumerate}

As the above argument holds for any $u \to v \in \cC(G_2)$, we see that $M_{G_2, \pi_{G_2}}$ is maximal matching for $\cC(G_2)$.

\paragraph{Step 3: $M_{G_2, \pi_{G_2}, S}$ is a conditional-root-greedy maximal matching.}\hspace{0pt}

We now compare the execution of \cref{alg:condition-root-greedy} on $(\pi_{G_1}, S)$ and $(\pi_{G_2}, S)$.
Note that $S$ remains unchanged.

We know the following:
\begin{itemize}
    \item Since $\pi_{G_2}(y) = \pi_{G_1}(x)$ and $a \to x \in S$, if $a$ exists and $a \to x$ is chosen by \cref{alg:condition-root-greedy} on $(\pi_{G_1}, S)$, then it means that there are \emph{no} $a \to v$ arc in $\cC(G_1)$ such that $a \to v \not\in S$.
    So, $a \to y$ will be chosen by \cref{alg:condition-root-greedy} on $(\pi_{G_2}, S)$ if $a$ exists.
    \item Since $\pi_{G_2}(y) = \pi_{G_1}(x)$, $x$ is chosen as a root by \cref{alg:condition-root-greedy} on $(\pi_{G_1}, S)$ if and only if $y$ is chosen as a root by \cref{alg:condition-root-greedy} on $(\pi_{G_2}, S)$.
    \item By definition of $b$, if it exists, then $y \to b \in M_{G_1, \pi_{G_1}, S} \iff x \to b \in M_{G_2, \pi_{G_2}, S}$.
    \item By the definition of $\pi_{G_2}$, we see that \cref{alg:condition-root-greedy} makes the ``same decisions'' when choosing arcs rooted on $V \setminus \{a,x,y,b\}$.
\end{itemize}

Therefore, $M_{G_2, \pi_{G_2}, S}$ is indeed a conditional-root-greedy maximal matching for $\cC(G_2)$ with respect to $\pi_{G_2}$ and $S$.
\end{proof}

\twoapproxratio*
\begin{proof}
Consider any two DAGs $G_s, G_t \in [G^*]$.
To transform $G_s = (V, E_s)$ to $G_t = (V, E_t)$, \cref{alg:repeatedfindedge} flips covered edges one by one such that $|E_s \setminus E_t|$ decreases in a monotonic manner.
We will repeatedly apply \cref{lem:same-mm-size} with $S = A(G_s) \cap A(G_t)$ on the sequence of covered edge reversals produced by \cref{alg:repeatedfindedge}.

Let $\pi_{G_s}$ be an arbitrary ordering for $G_s$ and we compute an initial conditional-root-greedy maximal matching for $\cC(G_s)$ with respect to some ordering $\pi_{G_s}$ and conditioning set $S$.
To see why \cref{lem:same-mm-size} applies at each step for reversing a covered edge from $x \to y$ to $y \to x$, we need to ensure the following:
\begin{enumerate}
    \item If $x$ has a parent vertex $a$ (i.e.\ $x \in \Ch_{G_1}(a)$), then $a \to x \in S$.
    
    If $a \to x \not\in S$, then then $a \to x$ is a covered edge that should be flipped to transform from $G_s$ to $G_t$.
    However, this means that \cref{alg:repeatedfindedge} would pick $a \to x$ to reverse instead of picking $x \to y$ to reverse.
    Contradiction.
    
    \item $x \to y, y \to x \not\in S$.
    
    This is satisfied by the definition of $S = E_s \cap E_t$ since reversing $x \to y$ to $y \to x$ implies that neither of them are in $S$.
    
    \item $y = \argmin_{z \;:\; x \to z \in \cC(G_1)} \{ \pi_{G_1}(z) + n^2 \cdot \mathbbm{1}_{x \to z \in S} \}$.
    
    Since $x \to y \not\in S$, this is equivalent to checking if $y = \argmin_{z \;:\; x \to z \in \cC(G_1)} \{ \pi_{G_1}(z) \}$.
    This is satisfied by line 7 of \cref{alg:repeatedfindedge}.
    
    \item $M_{G_1, \pi_{G_1}, S}$ is a conditional-root-greedy maximal matching for $\cC(G_1)$ with respect to some ordering $\pi_{G_1}$ and conditioning set $S$.
    
    This is satisfied since we always maintain a conditional-root-greedy maximal matching and $S$ is unchanged throughout.
\end{enumerate}

By applying \cref{lem:same-mm-size} with $S = A(G_s) \cap A(G_t)$ repeatedly on the sequence of covered edge reversals produced by \cref{alg:repeatedfindedge}, we see that there exists a conditional-root-greedy maximal matching $M_{G_s, \pi_{G_s}}$ for $\cC(G_s)$ and a conditional-root-greedy maximal matching $M_{G_t, \pi_{G_t}}$ for $\cC(G_t)$ such that $| M_{G_s, \pi_{G_s}} | = | M_{G_t, \pi_{G_t}} |$.

The claim follows since maximal matching is a 2-approximation to minimum vertex cover, and the verification number $\nu(G)$ of any DAG $G$ is the size of the minimum vertex cover of its covered edges $\cC(G)$.
\end{proof}

\tightnessoftwoapprox*
\begin{proof}
See \cref{fig:two-approx-tight}.
\end{proof}

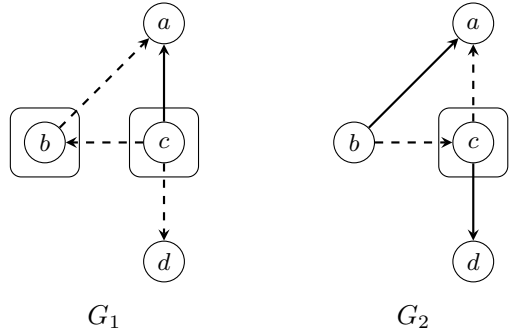
\begin{figure}[htbp]
\centering
\resizebox{0.4\linewidth}{!}{%
\begin{tikzpicture}
%
%
\node[draw, circle, minimum size=15pt, inner sep=2pt] at (0,0) (a1) {\small $a$};
\node[draw, circle, minimum size=15pt, inner sep=2pt, below=of a1] (c1) {\small $c$};
\node[draw, circle, minimum size=15pt, inner sep=2pt, left=of c1] (b1) {\small $b$};
\node[draw, circle, minimum size=15pt, inner sep=2pt, below=of c1] (d1) {\small $d$};
\draw[thick, -stealth, dashed] (b1) -- (a1);
\draw[thick, -stealth] (c1) -- (a1);
\draw[thick, -stealth, dashed] (c1) -- (b1);
\draw[thick, -stealth, dashed] (c1) -- (d1);
\node[draw, rounded corners, fit=(b1), inner sep=5pt] {};
\node[draw, rounded corners, fit=(c1), inner sep=5pt] {};
\node[] at ($(b1)!0.5!(c1) + (0,-2.25)$) {$G_1$};

%
%
\node[draw, circle, minimum size=15pt, inner sep=2pt] at (4,0) (a2) {\small $a$};
\node[draw, circle, minimum size=15pt, inner sep=2pt, below=of a2] (c2) {\small $c$};
\node[draw, circle, minimum size=15pt, inner sep=2pt, left=of c2] (b2) {\small $b$};
\node[draw, circle, minimum size=15pt, inner sep=2pt, below=of c2] (d2) {\small $d$};
\draw[thick, -stealth] (b2) -- (a2);
\draw[thick, -stealth, dashed] (c2) -- (a2);
\draw[thick, -stealth, dashed] (b2) -- (c2);
\draw[thick, -stealth] (c2) -- (d2);
\node[draw, rounded corners, fit=(c2), inner sep=5pt] {};
\node[] at ($(b2)!0.5!(c2) + (0,-2.25)$) {$G_2$};
\end{tikzpicture}
}
\caption{
The ratio of 2 in \cref{thm:two-approx-ratio} is tight: $G_1$ and $G_2$ belong in the same MEC with $\nu(G_1) = 2$ and $\nu(G_2) = 1$.
The dashed arcs represent the covered edges and the boxed vertices represent a minimum vertex cover of the covered edges.
}
\label{fig:two-approx-tight}
\end{figure}

\subsection{Adaptive search with imperfect advice}

\ifOPTreachedthenblindsearchfullyorients*
\begin{proof}
By \cref{thm:subsetsearch}, \texttt{SubsetSearch} fully orients edges within the node-induced subgraph induced by $V'$, i.e.\ \texttt{SubsetSearch} will perform atomic interventions on $\cI_{V'} \subseteq V$ resulting in $\cE_{\cI_{V'}}(G^*)[V'] = G^*[V']$.
Since $\cC(G^*) \subseteq E(G^*[V'])$ and all covered edges $\cC(G^*)$ were oriented, then according to \cref{lem:necessary}, it must be the case that $V^* \subseteq \cI_{V'}$ for some minimum vertex cover $V^*$ of $\cC(G^*)$, so we see that $R(G^*,V^*) \subseteq R(G^*, \cI_{V'})$.
By \cref{lem:sufficient}, we have $R(G^*,V^*) = A(G^*)$ and so $\texttt{SubsetSearch}(\cE(G^*), V')$ fully orients $\cE(G^*)$.
\end{proof}

We will now prove our main result (\cref{thm:search-with-advice}) which shows that the number of interventions needed is a function of the quality of the given advice DAG.
Let us first recall how we defined the quality of a given advice and restate our algorithm.

\qualitymeasure*

\setcounter{algorithm}{0}
\begin{algorithm}[htb]
\caption{Adaptive search algorithm with advice.}
\begin{algorithmic}[1]
    \Statex \textbf{Input}: Essential graph $\cE(G^*)$, advice DAG $\wt{G} \in [G^*]$, intervention size $k \in \N$
    \Statex \textbf{Output}: An intervention set $\cI$ such that each intervention involves at most $k$ nodes and $\cE_{\cI}(G^*) = G^*$.
    \State Let $\wt{V} \in \cV(\wt{G})$ be any atomic verifying set of $\wt{G}$.
    \If{$k = 1$}
        \State Define $\cI_0 = \wt{V}$ as an atomic intervention set.
    \Else
        \State Define $k' = \min\{k, |\wt{V}|/2\}$, $a = \lceil |\wt{V}|/k' \rceil \geq 2$, and $\ell = \lceil \log_a |C| \rceil$. Compute labelling scheme on $\wt{V}$ with $(|\wt{V}|, k, a)$ via \cref{lem:labelling-scheme} and define $\cI_0 = \{S_{x,y}\}_{x \in [\ell], y \in [a]}$, where $S_{x,y} \subseteq \wt{V}$ is the subset of vertices whose $x^{th}$ letter in the label is $y$.
    \EndIf
    \State Intervene on $\cI_0$ and initialize $r \gets 0$, $i \gets 0$, $n_0 \gets 2$.
	\While{$\cE_{\cI_{i}}(G^*)$ still has undirected edges}
        \If{$\rho(\cI_{i}, N^r_{\skel(\cE(G^*))}(\wt{V})) \geq n_i^2$}
            \State Increment $i \gets i + 1$ and record $r(i) \gets r$.
            \State Update $n_i \gets \rho(\cI_{i}, N^r_{\skel(\cE(G^*))}(\wt{V}))$
            \State $C_{i} \gets \texttt{SubsetSearch}(\cE_{\cI_{i}}(G^*), N^{r-1}_{\skel(\cE(G^*))}(\wt{V}), k)$
            \If{$\cE_{\cI_{i-1} \;\cup\; C_{i}}(G^*)$ still has undirected edges}
                \State $C'_{i} \gets \texttt{SubsetSearch}(\cE_{\cI_{i-1} \,\cup\, C_{i}}(G^*), N^{r}_{\skel(\cE(G^*))}(\wt{V}), k)$
                \State Update $\cI_{i} \gets \cI_{i-1} \cup C_{i} \cup C'_{i}$.
            \Else
                \State Update $\cI_{i} \gets \cI_{i-1} \cup C_{i}$.
            \EndIf
        \EndIf
        \State Increment $r \gets r + 1$.
	\EndWhile
	\State \textbf{return} $\cI_i$
\end{algorithmic}
\end{algorithm}

\searchwithadvice*
\begin{proof}
Consider \cref{alg:adaptive-search-with-advice-algo}.
Observe that $n_0 = 2$ ensures that $n_0^2 > n_0$.

In this proof, we will drop the subscript $\skel(\cE(G^*))$ when we discuss the $r$-hop neighbors $N^r_{\skel(\cE(G^*))}(\cdot)$.
We first prove the case where $k = 1$ then explain how to tweak the proof for the case of $k > 1$.

If \cref{alg:adaptive-search-with-advice-algo} terminates when $i = 0$, then $\cI = \cI_0 = \wt{V}$ and \cref{thm:two-approx-ratio} tells us that $|\cI| \in \cO(\nu_1(G^*))$.

Now, suppose \cref{alg:adaptive-search-with-advice-algo} terminates with $i = t$, for some final round $t >0$.
As \cref{alg:adaptive-search-with-advice-algo} uses an arbitrary verifying set of $\wt{G}$ in step 3, we will argue that $\cO(\max\{1, \log |N^{h(G^*, \wt{V})}(\wt{V})|\} \cdot \nu(G^*))$ interventions are used in the while-loop, for any arbitrary chosen $\wt{V} \in \cV(\wt{G})$.
The theorem then follows by taking a maximization over all possibilities in $\cV(\wt{G})$.

In Line 12, $r(i)$ records the hop value such that $\rho(\cI_{i}, N^{r(i)}(\wt{V})) \geq n_i^2$, for any $0 \leq i < t$.
By construction of the algorithm, we know the following:
\begin{enumerate}
    \item For any $0 < i \leq t$,
    \begin{equation}
    \label{eq:n-gamma-relationship}
    n_{i} = \rho(\cI_{i}, N^{r(i)}(\wt{V})) \geq n_{i-1}^2 > \rho(\cI_{i}, N^{r(i)-1}(\wt{V}))
    \end{equation}
    because $r(i)-1$ did \emph{not} trigger \cref{alg:adaptive-search-with-advice-algo} to record $r(i)$.
    
    \item By \cref{thm:subsetsearch} and \cref{eq:n-gamma-relationship}, for any $1 \leq i \leq t$,
    \begin{equation}
    \label{eq:C-bound}
    \begin{array}{cll}
    |C_i|
    &\in \cO(\log \rho(\cI_{i}, N^{r(i)-1}(\wt{V})) \cdot \nu_1(G^*))
    &\subseteq \cO(\log n_{i-1} \cdot \nu_1(G^*))\\
    |C'_i|
    &\in \cO(\log \rho(\cI_{i}, N^{r(i)}(\wt{V})) \cdot \nu_1(G^*))
    &\subseteq \cO(\log n_i \cdot \nu_1(G^*))
    \end{array}
    \end{equation}
    Note that the bound for $|C'_i|$ is an over-estimation (but this is okay for our analytical purposes) since some nodes previously counted for $\rho(\cI_{i}, N^{r(i)}(\wt{V}))$ may no longer be relevant in $\cE_{\cI_{i} \;\cup\; C_{i}}(G^*)$ after intervening on $C_i$.
    
    \item Since $n_{i-1} \leq \sqrt{n_i}$ for any $0 < i \leq t$, we know that $n_j \leq n_t^{1/2^{t-j}}$ for any $0 \leq j \leq t$.
    So, for any $0 \leq t' \leq t$, we have
    \begin{equation}
    \label{eq:log-sum-inequality}
    \sum_{i=0}^{t'} \log(n_i)
    \leq \sum_{i=0}^{t'} \log \left( n_{t'}^{1/2^{t'-i}} \right)
    = \sum_{i=0}^{t'} \frac{\log (n_{t'}) }{2^{t'-i}}
    \leq 2 \cdot \log (n_{t'})
    \end{equation}
    
    \item By definition of t, $h(G^*, \wt{V})$, and \cref{lem:if-OPT-reached-then-blind-search-fully-orients},
    \begin{equation}
    r(t-1) < h(G^*, \wt{V}) \leq r(t)
    \end{equation}
    and
    \begin{equation}
    N^{r(t-1)}(\wt{V}) \subsetneq N^{h(G^*, \wt{V})}(\wt{V}) \subseteq N^{r(t)}(\wt{V})    
    \end{equation}
\end{enumerate}

Combining \cref{eq:n-gamma-relationship}, \cref{eq:C-bound}, and \cref{eq:log-sum-inequality}, we get
\begin{equation}
\label{eq:sum-of-previous-C}
\sum_{i=1}^{t-1} \left( |C_i| + |C'_i| \right)
\in \cO \left( \left( \sum_{i=1}^{t-1} \log n_{i-1} + \log n_{i} \right) \cdot \nu_1(G^*) \right)
\subseteq \cO \left( \sum_{i=1}^{t-1} \log n_{i} \cdot \nu_1(G^*) \right)
\subseteq \cO \left( \log n_{t-1} \cdot \nu_1(G^*) \right)
\end{equation}

To relate $|\cI_t|$ with $|N^{h(G^*, \wt{V})}(\wt{V})|$, we consider two scenarios depending on whether the essential graph was fully oriented after intervening on $C_t$ or $C'_t$.

\textbf{Scenario 1: Fully oriented after intervening on $C_t$, i.e.\ $\cE_{\cI_{t-1} \;\cup\; C_{t}}(G^*) = G^*$.}
Then,
\[
\cI_t
= C_{t} \;\dot\cup\; \cI_{t-1}
= C_{t} \;\dot\cup\; (C_{t-1} \;\dot\cup\; C'_{t-1}) \;\dot\cup\; \cI_{t-2}
= \ldots
= C_{t} \;\dot\cup\; \bigcup_{i=1}^{t-1} (C_i \;\dot\cup\; C'_i) \;\dot\cup\; \wt{V}
\]

In this case, $h(G^*, \wt{V}) = r(t)-1$.
By definition, $n_{t-1} \leq |N^{r(t-1)}(\wt{V})|$ and we have
\begin{equation}
\label{eq:scenario1-bound}
n_{t-1} \leq |N^{r(t-1)}(\wt{V})| < |N^{h(G^*, \wt{V})}(\wt{V})|
\end{equation}
since $N^{r(t-1)}(\wt{V}) \subsetneq N^{h(G^*, \wt{V})}(\wt{V})$.
So,
\begin{align*}
|\cI_t| - |\wt{V}|
&= |C_{t}| + \sum_{i=1}^{t-1} \left( |C_i| + |C'_i| \right)\\
& \in \cO \left( \log n_{t-1} \cdot \nu_1(G^*) \right) + \cO \left( \log n_{t-1} \cdot \nu_1(G^*) \right) && \text{By \cref{eq:C-bound} and \cref{eq:sum-of-previous-C}}\\
&\subseteq \cO \left( \log |N^{h(G^*, \wt{V})}(\wt{V})| \cdot \nu_1(G^*) \right) && \text{\cref{eq:scenario1-bound}}
\end{align*}

\textbf{Scenario 2: Fully oriented after intervening on $C'_t$, i.e.\ $\cE_{\cI_{t-1} \;\cup\; C_{t} \;\cup\; C'_{t}}(G^*) = G^*$.}
Then,
\[
\cI_t
= C_{t} \;\dot\cup\; C'_{t} \;\dot\cup\; \cI_{t-1}
= \ldots
= C_{t} \;\dot\cup\; C'_{t} \;\dot\cup\; \bigcup_{i=1}^{t-1} (C_i \;\dot\cup\; C'_i) \;\dot\cup\; \wt{V}
\]

In this case, $h(G^*, \wt{V}) = r(t)$ and $N^{h(G^*, \wt{V})}(\wt{V}) = N^{r(t)}(\wt{V})$. So,
\begin{equation}
\label{eq:scenario2-bound}
n_t \leq |N^{r(t)}(\wt{V})| = |N^{h(G^*, \wt{V})}(\wt{V})|
\end{equation}
So,
\begin{align*}
|\cI_t| - |\wt{V}|
&= |C_{t}| + |C'_{t}| + \sum_{i=1}^{t-1} \left( |C_i| + |C'_i| \right)\\
& \in \cO \left( \left( \log n_{t-1} + n_t \right) \cdot \nu_1(G^*) \right) + \cO \left( \log n_{t-1} \cdot \nu_1(G^*) \right) && \text{By \cref{eq:C-bound} and \cref{eq:sum-of-previous-C}}\\
&\subseteq \cO \left( \log |N^{h(G^*, \wt{V})}(\wt{V})| \cdot \nu_1(G^*) \right) && \text{\cref{eq:scenario2-bound}}
\end{align*}

Since $|\wt{V}| \in \cO(\nu_1(G^*))$, we can conclude
\[
|\cI_t|
\in \cO \left( \nu(G^*) + \log |N^{h(G^*, \wt{V})}(\wt{V})| \cdot \nu_1(G^*) \right)
\subseteq \cO \left( \max \left\{1, \log |N^{h(G^*, \wt{V})}(\wt{V})| \right\} \cdot \nu_1(G^*) \right)
\]
in either scenario, as desired.
The theorem then follows by taking a maximization over all $\wt{V} \in \cV(\wt{G})$.

\textbf{Adapting the proof for $k > 1$}

By \cref{thm:efficient-near-optimal-bounded}, $\nu_k(G^*) \geq \lceil \nu_1(G^*)/k \rceil$.
So, $|\cI_0| \in \cO(\log k \cdot \nu_k(G^*))$ via \cref{lem:labelling-scheme}.
The rest of the proof follows the same structure except that we use the bounded size guarantee of \cref{thm:subsetsearch}, which incurs an additional multiplicative $\log k$ factor.

\textbf{Polynomial running time}

By construction, the \cref{alg:adaptive-search-with-advice-algo} is deterministic.
Furthermore, \cref{alg:adaptive-search-with-advice-algo} runs in polynomial time because:
\begin{itemize}
    \item Hop information and relevant nodes can be computed in polynomial time via breadth first search and maintaining suitable neighborhood information.
    \item It is known that performing Meek rules to obtain essential graphs takes polynomial time (\cite{wienobst2021extendability}).
    \item \cref{alg:adaptive-search-with-advice-algo} makes at most two calls to \texttt{SubsetSearch} whenever the number of relevant nodes is squared.
    Each \texttt{SubsetSearch} call is known to run in polynomial time (\cref{thm:subsetsearch}).
    Since this happens each time the number of relevant nodes is squared, this can happen at most $\cO(\log n)$ times.
\end{itemize}
\end{proof}

\searchwithpartialadvice*
\begin{proof}
Apply \cref{thm:search-with-advice} while taking a maximization over all possible advice DAGs $\wt{G}$ consistent with the given partial advice.
\end{proof}

\section{Path essential graph}
\label{sec:appendix-path}

In this section, we explain why our algorithm (\cref{alg:adaptive-search-with-advice-algo}) is simply the classic ``binary search with prediction''\footnote{e.g.\ see \url{https://en.wikipedia.org/wiki/Learning_augmented_algorithm\#Binary_search}} when the given essential graph $\cE(G^*)$ is an undirected path on $n$ vertices.
So, another way to view our result is a \emph{generalization} that works on essential graphs of arbitrary moral DAGs.

When the given essential graph is a path $\cE(G^*)$ on $n$ vertices, we know that there are $n$ possible DAGs in the Markov equivalence class where each DAG corresponds to choosing a single root node and having all edges pointing away from it.
Observe that a verifying set of any DAG is then simply the root node as the set of of covered edges in any rooted tree are precisely the edges incident to the root.

Therefore, given any $\wt{G} \in [G^*]$, we se that $h(G^*, \wt{V})$ measures the number of hops between the root of the advice DAG $\wt{G}$ and the root of the true DAG $G^*$.
Furthermore, by Meek rule R1, whenever we intervene on a vertex $u$ on the path, we will fully orient the ``half'' of the path that points away from the root while the subpath between $u$ and the root remains unoriented (except the edge directly incident to $u$).
So, one can see that \cref{alg:adaptive-search-with-advice-algo} is actually mimicking exponential search from the root of $\wt{G}$ towards the root of $G^*$.
Then, once the root of $G^*$ lies within the $r$-hop neighborhood $H$, \texttt{SubsetSearch} uses $\cO(\log |V(H)|)$ interventions, which matches the number of queries required by binary search within a fixed interval over $|V(H)|$ nodes.

\section{Experiments}
\label{sec:appendix-experiments}

In this section, we provide more details about our experiments.

All experiments were run on a laptop with Apple M1 Pro chip and 16GB of memory.
Source code implementation and experimental scripts are available at\\
\url{https://github.com/cxjdavin/active-causal-structure-learning-with-advice}.

\subsection{Experimental setup}

For experiments, we evaluated our advice algorithm on the synthetic graph instances of \cite{wienobst2021polynomial}\footnote{See Appendix E of \cite{wienobst2021polynomial} for details about each class of synthetic graphs. Instances are available at \url{https://github.com/mwien/CliquePicking/tree/master/aaai_experiments}} on graph instances of sizes $n = \{16, 32, 64\}$.
For each undirected chordal graph instance, we do the following:
\begin{enumerate}
    \item Set $m = 1000$ as the number of advice DAGs that we will sample.
    \item Use the uniform sampling algorithm of \cite{wienobst2021polynomial} to uniformly sample $m$ advice DAGs $\wt{G}_1, \ldots, \wt{G}_{m}$.
    \item Randomly select $G^*$ from one of $\wt{G}_1, \ldots, \wt{G}_{m}$.
    \item For each $\wt{G} \in \{\wt{G}_1, \ldots, \wt{G}_{m}\}$,
    \begin{itemize}
        \item Compute a minimum verifying set $\wt{V}$ of $\wt{G}$.
        \item Define and compute $\psi(G^*, \wt{V}) = \left| \rho \left( \wt{V}, N_{\skel(\cE(G^*))}^{h(G^*, \wt{V})}(\wt{V}) \right) \right|$.
        \item Compute a verifying set using $(\cE(G^*), \wt{G})$ as input to \cref{alg:adaptive-search-with-advice-algo}.
    \end{itemize}
    \item Aggregate the sizes of the verifying sets used based on $\psi(G^*, \wt{V})$ and compute the mean and standard deviations.
    \item Compare against verification number $\nu_1(G^*)$ and the number of interventions used by the fully adaptive search (without advice, which we denote as ``blind search'' in the plots) of \cite{choo2022verification}. 
    \item Compute the empirical distribution of the quality measure amongst the $m$ advice DAGs, then use standard sample complexity arguments for estimating distributions up to $\eps$ error in TV distance to compute a confidence interval for which the true cumulative probability density of all DAGs within the MEC lies within\footnote{For example, see Theorem 1 of \cite{canonne2020short}.}.
    To be precise, it is known that for a discrete distribution $P$ on $k$ elements, when there are $m \geq \max\{k/\eps^2, (2/\eps^2) \cdot \ln (2/\delta)\}$ uniform samples, the probability that the TV distance between the true distribution $P$ and the empirical distribution $P$ is less than $\eps$ is at least $1-\delta$.
    Since the \emph{upper bound} on the domain size of quality measure is the number of nodes $n$, by setting $m=1000$ and $\delta = 0.01$, we can compute $\eps = \max\{\sqrt{n/m}, \sqrt{(2/m) \cdot \ln (2/\delta)}\}$ and conclude that the probability that the true cumulative probability density of all DAGs within the MEC lies within $\eps$ distance (clipped to be between 0 and 1) of the empirical distribution is at least 99\%.
\end{enumerate}

\subsection{Experimental remarks}

\begin{itemize}
    \item The uniform sampling code of \cite{wienobst2021polynomial} is written in Julia and it uses a non-trivial amount of memory, which may make it unsuitable for running on a shared server with memory constraints.
    \item Note that $\psi(G^*, \wt{V}) \leq \psi(G^*, \wt{G}) = \max_{\wt{V} \in \cV(\wt{G})} \left| \rho \left( \wt{V}, N^{h(G^*, \wt{V})}_{\skel(\cE(G^*))}(\wt{V}) \right) \right|$.
    We use $\psi(G^*, \wt{V})$ as a proxy for $\psi(G^*, \wt{G})$ because we do not know if there is an efficient way to compute the latter besides the naive (possibly exponential time) enumeration over all possible minimum verifying sets.
    \item We also experimented with an ``unsafe'' variant of \cref{alg:adaptive-search-with-advice-algo} where we ignore the second tweak of intervening one round before.
    In our synthetic experiments, both variants use a similar number of interventions.
    \item We do not plot the theoretical upper bounds $\cO(\log \psi(G^*, \wt{V}) \cdot \nu_1(G^*))$ or $\cO(\log n \cdot \nu_1(G^*))$ because these values are a significantly higher than the other curves and result in ``squashed'' (and less interesting/interpretable) plots.
    \item Even when $\psi(G^*, \wt{V}) = 0$, there could be cases where \cite{choo2022verification} uses more interventions than $\nu_1(G^*)$.
    For example, consider \cref{fig:two-approx-tight} with $G^* = G_2$ and $\wt{G} = G_1$.
    After intervening on $\wt{V} = \{b,c\}$, the entire graph will be oriented so the $\psi(G^*, \wt{V}) = 0$ while $\nu_1(G^*) = 1 < 2 = |\wt{V}|$.
    Fortunately, \cref{thm:two-approx-ratio} guarantees that $|\wt{V}| \leq 2 \cdot \nu_1(G^*)$.
    \item Note that the error bar may appear ``lower'' than the verification number even though all intervention sizes are at least as large as the verification number.
    For instance, if $\nu_1(G^*) = 6$ and we used $(6,6,7)$ interventions on three different $\wt{G}$'s, each with $\psi(G^*, \wt{V}) = 0$.
    In this case, the mean is $6.3333\ldots$ while the standard deviation is $0.4714\ldots$, so the error bar will display an interval of $[5.86\ldots, 6.80\ldots]$ whose lower interval is below $\nu_1(G^*) = 6$.
\end{itemize}

\subsection{All experimental plots}

For details about the synthetic graph classes, see Appendix E of \cite{wienobst2021polynomial}.
Each experimental plot is for one of the synthetic graphs $G^*$, with respect to $1000 \ll |[G^*]|$ uniformly sampled advice DAGs $\wt{G}$ from the MEC $[G^*]$.
The solid lines indicate the number of atomic interventions used while the dotted lines indicate the empirical cumulative probability density of $\wt{G}$.
The true cumulative probability density lies within the shaded area with probability at least 0.99.

\begin{figure}[htbp]
\centering
\begin{subfigure}[t]{0.3\linewidth}
    \centering
    \includegraphics[width=\linewidth]{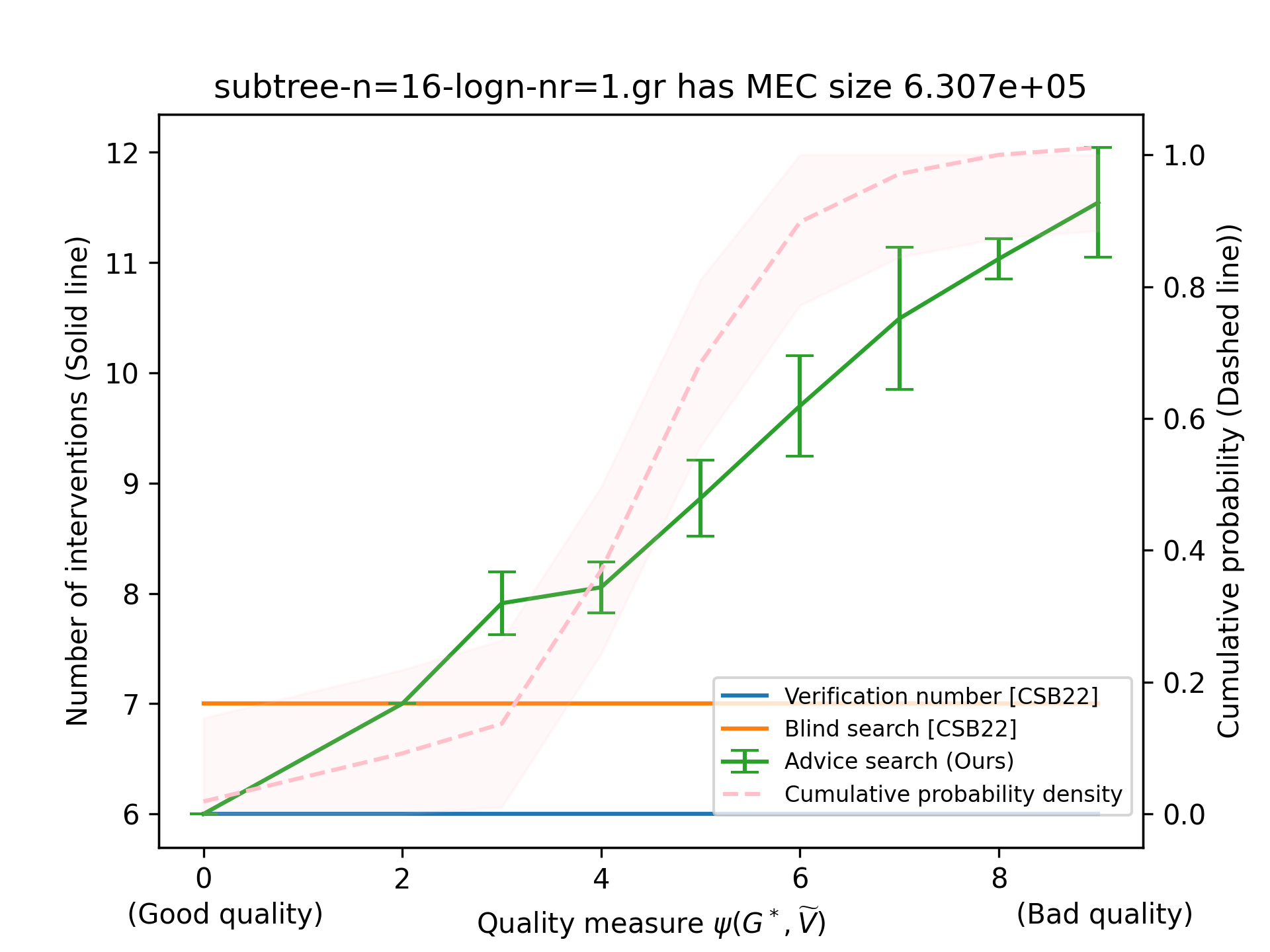}
    \caption{$n = 16$}
\end{subfigure}
\begin{subfigure}[t]{0.3\linewidth}
    \centering
    \includegraphics[width=\linewidth]{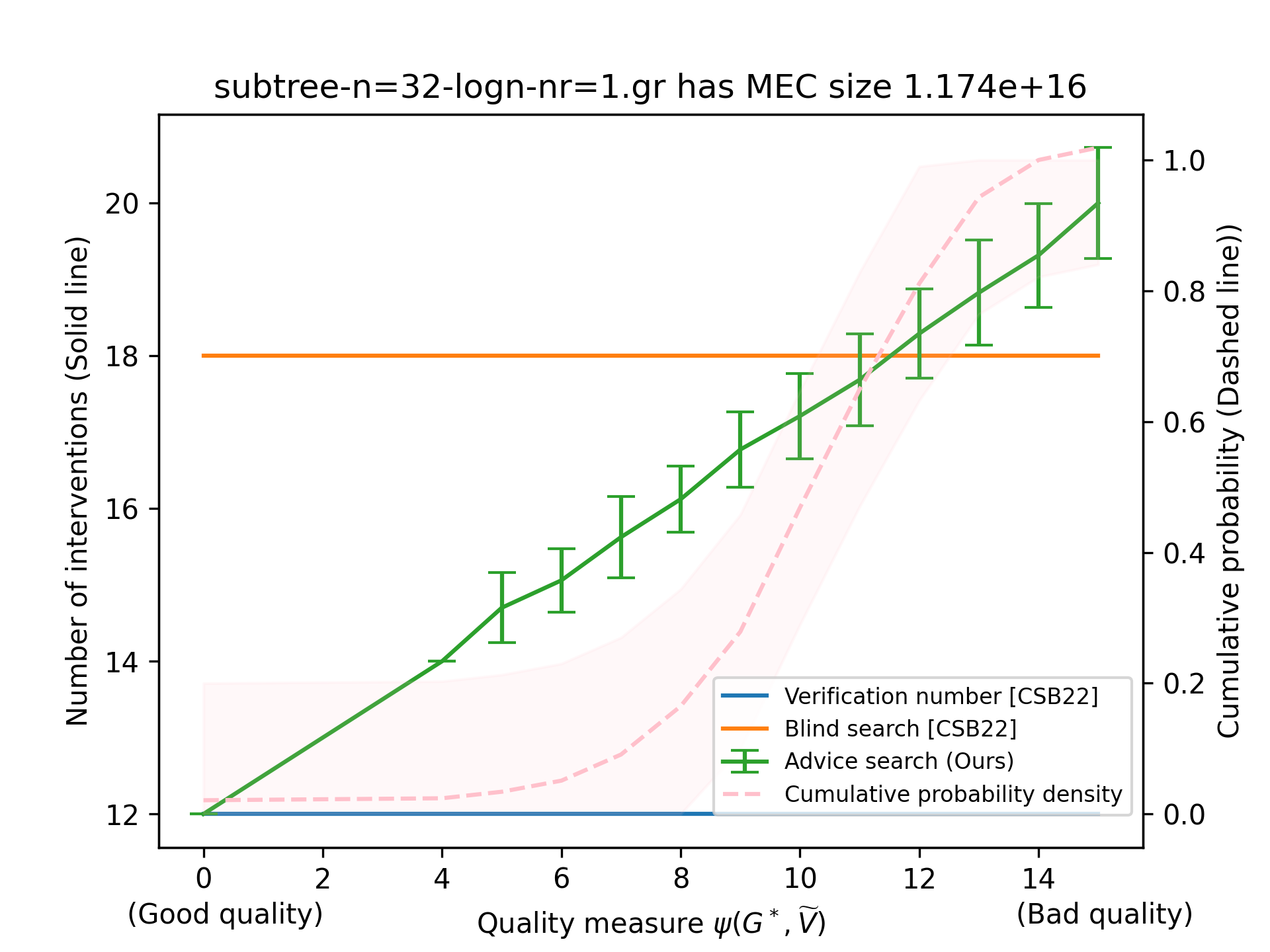}
    \caption{$n = 32$}
\end{subfigure}
\begin{subfigure}[t]{0.3\linewidth}
    \centering
    \includegraphics[width=\linewidth]{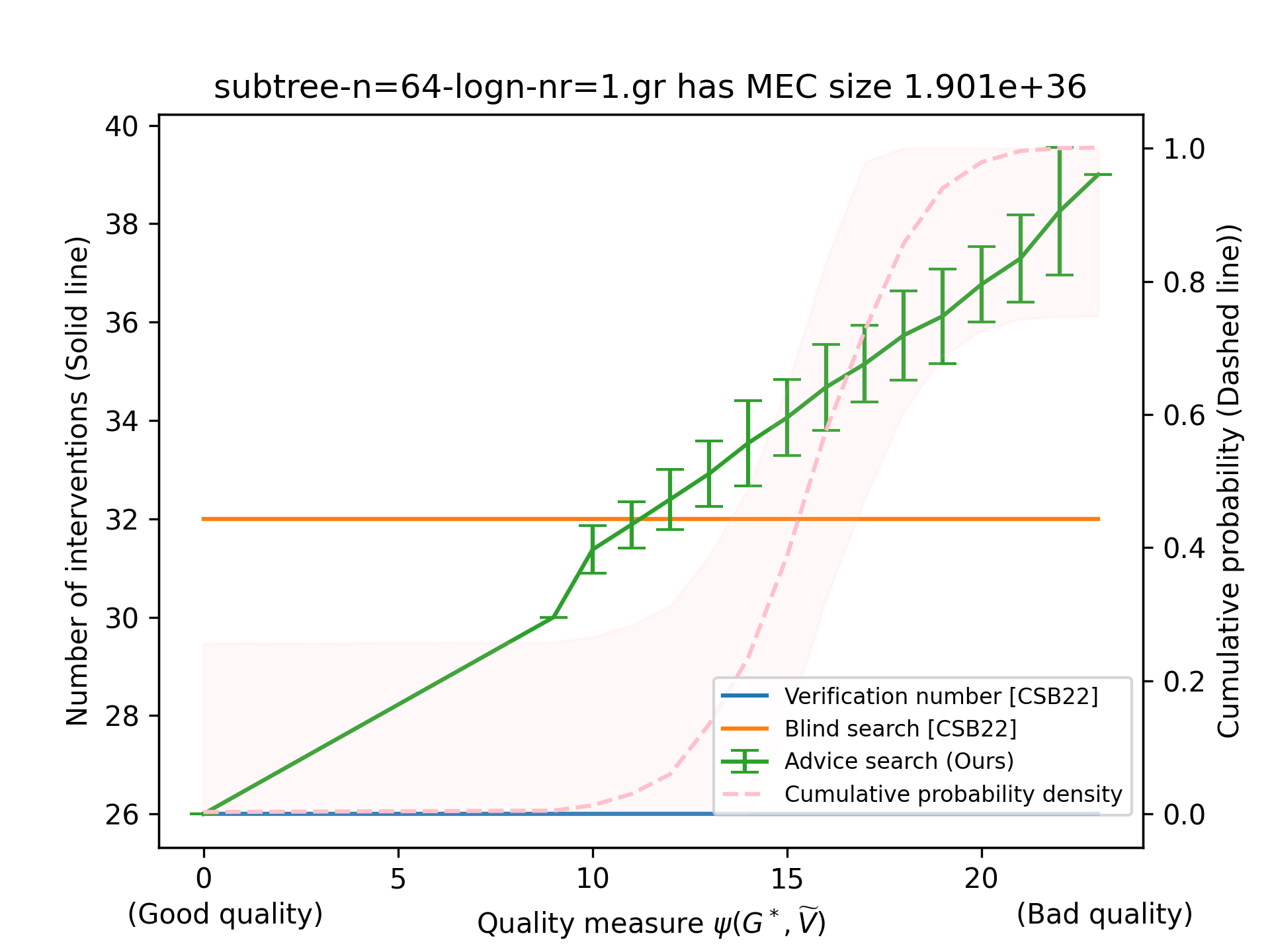}
    \caption{$n = 64$}
\end{subfigure}
\caption{Subtree-logn synthetic graphs}
\label{fig:subtree-logn}
\end{figure}

\begin{figure}[htbp]
\centering
\begin{subfigure}[t]{0.3\linewidth}
    \centering
    \includegraphics[width=\linewidth]{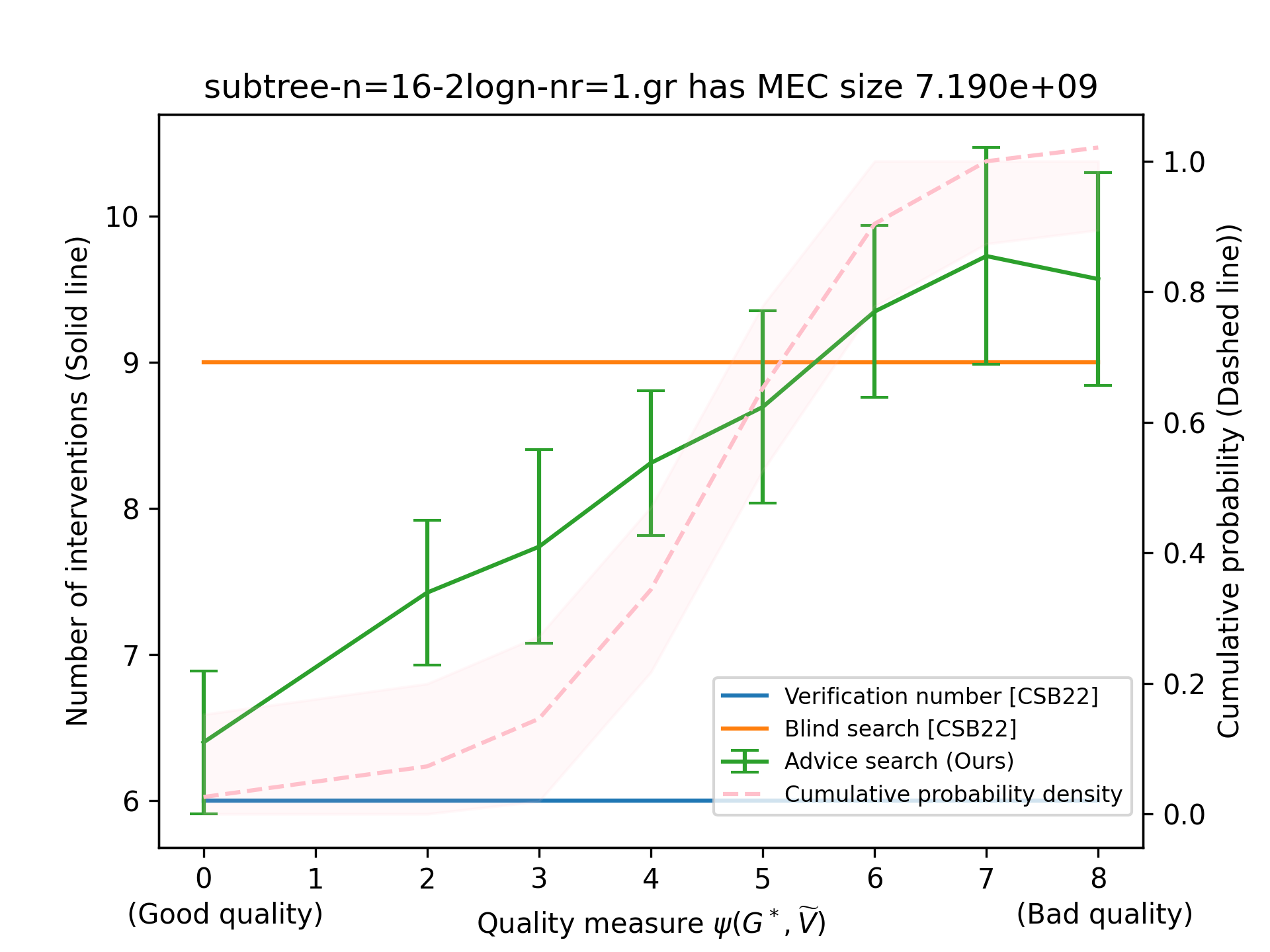}
    \caption{$n = 16$}
\end{subfigure}
\begin{subfigure}[t]{0.3\linewidth}
    \centering
    \includegraphics[width=\linewidth]{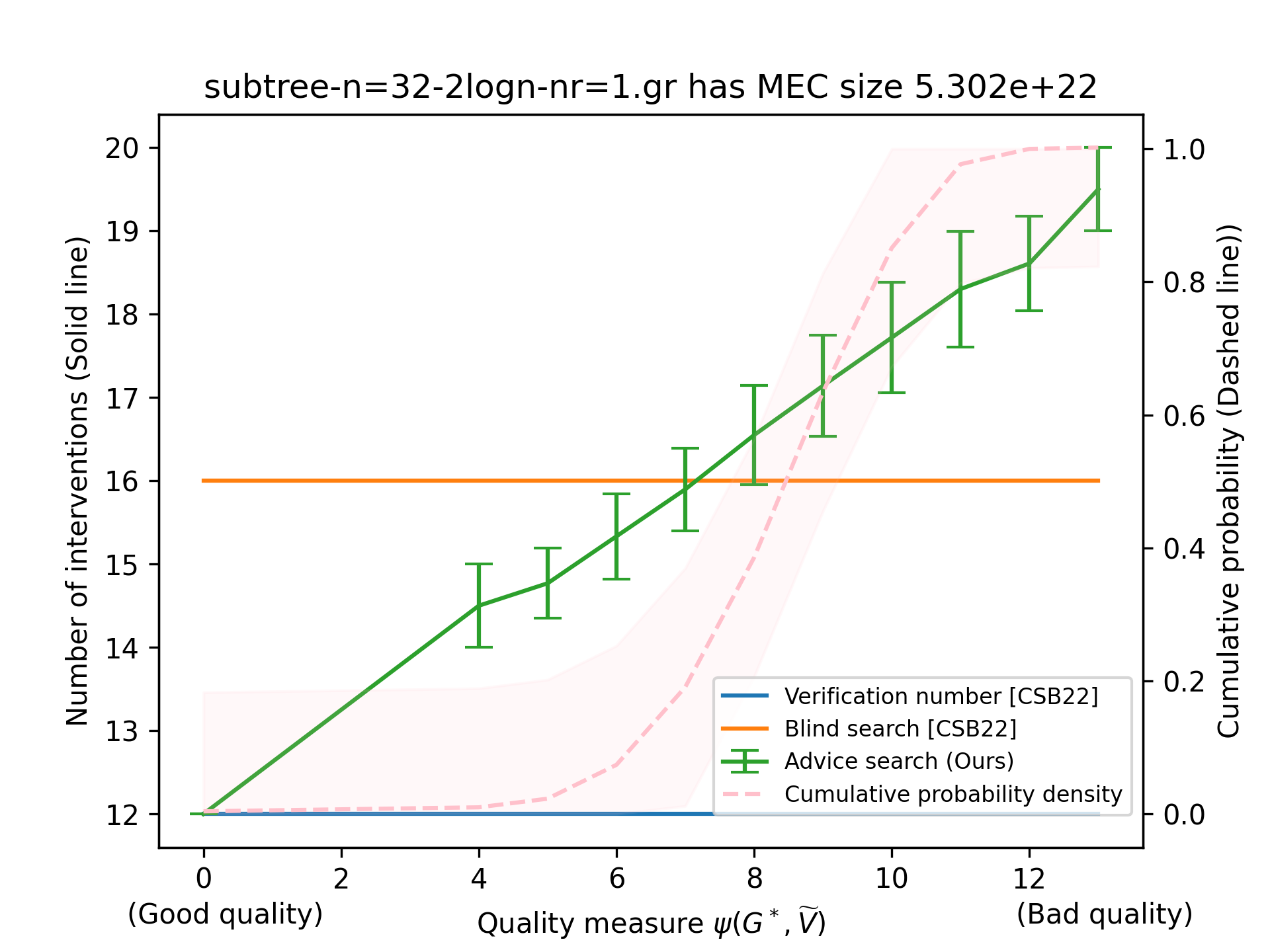}
    \caption{$n = 32$}
\end{subfigure}
\begin{subfigure}[t]{0.3\linewidth}
    \centering
    \includegraphics[width=\linewidth]{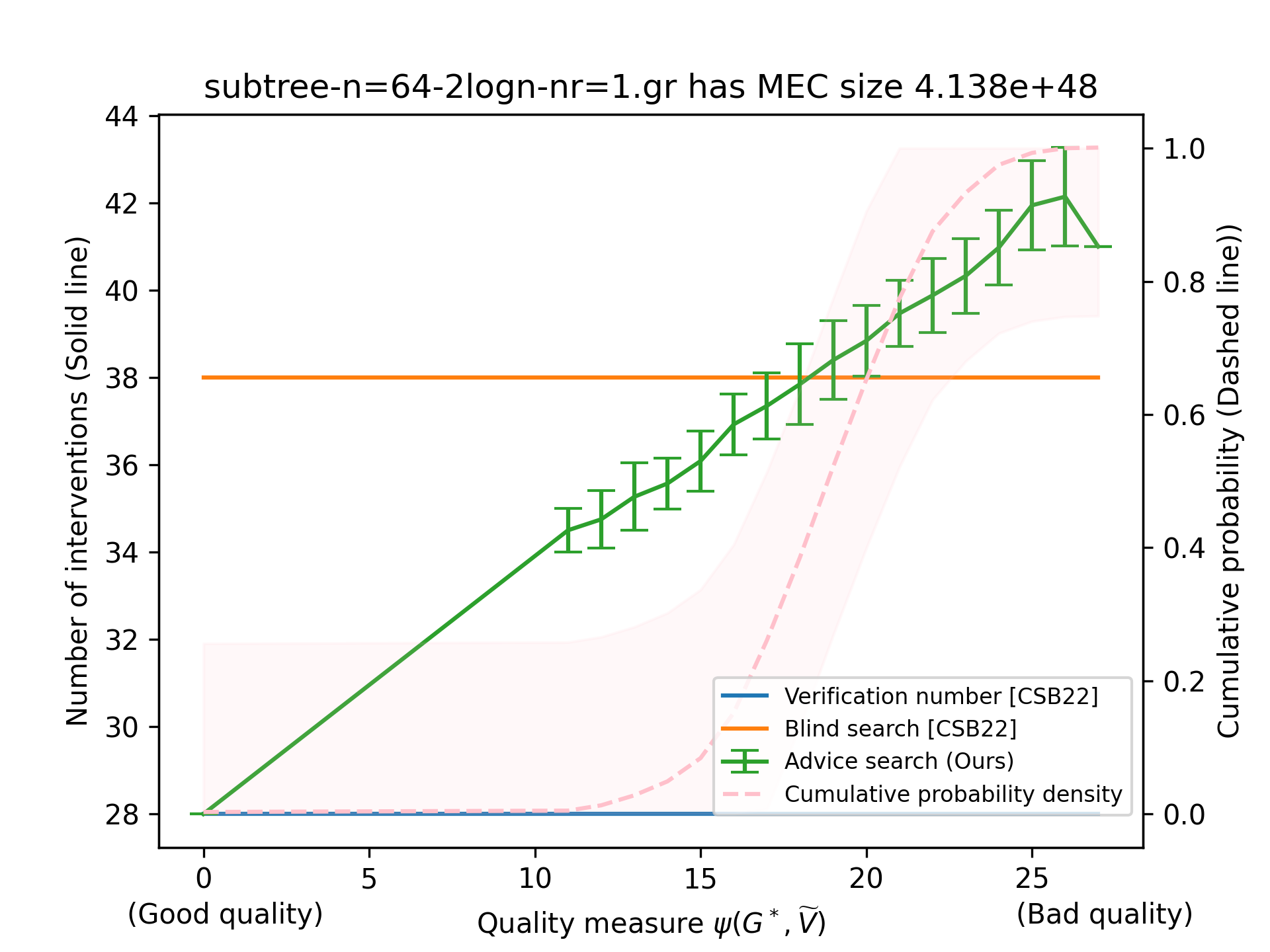}
    \caption{$n = 64$}
\end{subfigure}
\caption{Subtree-2logn synthetic graphs}
\label{fig:subtree-2logn}
\end{figure}

\begin{figure}[htbp]
\centering
\begin{subfigure}[t]{0.3\linewidth}
    \centering
    \includegraphics[width=\linewidth]{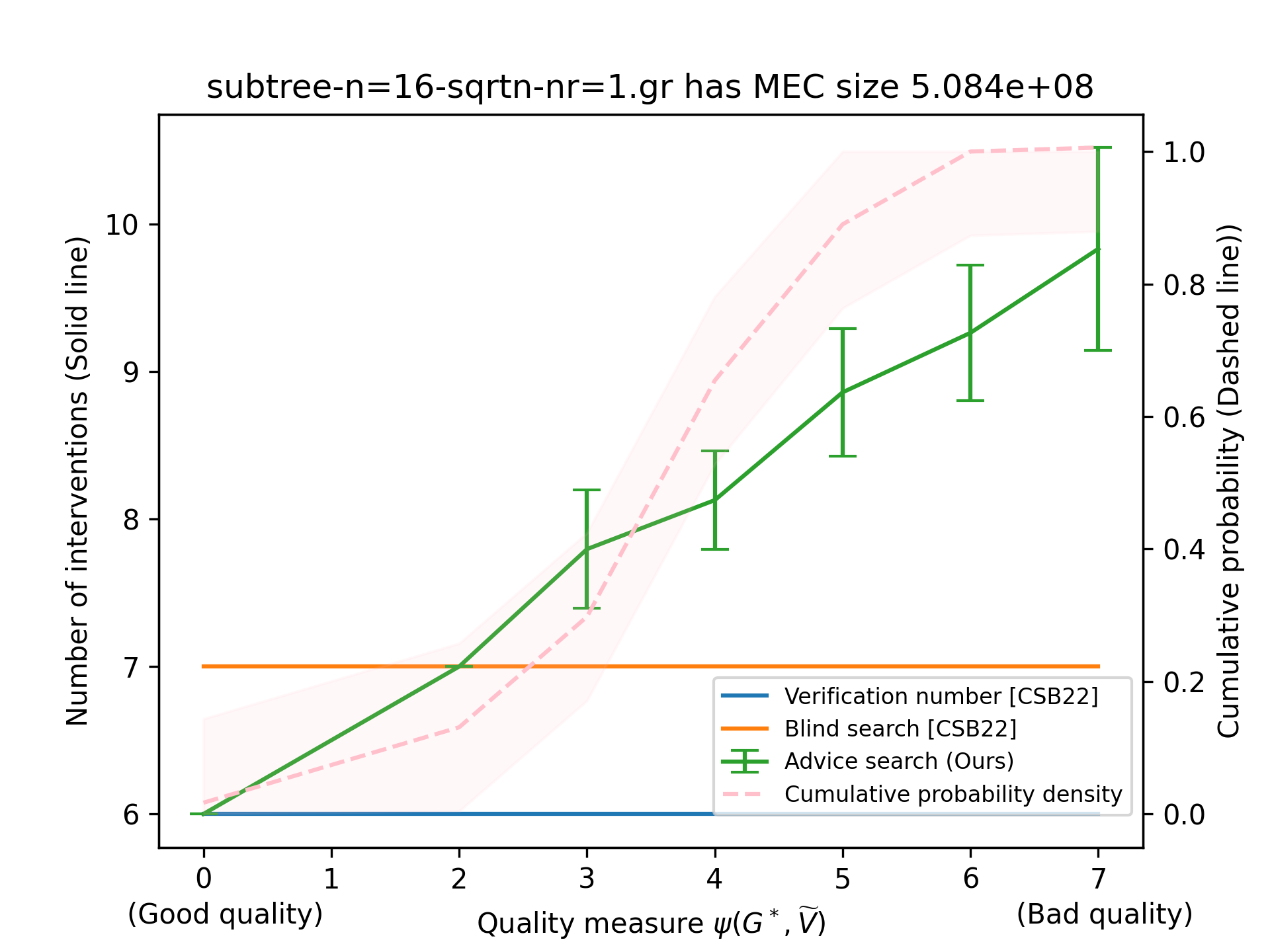}
    \caption{$n = 16$}
\end{subfigure}
\begin{subfigure}[t]{0.3\linewidth}
    \centering
    \includegraphics[width=\linewidth]{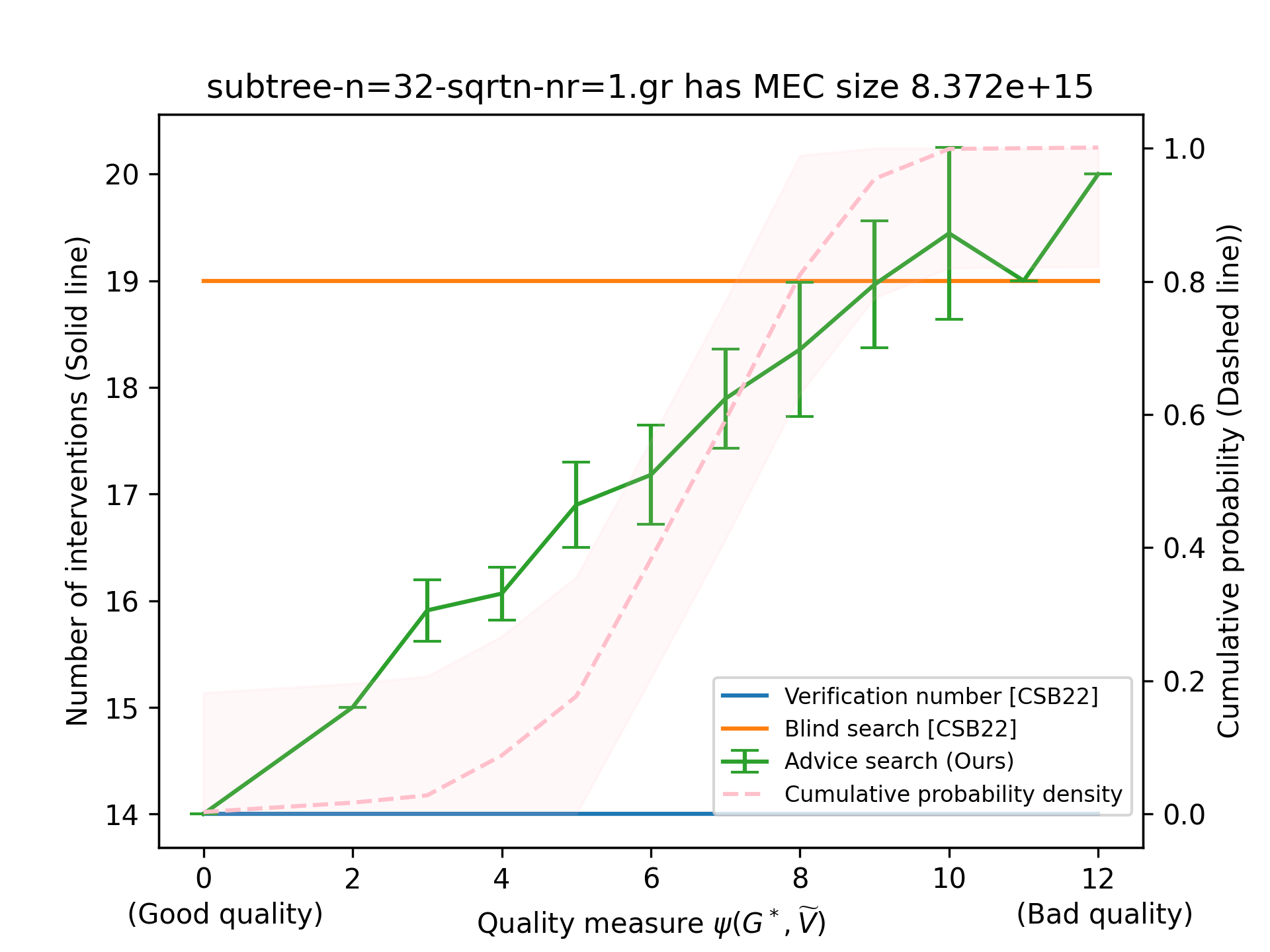}
    \caption{$n = 32$}
\end{subfigure}
\begin{subfigure}[t]{0.3\linewidth}
    \centering
    \includegraphics[width=\linewidth]{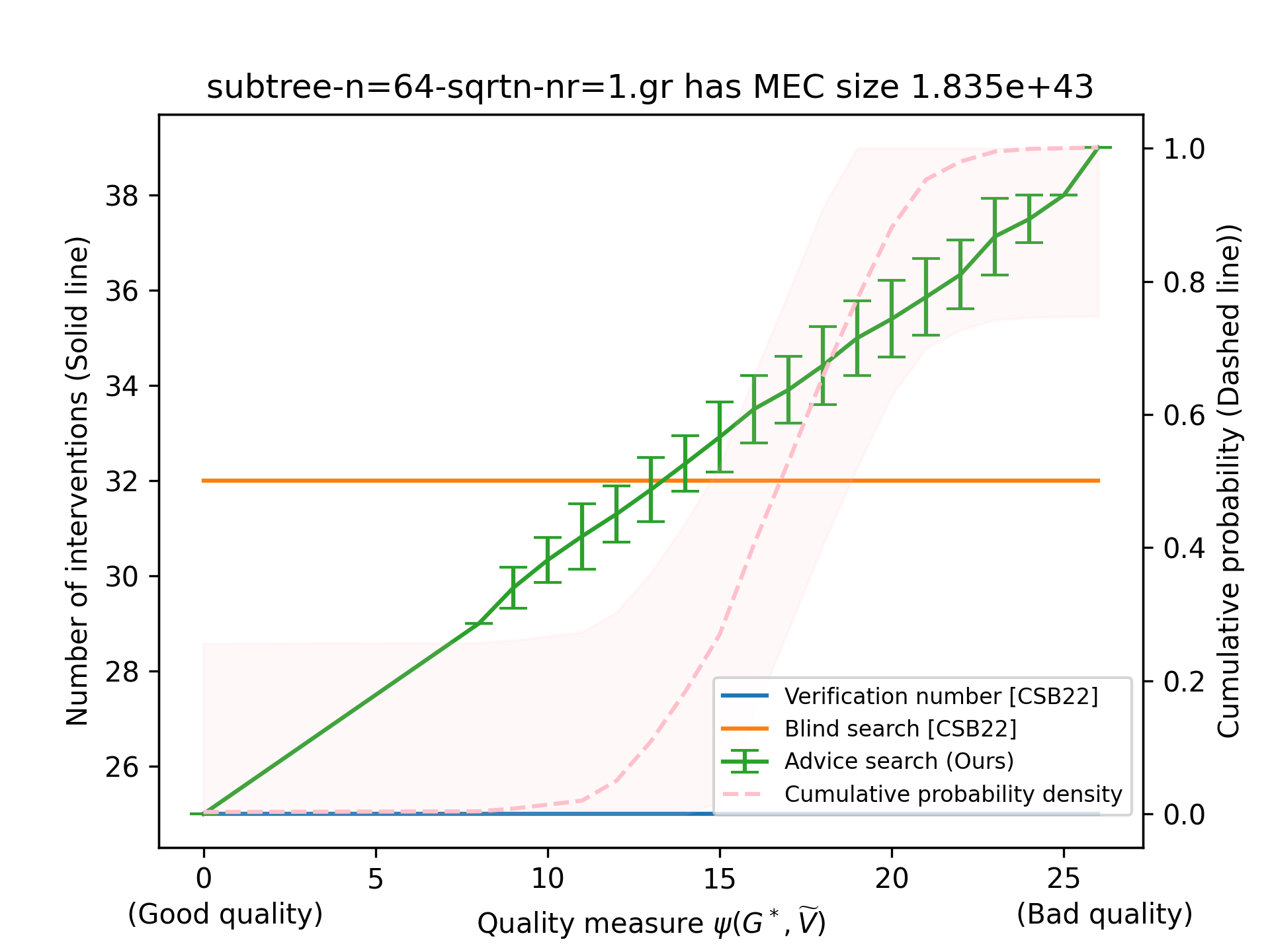}
    \caption{$n = 64$}
\end{subfigure}
\caption{Subtree-sqrtn synthetic graphs}
\label{fig:subtree-sqrtn}
\end{figure}

\begin{figure}[htbp]
\centering
\begin{subfigure}[t]{0.3\linewidth}
    \centering
    \includegraphics[width=\linewidth]{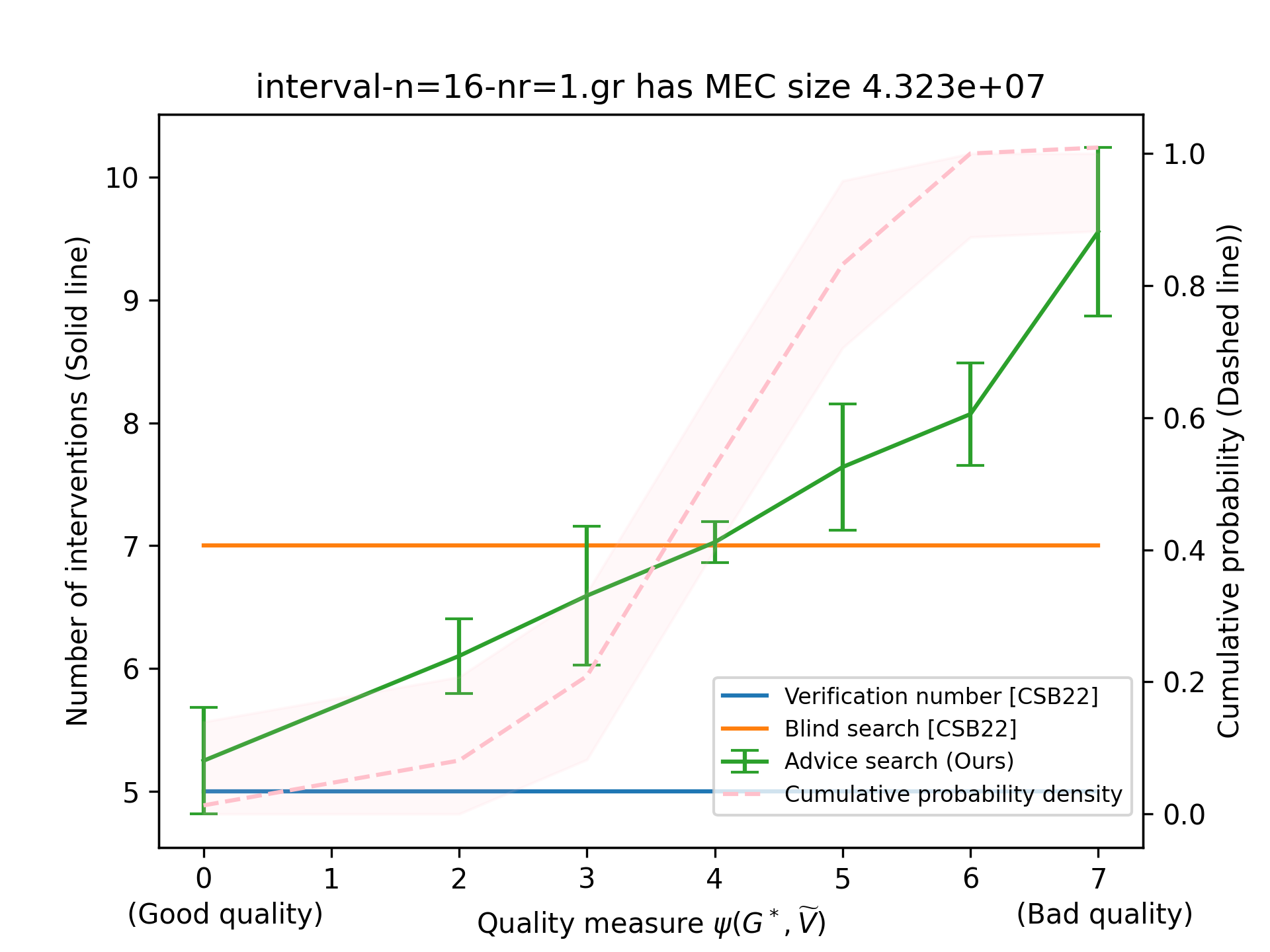}
    \caption{$n = 16$}
\end{subfigure}
\begin{subfigure}[t]{0.3\linewidth}
    \centering
    \includegraphics[width=\linewidth]{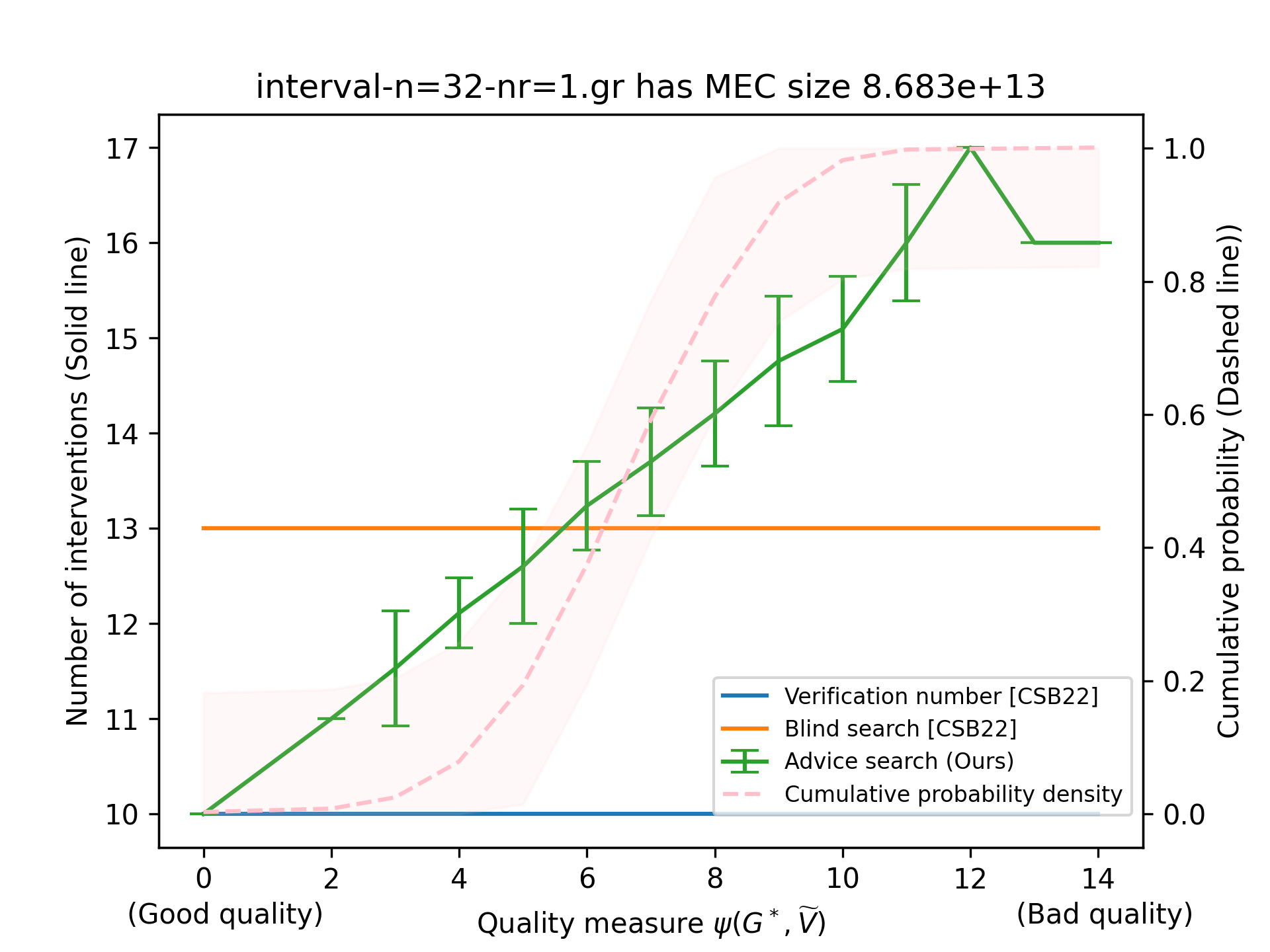}
    \caption{$n = 32$}
\end{subfigure}
\begin{subfigure}[t]{0.3\linewidth}
    \centering
    \includegraphics[width=\linewidth]{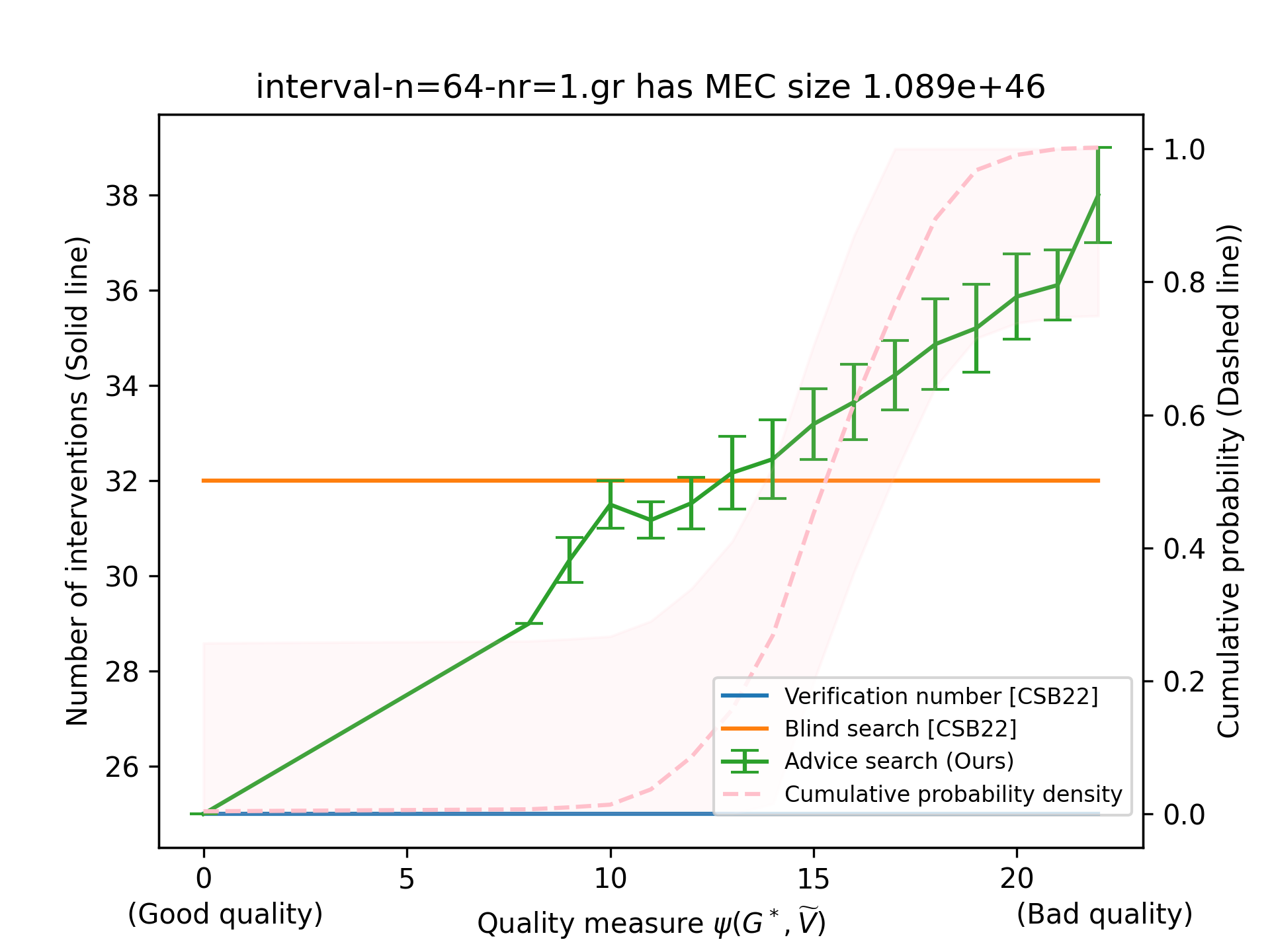}
    \caption{$n = 64$}
\end{subfigure}
\caption{Interval synthetic graphs}
\label{fig:interval}
\end{figure}

\begin{figure}[htbp]
\centering
\begin{subfigure}[t]{0.3\linewidth}
    \centering
    \includegraphics[width=\linewidth]{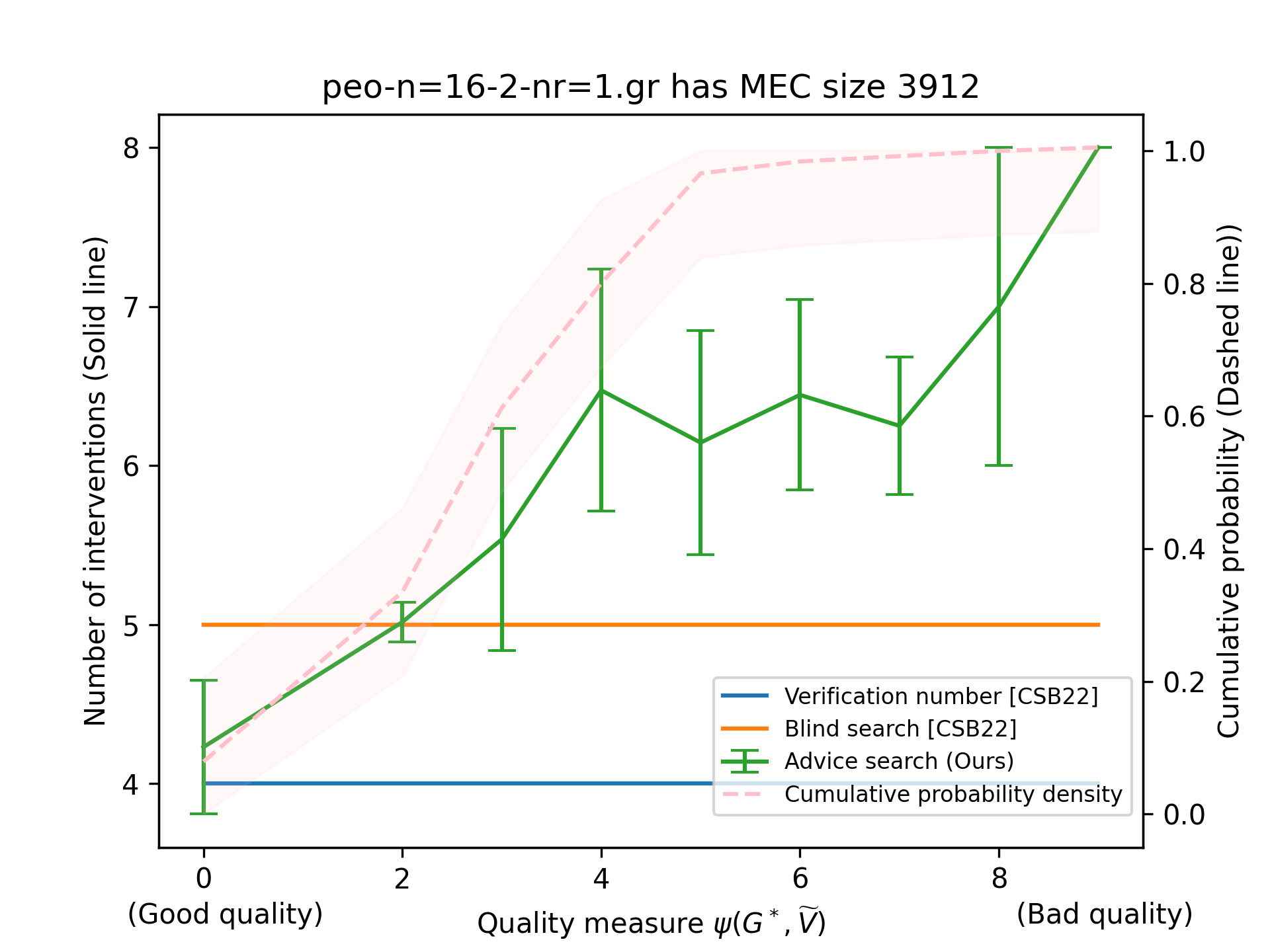}
    \caption{$n = 16$}
\end{subfigure}
\begin{subfigure}[t]{0.3\linewidth}
    \centering
    \includegraphics[width=\linewidth]{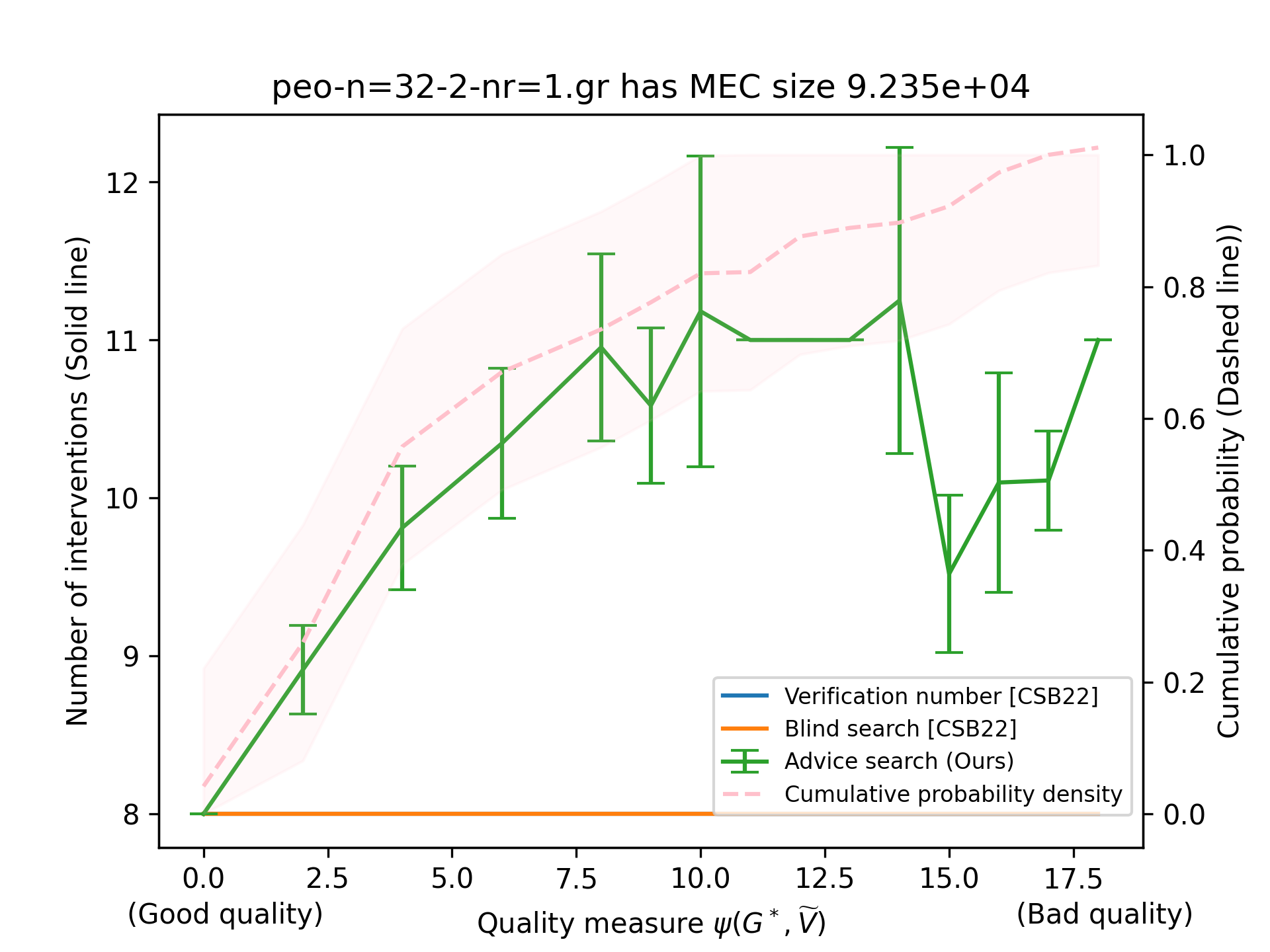}
    \caption{$n = 32$}
\end{subfigure}
\begin{subfigure}[t]{0.3\linewidth}
    \centering
    \includegraphics[width=\linewidth]{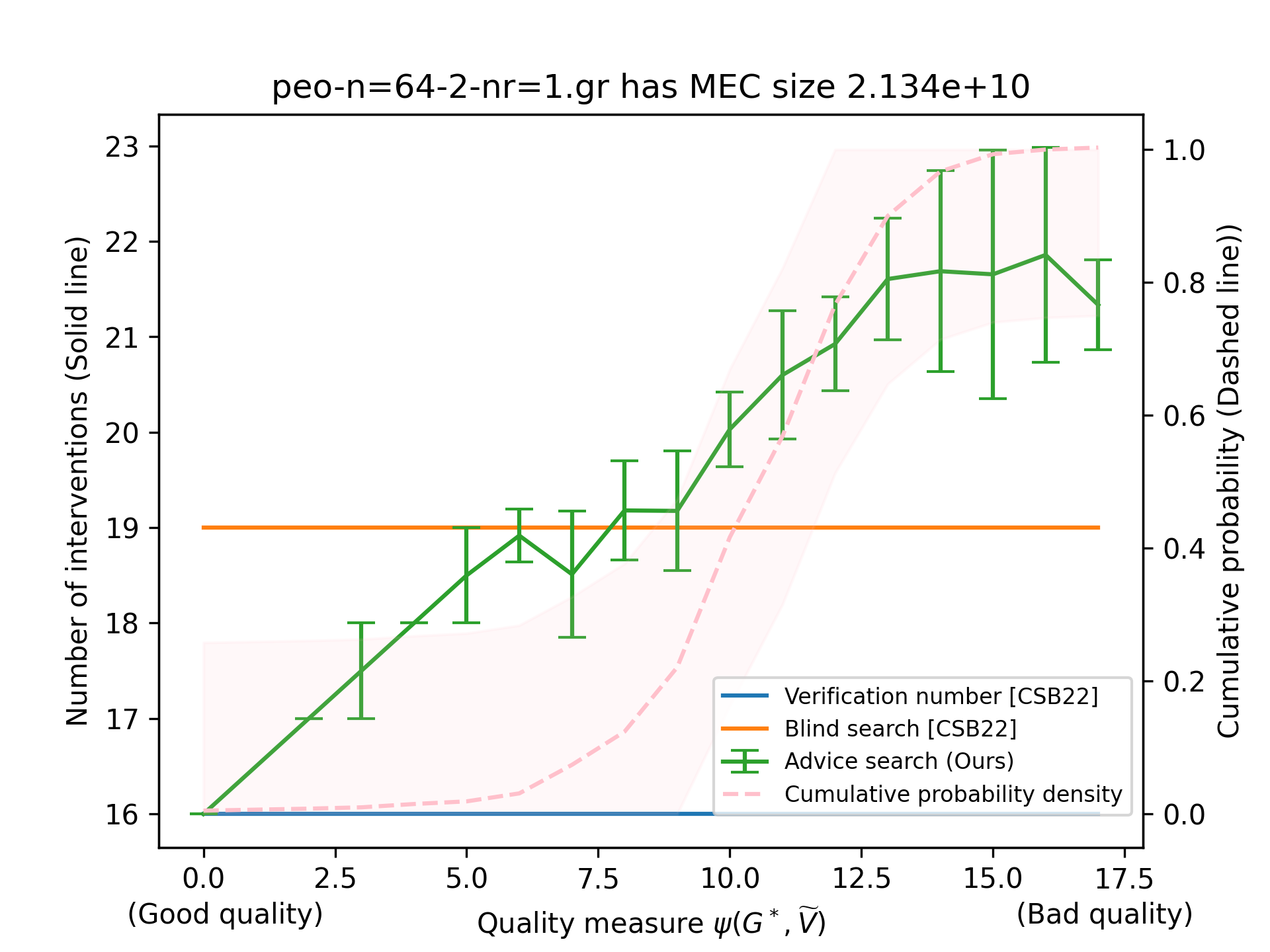}
    \caption{$n = 64$}
\end{subfigure}
\caption{peo-2 synthetic graphs}
\label{fig:peo-2}
\end{figure}

\begin{figure}[htbp]
\centering
\begin{subfigure}[t]{0.3\linewidth}
    \centering
    \includegraphics[width=\linewidth]{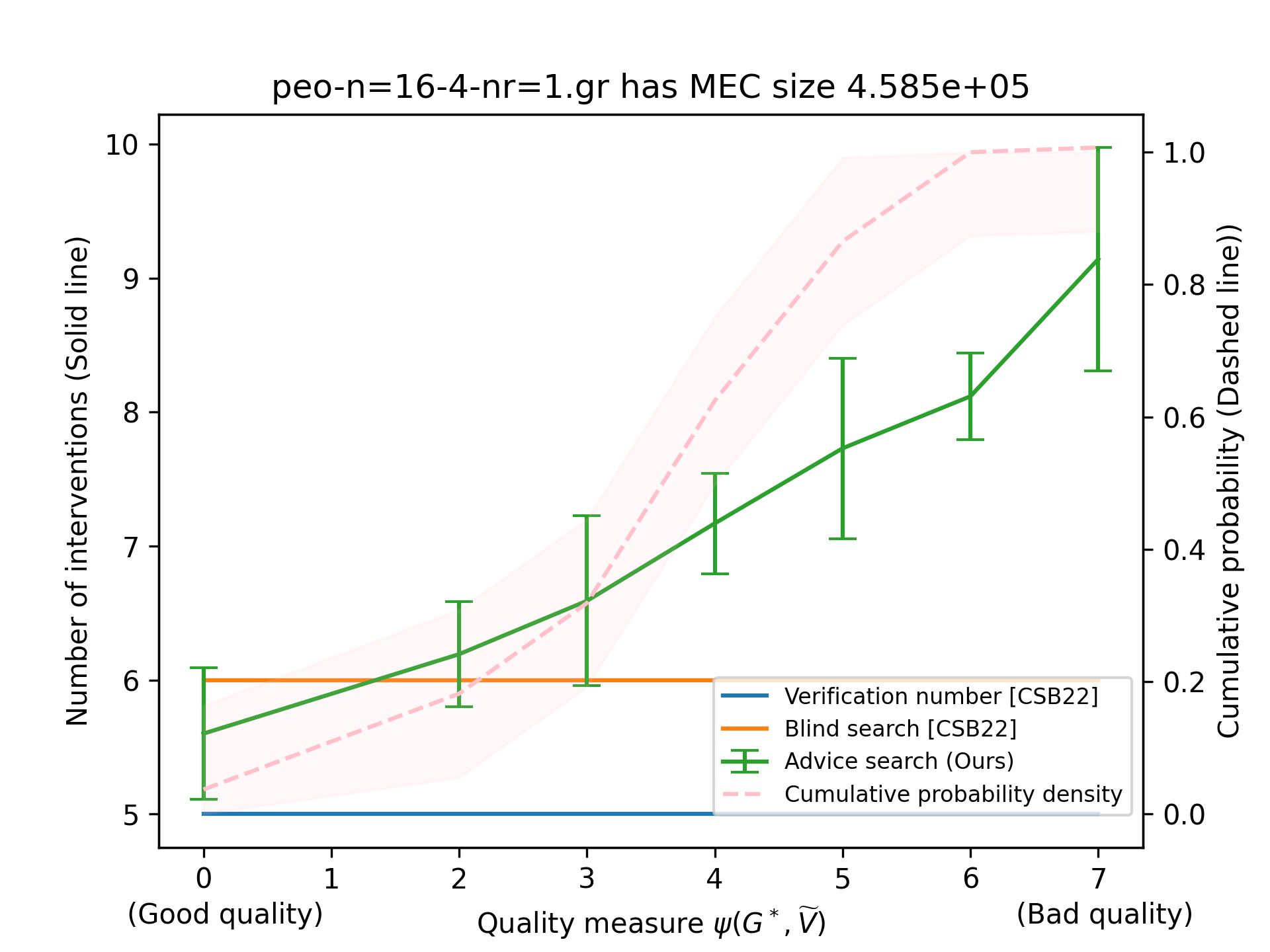}
    \caption{$n = 16$}
\end{subfigure}
\begin{subfigure}[t]{0.3\linewidth}
    \centering
    \includegraphics[width=\linewidth]{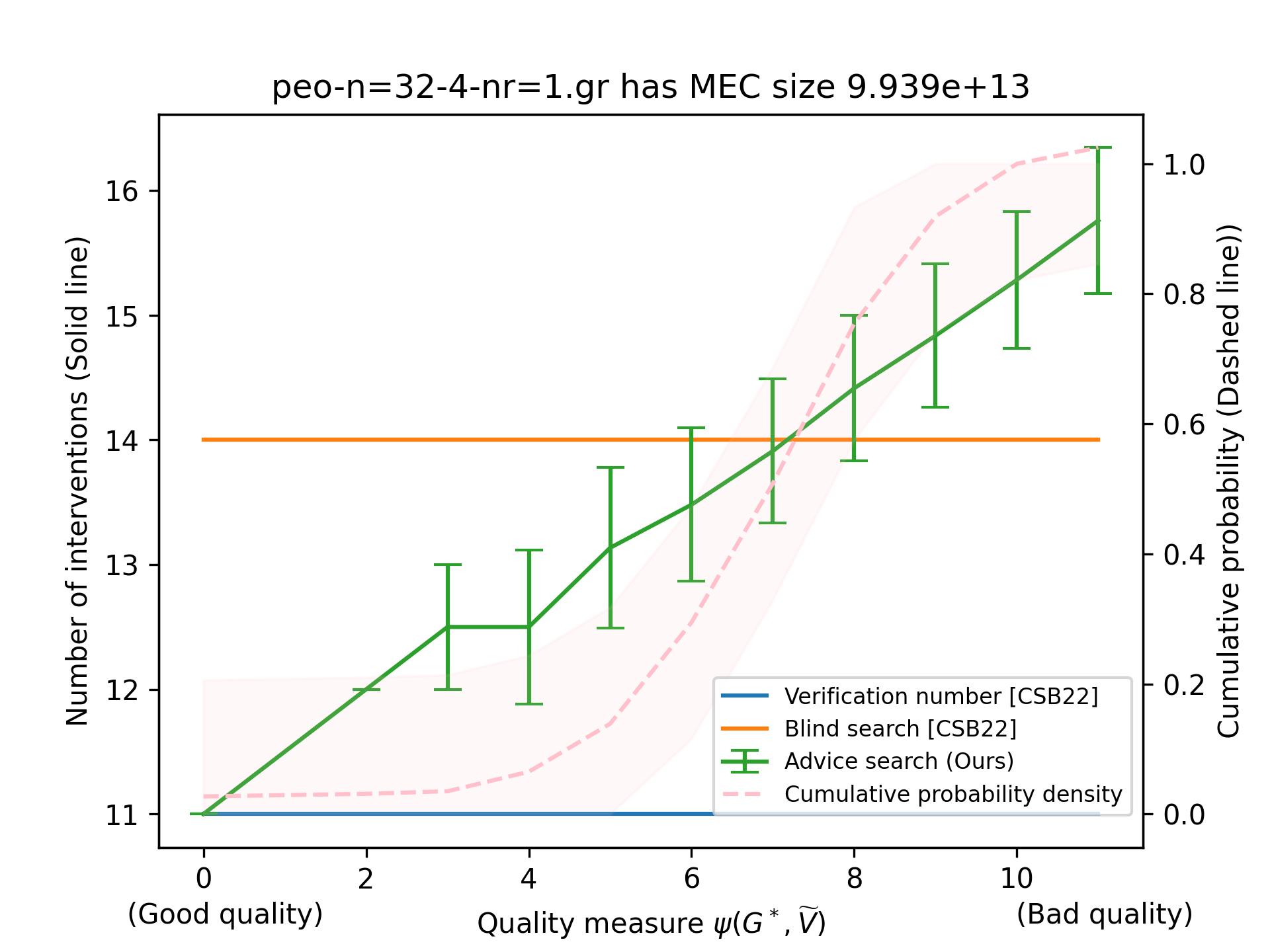}
    \caption{$n = 32$}
\end{subfigure}
\begin{subfigure}[t]{0.3\linewidth}
    \centering
    \includegraphics[width=\linewidth]{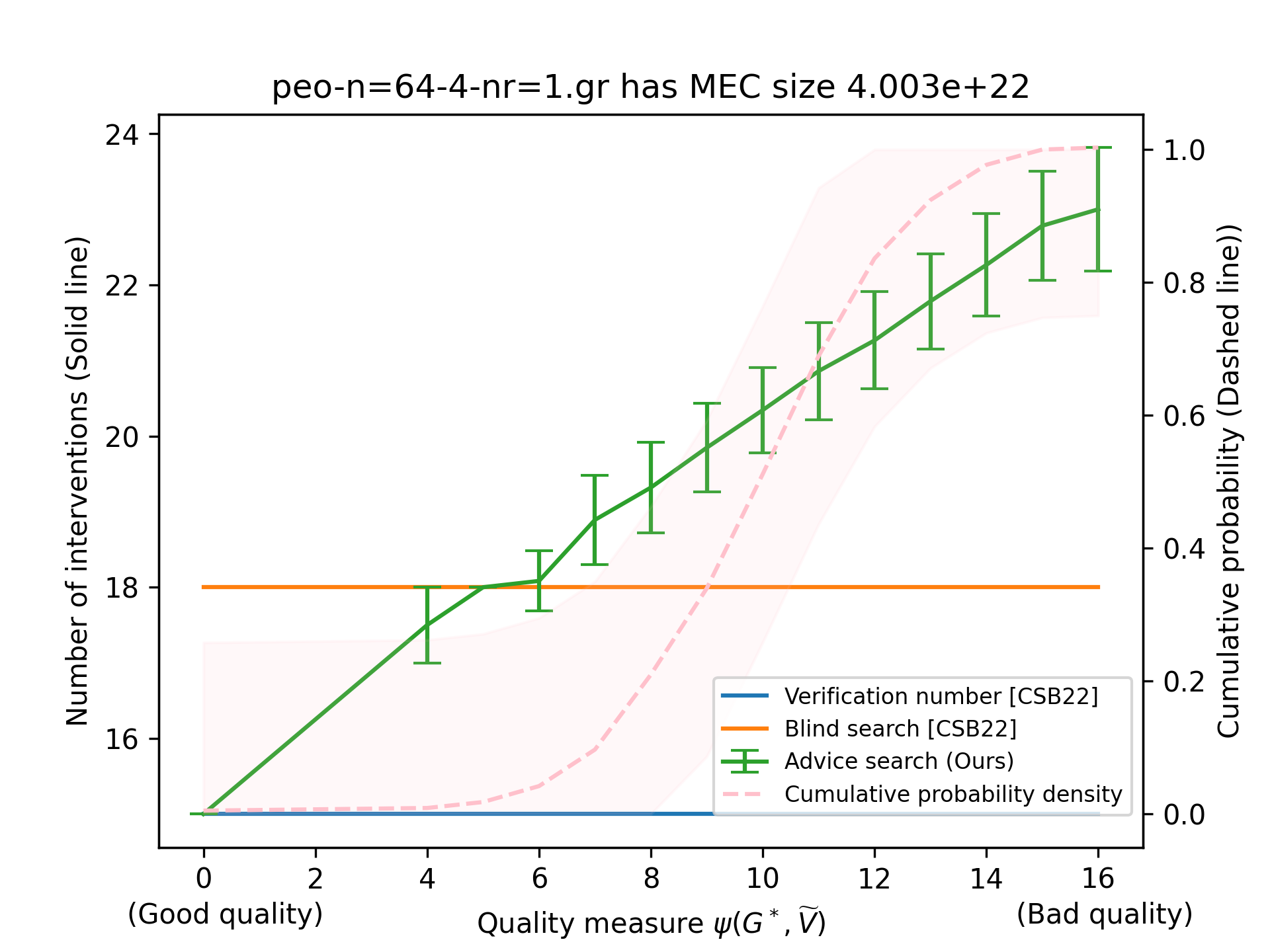}
    \caption{$n = 64$}
\end{subfigure}
\caption{peo-4 synthetic graphs}
\label{fig:peo-4}
\end{figure}

\begin{figure}[htbp]
\centering
\begin{subfigure}[t]{0.3\linewidth}
    \centering
    \includegraphics[width=\linewidth]{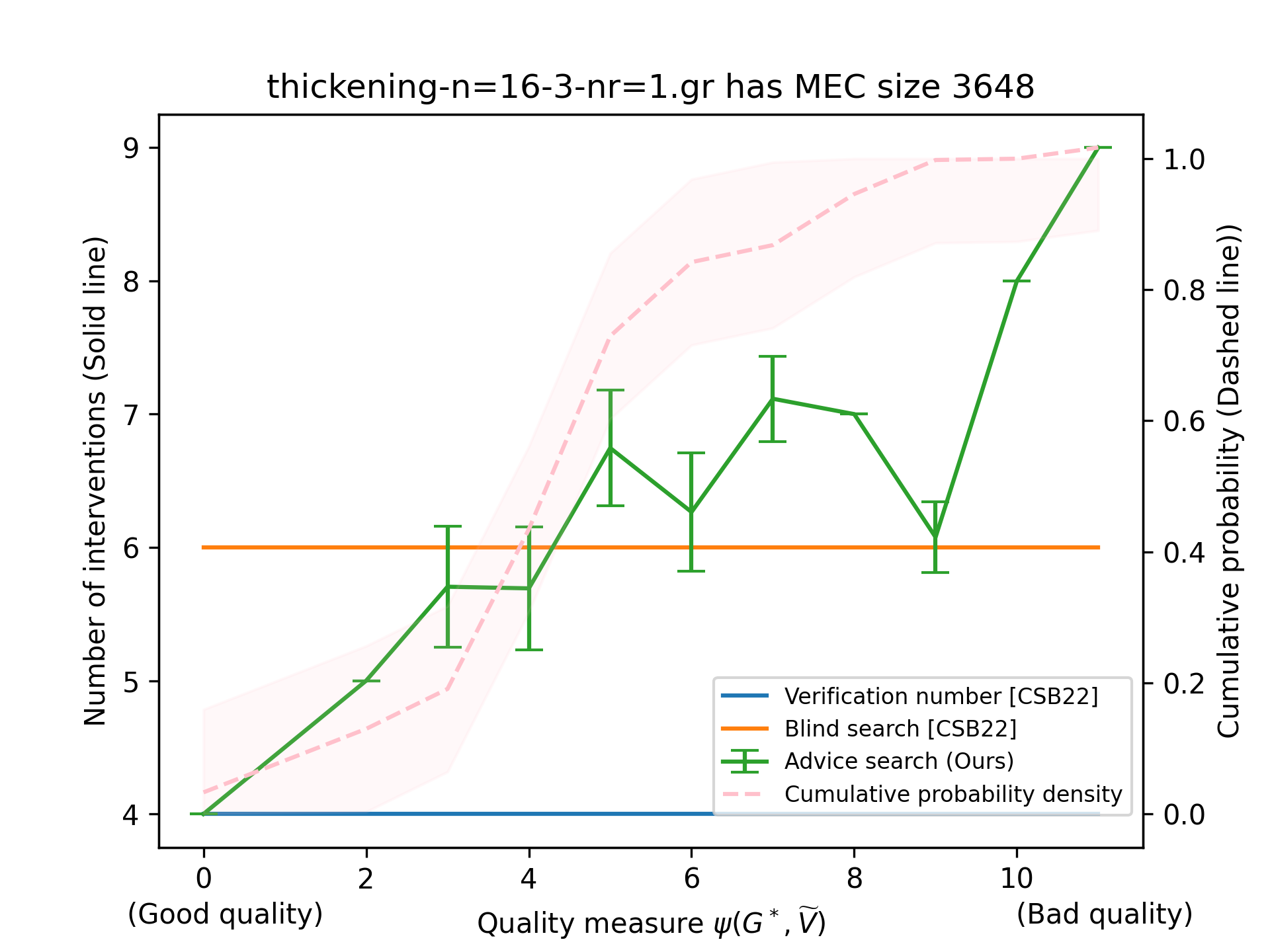}
    \caption{$n = 16$}
\end{subfigure}
\begin{subfigure}[t]{0.3\linewidth}
    \centering
    \includegraphics[width=\linewidth]{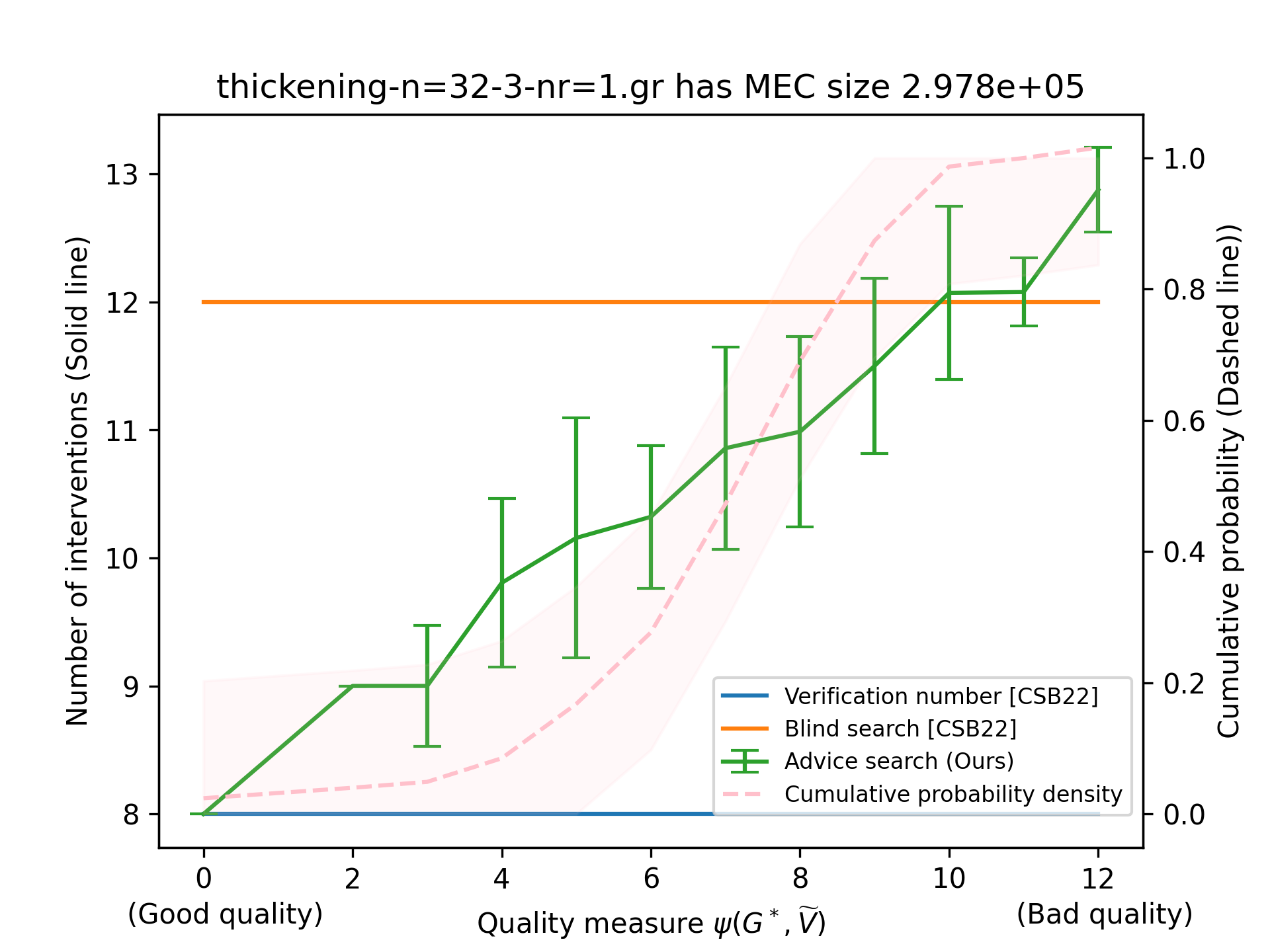}
    \caption{$n = 32$}
\end{subfigure}
\begin{subfigure}[t]{0.3\linewidth}
    \centering
    \includegraphics[width=\linewidth]{plots/thickening-n=64-3-nr=1.gr_1000.png}
    \caption{$n = 64$}
\end{subfigure}
\caption{Thickening-3 synthetic graphs}
\label{fig:thickening-3}
\end{figure}

\begin{figure}[htbp]
\centering
\begin{subfigure}[t]{0.3\linewidth}
    \centering
    \includegraphics[width=\linewidth]{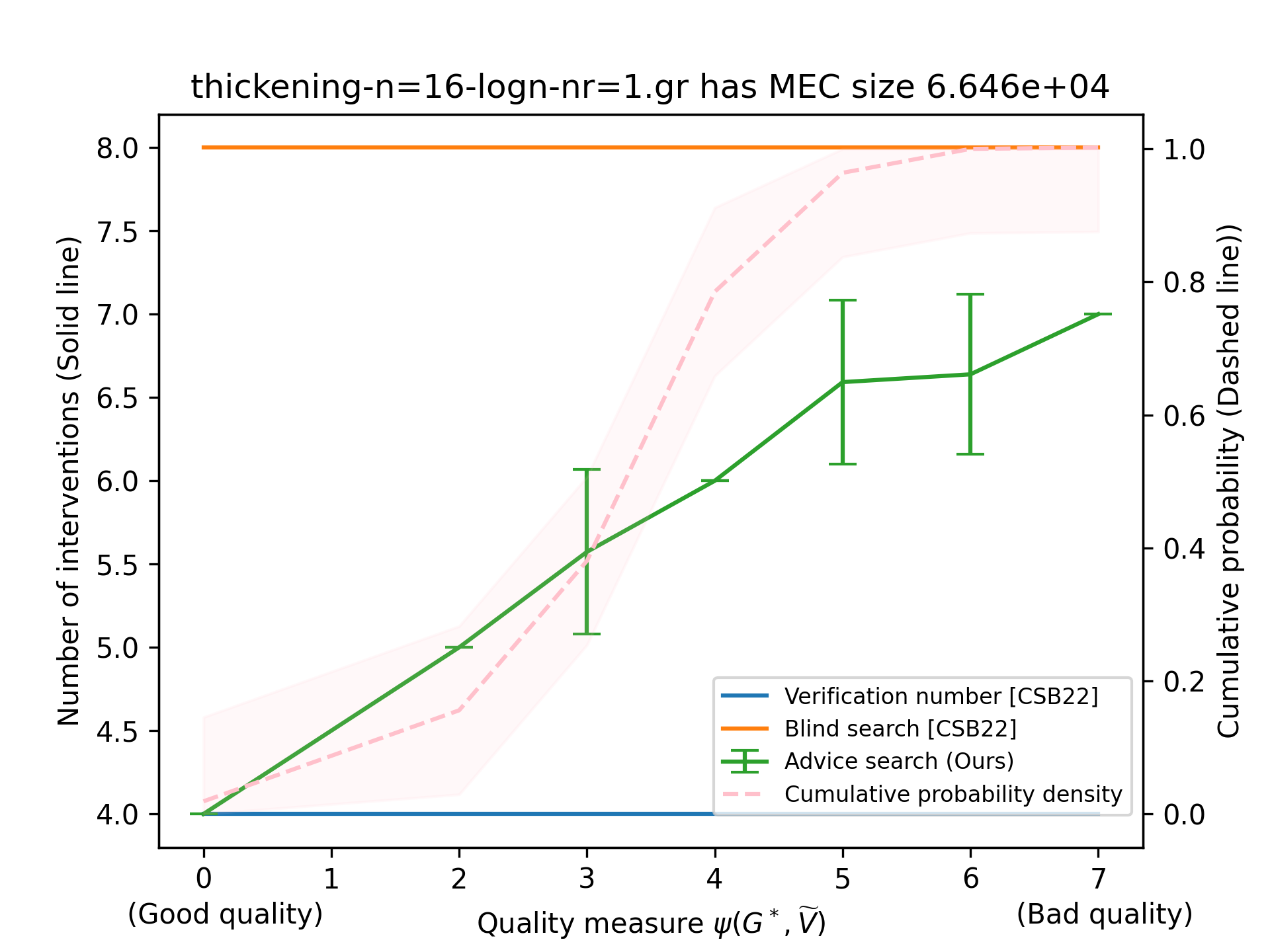}
    \caption{$n = 16$}
\end{subfigure}
\begin{subfigure}[t]{0.3\linewidth}
    \centering
    \includegraphics[width=\linewidth]{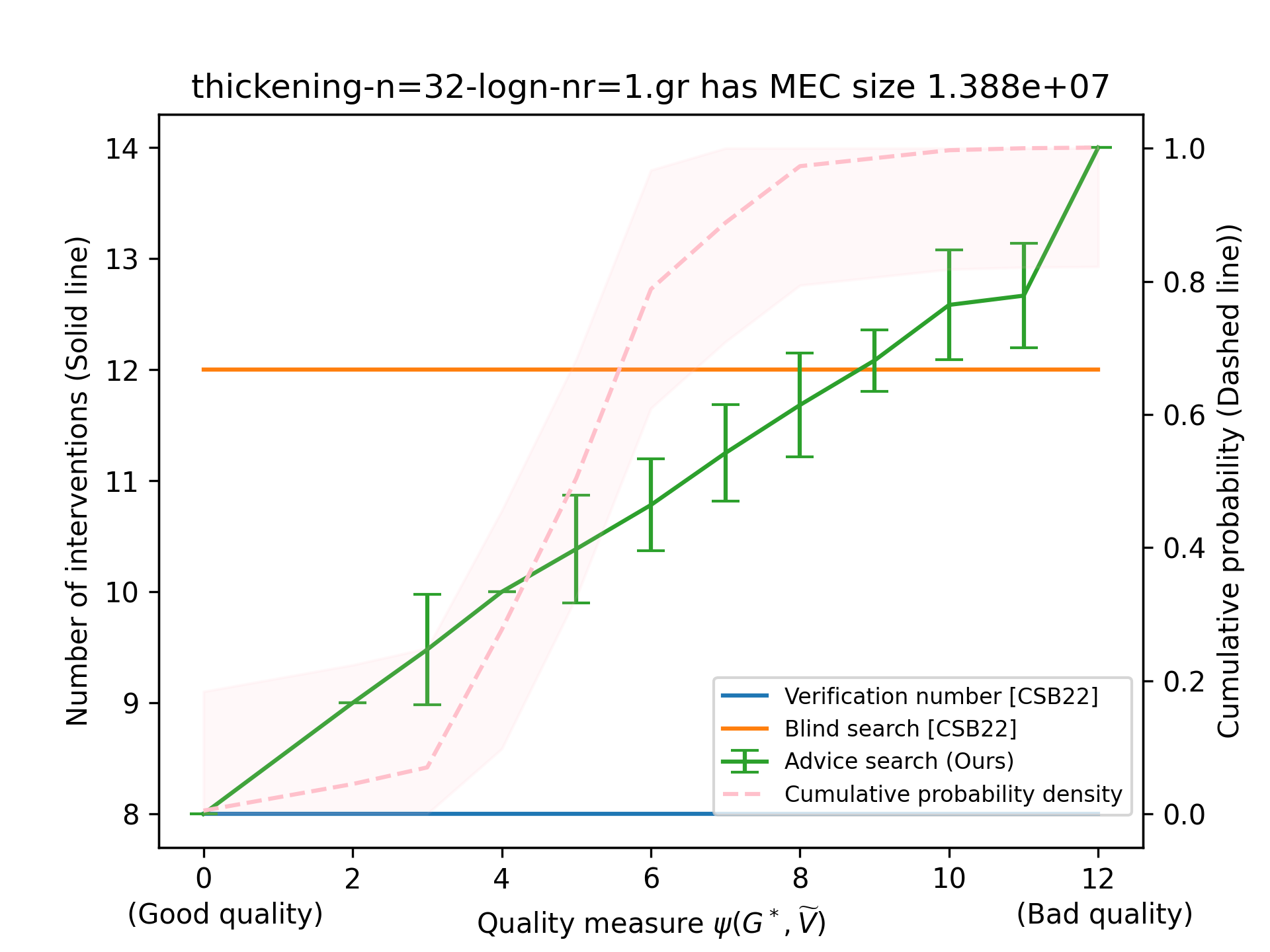}
    \caption{$n = 32$}
\end{subfigure}
\begin{subfigure}[t]{0.3\linewidth}
    \centering
    \includegraphics[width=\linewidth]{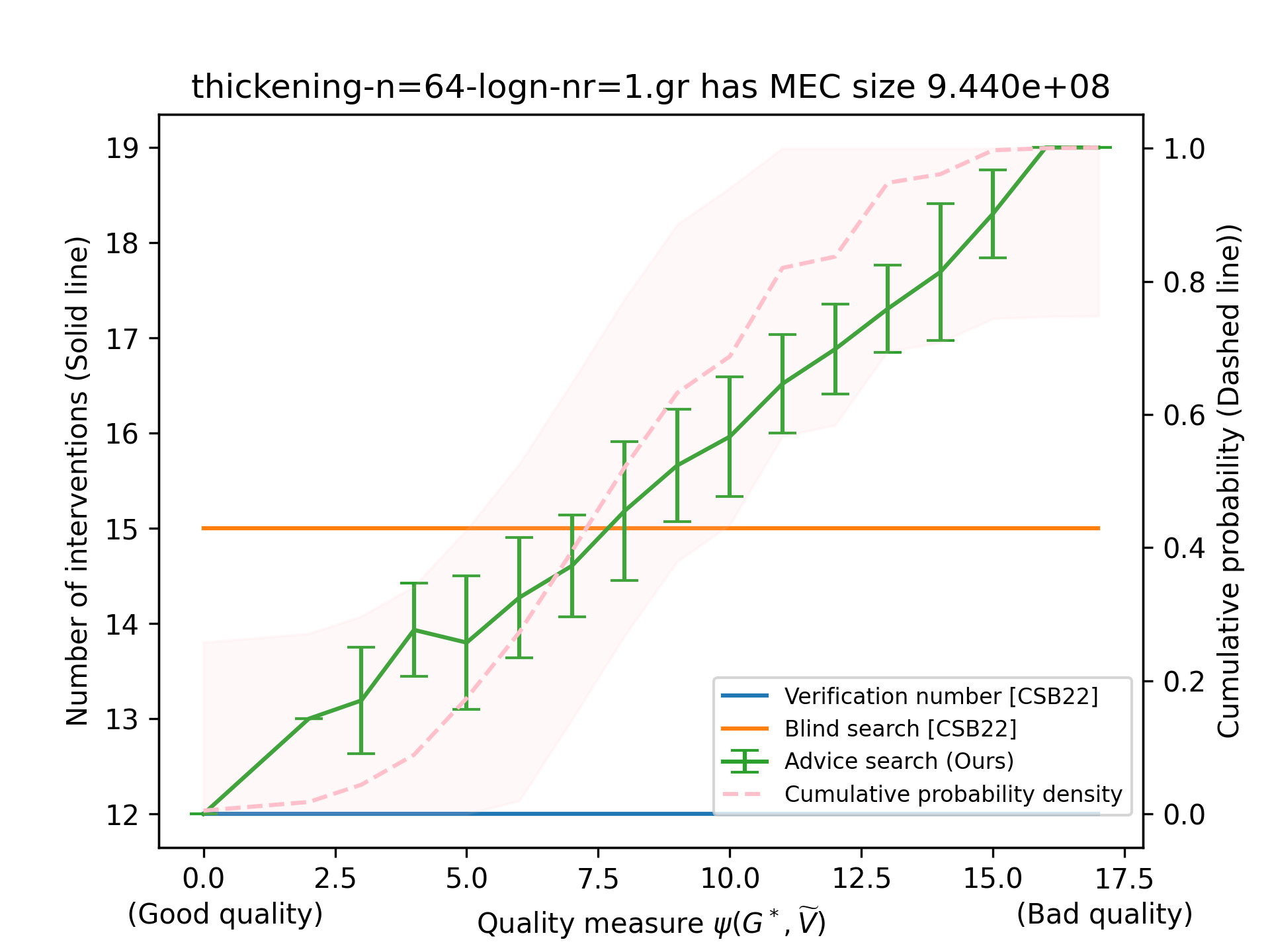}
    \caption{$n = 64$}
\end{subfigure}
\caption{Thickening-logn synthetic graphs}
\label{fig:thickening-logn}
\end{figure}

\begin{figure}[htbp]
\centering
\begin{subfigure}[t]{0.3\linewidth}
    \centering
    \includegraphics[width=\linewidth]{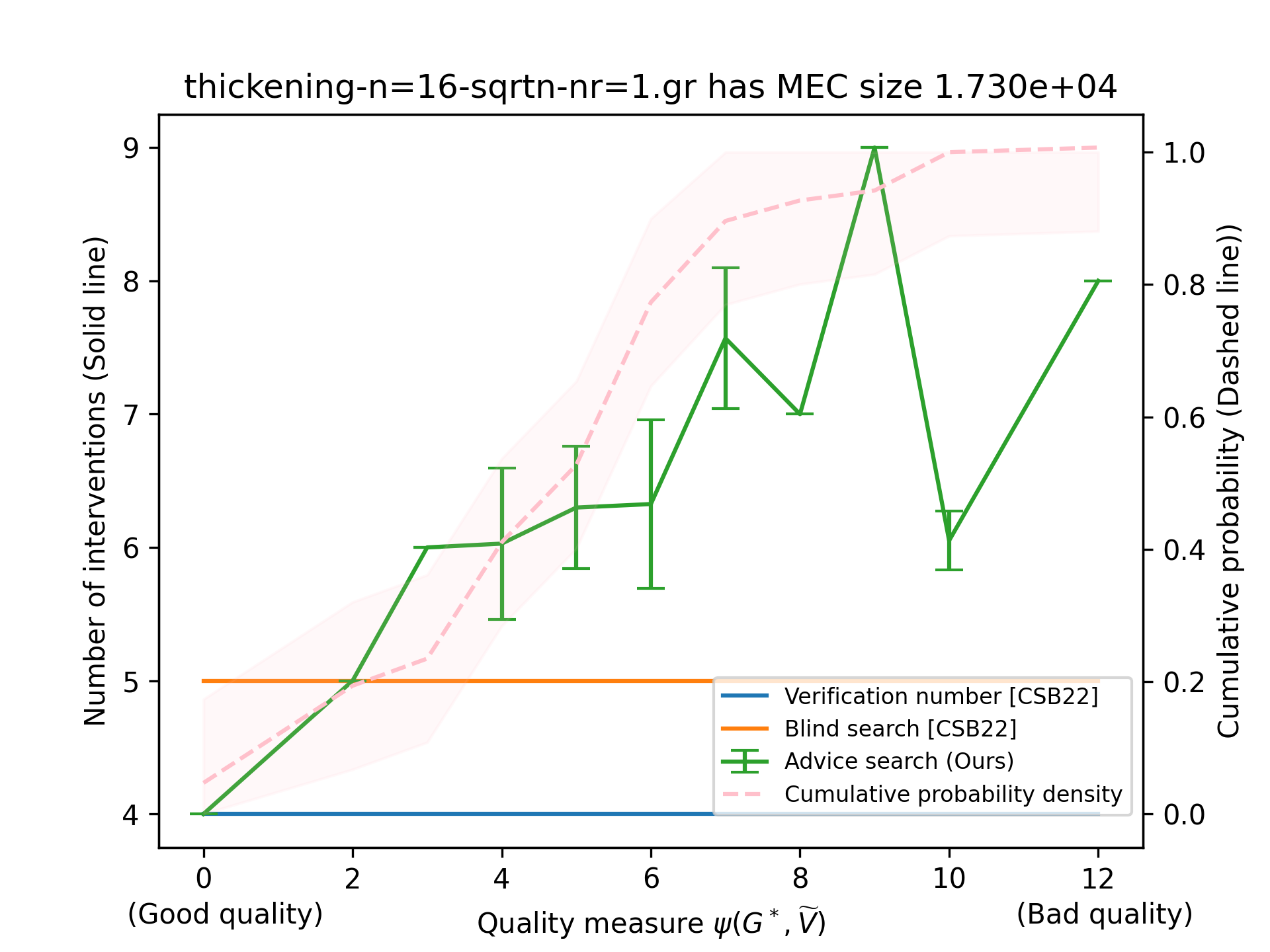}
    \caption{$n = 16$}
\end{subfigure}
\begin{subfigure}[t]{0.3\linewidth}
    \centering
    \includegraphics[width=\linewidth]{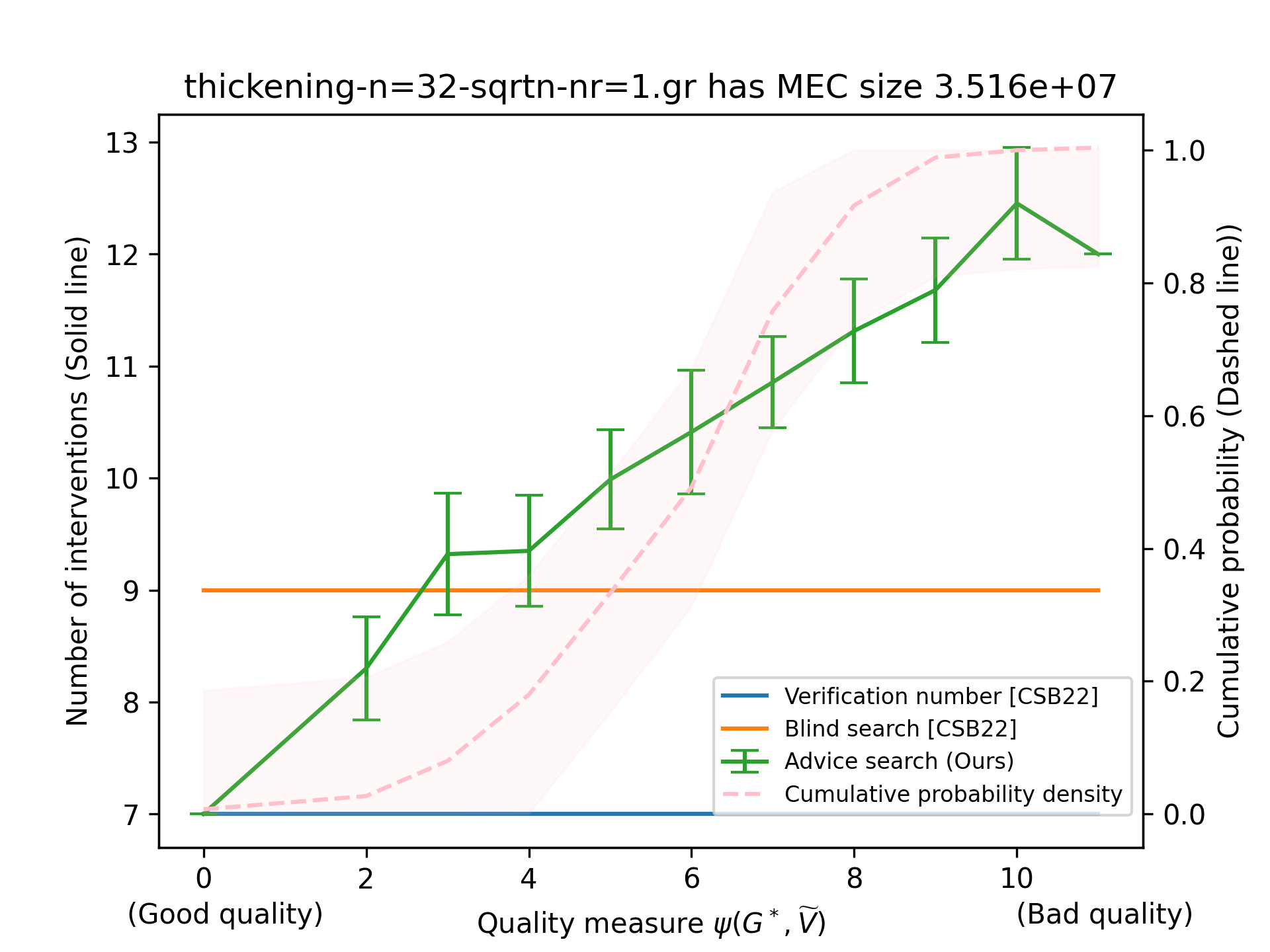}
    \caption{$n = 32$}
\end{subfigure}
\begin{subfigure}[t]{0.3\linewidth}
    \centering
    \includegraphics[width=\linewidth]{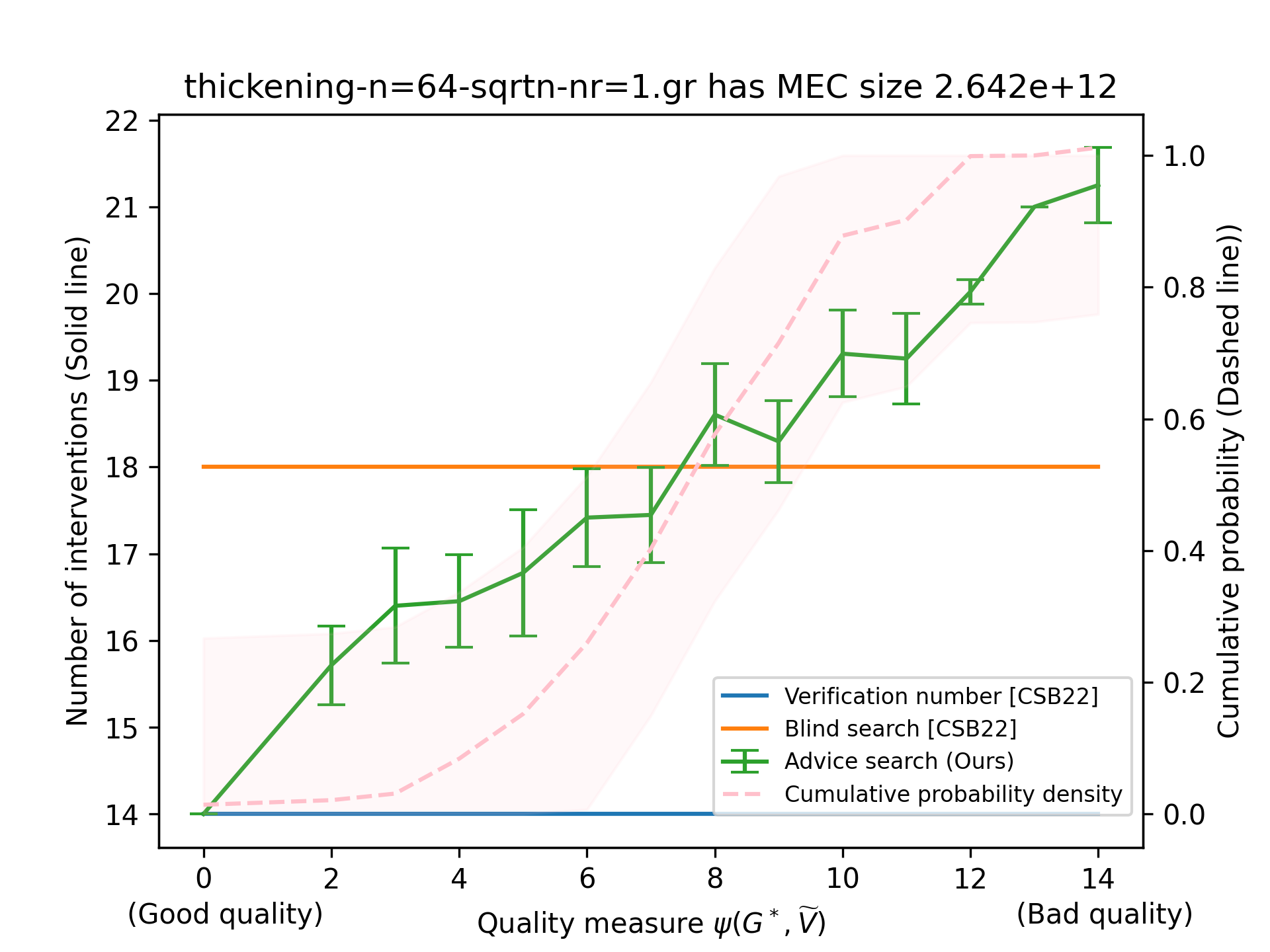}
    \caption{$n = 64$}
\end{subfigure}
\caption{Thickening-sqrtn synthetic graphs}
\label{fig:thickening-sqrtn}
\end{figure}

\end{document}